\documentclass[twoside,11pt]{article}

\usepackage[nohyperref]{jmlr2e}
\usepackage{hyperref}
\hypersetup{
  colorlinks   = true, 
  urlcolor     = black, 
  linkcolor    = black, 
  citecolor    = black, 
  filecolor    = black 
}

\usepackage{makecell}
\usepackage{cite}
\usepackage{amsmath,amssymb,amsfonts}
\usepackage{algorithm,algpseudocode}
\usepackage{graphicx,graphics}
\usepackage{textcomp}
\usepackage{bm}
\usepackage[dvipsnames,svgnames,x11names]{xcolor}

\usepackage{enumitem}
\usepackage{array}
\usepackage{mathtools}
\usepackage{tcolorbox}
\usepackage{color}
\usepackage{theorem}
\usepackage{cases}
\usepackage{url}
\usepackage{tikz}
\usetikzlibrary{shapes,arrows}
\input{mysymbol.sty}

\usepackage{multirow}
\usepackage{hhline}

\def \be {\begin{equation*}}
\def \ee {\end{equation*}}
\def \bi {\begin{itemize}[align=parleft, leftmargin=*] \item}
\def \ei {\end{itemize}}
\def \im {\item}
\newcommand{\Rgr}{{ R_{\scriptscriptstyle \rm G} }}
\newcommand{\Rno}{{ R_{\scriptscriptstyle \rm N} }}
\newcommand{\bbRgr}{{ \bbR_{\scriptscriptstyle \rm G} }}
\newcommand{\bbRno}{{ \bbR_{\scriptscriptstyle \rm N} }}

\newcommand{\bbSCp} {{ \bbS'_{\rm \scriptscriptstyle C} }}
\newcommand{\hbSC} {{ \hbS_{\rm \scriptscriptstyle C} }}

\newcommand{\bbSVp} {{ \bbS'_{\rm \scriptscriptstyle V} }}
\newcommand{\hbSV} {{ \hbS_{\rm \scriptscriptstyle V} }}

\newcommand{\vect} {{\rm vec}}
\newcommand{\diag} {{\rm diag}}

\newcommand{\SA} {{ \ccalS_{\rm \scriptscriptstyle A} }}
\newcommand{\SL} {{ \ccalS_{\rm \scriptscriptstyle L} }}
\newcommand{\Nmin} {{ N_{\rm\scriptscriptstyle \min} }}
\newcommand{\Nmax} {{ N_{\rm\scriptscriptstyle \max} }}

\newcommand{\st} {{\mathbf{ST}}}
\newcommand{\strw} {{\mathbf{ST}\text{-}\mathbf{Rw}}}
\newcommand{\stba} {{\mathbf{ST}\text{-}\mathbf{Ba}}}
\newcommand{\fstcg} {{\mathbf{FST}_{\rm \scriptscriptstyle C}\text{-}R_{\rm \scriptscriptstyle G}}}
\newcommand{\fstcn} {{\mathbf{FST}_{\rm \scriptscriptstyle C}\text{-}R_{\rm \scriptscriptstyle N}}}
\newcommand{\fstvg} {{\mathbf{FST}_{\rm \scriptscriptstyle V}\text{-}R_{\rm \scriptscriptstyle G}}}
\newcommand{\fstvn} {{\mathbf{FST}_{\rm \scriptscriptstyle V}\text{-}R_{\rm \scriptscriptstyle N}}}
\newcommand{\fgl} {{\mathbf{FGL}}}

\newcommand{\normm}[1]{{\vert\kern-0.25ex\vert\kern-0.25ex\vert #1 
    \vert\kern-0.25ex\vert\kern-0.25ex\vert}}
\newcommand{\normms}[1]{{\left\vert\kern-0.25ex\left\vert\kern-0.25ex\left\vert #1 
    \right\vert\kern-0.25ex\right\vert\kern-0.25ex\right\vert}}
\newcommand{\normv}[1]{{\vert\kern-0.25ex\vert #1 
    \vert\kern-0.25ex\vert}}
\newcommand{\normvs}[1]{{\left\vert\kern-0.25ex\left\vert #1 
    \right\vert\kern-0.25ex\right\vert}}
    
\usepackage{moresize}
\usepackage{pgfplots}
\pgfplotsset{compat=1.17}
\pgfplotstableset{col sep=comma}

\newtheorem{mytheorem}{\bf Theorem}

\newtheorem{myassumption}{\bf Assumption}
\newtheorem{mycorollary}{\bf Corollary}
\newtheorem{mylemma}{\bf Lemma}
\newtheorem{myremark}{\bf Remark}
\usepackage{lastpage}
\jmlrheading{XX}{2025}{1-\pageref{LastPage}}{09/25; Revised
X/XX}{X/XX}{XX-XXXX}{Madeline Navarro, Andrei Buciulea, Samuel Rey, Antonio G. Marques, and Santiago Segarra}

\ShortHeadings{Estimating Fair Graphs from Graph-Stationary Data}{Navarro, Buciulea, Rey, Marques, and Segarra}
\firstpageno{1}

\begin{document}

\pgfplotsset{compat=1.17}
\pgfplotstableset{col sep=comma}
\tikzset{every mark/.append style={scale=1.5, solid}, font=\footnotesize}
\pgfplotsset{
    width=1.05\textwidth,
    legend style={
        font=\ssmall ,  
        inner xsep=1pt,
        inner ysep=1pt,
        nodes={inner sep=1pt}},
    legend cell align=left,
	every axis/.append style={line width=0.5pt},
	every axis plot/.append style={line width=1.25pt},
    every axis y label/.append style={yshift=-5pt}
}

\title{Estimating Fair Graphs from Graph-Stationary Data}

\author{\name Madeline Navarro \email nav@rice.edu \\
       \addr Department of Electrical and Computer Engineering\\
       Rice University\\
       Houston, TX 77005-1827, USA
       \AND
       \name Andrei Buciulea \email andrei.buciulea@urjc.es \\
       \addr Department of Signal Theory and Communications\\
       King Juan Carlos University\\
       Madrid, Spain
       \AND
       \name Samuel Rey \email samuel.rey.escudero@urjc.es \\
       \addr Department of Signal Theory and Communications\\
       King Juan Carlos University\\
       Madrid, Spain
       \AND
       \name Antonio G. Marques \email antonio.garcia.marques@urjc.es \\
       \addr Department of Signal Theory and Communications\\
       King Juan Carlos University\\
       Madrid, Spain
       \AND
       \name Santiago Segarra \email segarra@rice.edu \\
       \addr Department of Electrical and Computer Engineering\\
       Rice University\\
       Houston, TX 77005-1827, USA}

\editor{TBD}

\maketitle

\newcommand{\mad}[1]{{\color{BrickRed}[\textbf{Mad:} #1]}}
\newcommand{\sam}[1]{{\color{blue}[\textbf{Sam:} #1]}}
\newcommand{\ab}[1]{{\color[RGB]{204,163,0}[\textbf{AB:} #1]}}
\newcommand \draft[1]{{\color{Gray}#1}}
\newcommand \precite[1] {{\color{Orange}[ref. #1]}}

\begin{abstract}%
We estimate \emph{fair graphs from graph-stationary nodal observations} such that connections are not biased with respect to sensitive attributes.
Edges in real-world graphs often exhibit preferences for connecting certain pairs of groups.
Biased connections can not only exacerbate but even induce unfair treatment for downstream graph-based tasks.
We therefore consider group and individual fairness for graphs corresponding to group- and node-level definitions, respectively.
To evaluate the fairness of a given graph, we provide multiple bias metrics, including novel measurements in the spectral domain.
Furthermore, we propose \emph{Fair Spectral Templates (FairSpecTemp)}, an optimization-based method with two variants for estimating fair graphs from stationary graph signals, a general model for graph data subsuming many existing ones.
One variant of FairSpecTemp exploits commutativity properties of graph stationarity while directly constraining bias, while the other implicitly encourages fair estimates by restricting bias in the graph spectrum and is thus more flexible.
Our methods enjoy high probability performance bounds, yielding a \textit{conditional} tradeoff between fairness and accuracy.
In particular, our analysis reveals that \emph{accuracy need not be sacrificed to recover fair graphs}.
We evaluate FairSpecTemp on synthetic and real-world data sets to illustrate its effectiveness and highlight the advantages of both variants of FairSpecTemp.
\end{abstract}


\section{Introduction}
\label{s:intro}

Due to their well-understood properties yet rich modeling capacity, graphs have become a staple tool for interconnected data in many disciplines~\citep{kolaczyk2009statisticalanalysisnetwork,ortega2018graphsignalprocessing,marques2020GraphSignalProcessing,djuric2018cooperativegraphsignal}.
A graph itself may be of interest for analysis, for example, to assess social interactions~\citep{farine2015ConstructingConductingInterpreting}, or it can be used for tasks such as improving recommendations~\citep{duricic2023Beyondaccuracyreviewdiversity,mansoury2022GraphBasedApproachMitigating} or modeling epidemic spread~\citep{achterberg2022ComparingAccuracySeveral}.
In addition to classical fields, graph neural networks (GNNs) and other graph-based machine learning tools have enjoyed ample research to process complex data for challenging tasks~\citep{wu2021ComprehensiveSurveyGraph}.

In many cases, the graph of interest is unavailable, and we must instead build its connections from data that reflects the underlying pairwise relationships~\citep{brugere2019NetworkStructureInference}.
Sensitive yet critical applications include identifying correlations between financial institutions~\citep{franzolini2024ChangePointDetection}, predicting pairwise connections between neurons or genes~\citep{yatsenko2015improvedestimationinterpretation,cai2013InferenceGeneRegulatory}, and tracking movement for contact tracing~\citep{chang2021Mobilitynetworkmodels}.
Some estimate graphs to aid another task, which includes clustering nodes~\citep{berahmand2025comprehensivesurveyspectral} and even improving GNN predictions through interpretable, optimization-based approaches~\citep{tenorio2023RobustGraphNeural,kose2024FilteringRewiringBias}.

Estimating models is essential in data science, yet real-world data is known to encode historical biases that may yield unfair outcomes if incorporated in the estimation process~\citep{chouldechova2017fairpredictiondisparate,lambrecht2019algorithmicbias}.
Graph connections in particular can come with biases that are challenging to address~\citep{lambrecht2019algorithmicbias,yang2024YourNeighborMatters}.
Preferences in edges between certain subpopulations have long been observed in social network analysis~\citep{karimi2018homophilyinfluencesranking,stoica2018algorithmicglassceiling,hofstra2017sourcessegregationsocial,halberstam2016homophilygroupsize,stewart2019informationgerrymanderingundemocratic}, which can limit communication and increase disparate treatment across communities~\citep{nilforoshan2023Humanmobilitynetworks,pariser2011filterbubble,chang2021Mobilitynetworkmodels}.
Even without a proclivity for linking certain groups of nodes, a poorly connected graph may yield unjust resource allocation or information spread~\citep{ogryczak2014FairOptimizationNetworks,bouveret2017FairDivisionGraph,christodoulou2023Fairallocationgraphs}.

Recovering connections from unfair data can lead to discriminatory outcomes, even if the underlying graph is not biased~\citep{navarro2024FairGLASSO,chang2021Mobilitynetworkmodels,rahmattalabi2019Exploringalgorithmicfairness}.
Both edges and nodal data can exhibit unfair behavior, either of which has been found to amplify harmful biases in downstream tasks~\citep{ribeiro2023AmplificationParadoxRecommender,jiang2022FMPFairGraph}.
For example, a social network estimated by tracking human movement may lead to epidemic intervention strategies that discriminate across income levels~\citep{chang2021Mobilitynetworkmodels}.
Unbiased graph estimation remains new and underexplored, although recent progress includes methods for particular settings~\citep{zhou2024Fairnessawareestimationgraphical,zhang2023UnifiedFrameworkFair,moorthy2025EnsuringFairnessSpectral,tarzanagh2023FairCommunityDetection}.
However, fairness has received little attention in graph signal processing (GSP) and statistics~\citep{kose2024FairnessAwareOptimalGraph,navarro2024FairGLASSO,navarro2024MitigatingSubpopulationBias}.
Indeed, the majority of work promoting fairness for graphs lies in machine learning research~\citep{zhang2025FairnessamidstnonIID,dong2023FairnessGraphMining}.
These approaches often require black-box tools, for which it can be challenging to develop theoretical analyses or empirical explanations to relate biases to graph data~\citep{luo2024FUGNNHarmonizingFairness}.
Fair models become untrustworthy if their decisions cannot be understood~\citep{dong2022StructuralExplanationBias}, necessitating the use of principled approaches to mitigate biases in graphs.

We therefore propose to estimate graphs with unbiased connections from nodal observations under the assumption of \emph{graph stationarity}. Graph stationarity can be interpreted in several equivalent ways: as data whose covariance matrix is a polynomial of the graph, as data whose covariance matrix shares the eigenvectors of the graph, or as data generated by diffusing white noise through a linear graph (network) operator~\citep{dong2019LearningGraphsData,mateos2019connectingthedots,marques2017stationarygraphprocesses}. Importantly, graph stationarity is a well-founded and general graph signal model that encompasses many widely used types of networked data, including Gaussian Markov random fields (GMRFs)~\citep{marques2017stationarygraphprocesses}. 
To formalize our goal of fair connectivity, we introduce two notions: \emph{group} fairness, balancing edges across nodal groups in aggregate; and \emph{individual} fairness, encouraging equitable links on a node-by-node basis.
We then present metrics to measure biases from the two perspectives~\citep{navarro2024FairGLASSO,navarro2024MitigatingSubpopulationBias,kose2024FairnessAwareOptimalGraph}, which we relate to the plethora of existing bias metrics.
We further propose novel, alternative formulations to measure unfairness in the graph frequency domain.
Equipped with these metrics, we introduce two optimization-based approaches to recover fair graphs, each founded on graph stationarity with different advantages~\citep{segarra2017networktopologyinference,navarro2022jointinferencemultiple}.
Our contributions are as follows. 


\begin{itemize}[left= 5pt .. 15pt, noitemsep]
    \item[(1)] We present two desirable notions of fairness on graphs from group- and node-level perspectives, or group and individual fairness, respectively, where nodes are partitioned into subpopulations based on sensitive attributes.
    We also explore multiple metrics to quantify graph bias, including novel formulations of bias in the graph spectrum.
    \item[(2)] We propose fair graph estimation from nodal observations through two optimization-based approaches, both founded on graph stationarity.
    We also demonstrate when convex relaxations of the proposed optimization problems maintain the desired solutions.
    \item[(3)] We present high-probability performance bounds in terms of both group and individual fairness, characterizing the tradeoff between fairness and accuracy in graph estimation.
    Importantly, we show that the fairness-accuracy tradeoff depends on the bias of the graph to be estimated, where a fair target graph can preclude a sacrifice in accuracy.
\end{itemize}

\subsection{Related Work}
\label{ss:related}

In this section, we review past work related to fair graph estimation.
This includes how past works define fairness for graphs and how these definitions are incorporated in graph machine learning and GSP.

\subsubsection{Graph Estimation}
\label{sss:nti}

Graph estimation requires that nodal observations can be described using pairwise relationships, usually formulated through an algebraic or statistical model \citep{mateos2019connectingthedots,dong2019LearningGraphsData}.
Statistical methods are among those most common, with staple models including correlation networks and GMRFs~\citep{kolaczyk2009statisticalanalysisnetwork,friedman2008sparseinversecovariance,meinshausen2006highdimensionalgraphsvariable,buciulea2025polynomialgraphicallasso,ying2020NonconvexSparseGraph}.
Others extend the graph estimation task for specific settings such as causal relationships or particular types of dynamical processes on graphs~\citep{rey2025NonnegativeWeightedDAG,cai2013InferenceGeneRegulatory,baingana2017TrackingSwitchedDynamic}.
Methods in GSP pose graph signal models to describe nodal behavior that, in some cases, generalize existing assumptions on networked data, including imposing nodal data that is smooth or low-pass on the graph~\citep{kalofolias2016howlearngraph,dong2016learninglaplacianmatrix,buciulea2022learninggraphssmooth,saboksayr2021AcceleratedGraphLearning}.
Of particular interest to our setting are works that estimate graphs from stationary graph signals, which was originated by \citet{segarra2017networktopologyinference} and has been adapted for several other approaches and settings~\citep{rey2022enhancedgraphlearningschemes,navarro2022jointinferencemultiple,navarro2024jointnetworktopology,shafipour2020OnlineTopologyInference,pasdeloup2017characterizationinferencegraph}.
A small number of papers considered estimating graphs under considerations of fairness, which we discuss further in the sequel.

\subsubsection{Fairness on Graphs}
\label{sss:fairgraphs}

Past works primarily consider fairness for graphs at a node-level, analogous to traditional fairness goals.
In particular, group fairness on graphs promotes balanced treatment of subpopulations of nodes~\citep{hardt2016equalityopportunitysupervised,feldman2015certifyingremovingdisparate}, where it is assumed that the graph may introduce or exacerbate biased outcomes~\citep{rahman2019FairwalkFairGraph,bose2019Compositionalfairnessconstraints}.
In such cases, graph machine learning works typically encourage group fairness for node classification or removing sensitive information from node embeddings~\citep{ma2021Subgroupgeneralizationfairness,lin2024BeMapBalancedmessage,palowitch2020DebiasingGraphRepresentations,kose2023DynamicFairNode,dong2022EDITSModelingMitigating,kose2022FairNodeRepresentation,jiang2022FMPFairGraph,dai2021SayNoDiscrimination}.
Individual fairness is a stricter definition~\citep{dwork2012fairnessthroughawareness}, requiring that each node be treated equitably regardless of group membership, as opposed to only balancing treatment in aggregate.
However, few explicitly consider individual fairness from a graph connectivity perspective~\citep{kose2024FairnessAwareOptimalGraph,kose2023DynamicFairNode}, which we discuss in Section~\ref{ss:indiv_fairness}.

Alternatively, some emphasize dyadic fairness, balancing edges with respect to \emph{pairs} of node groups, discussed in more detail in Section~\ref{ss:group_fairness}.
A prominent application of dyadic fairness is fair link prediction, which shares our goal of promoting equity across node pairs, with applications for recommender systems~\citep{duricic2023Beyondaccuracyreviewdiversity,mansoury2022GraphBasedApproachMitigating,buyl2020DeBayesBayesianmethod,beutel2019FairnessRecommendationRanking}, knowledge graphs~\citep{zhang2023BiasedDebiasedPolarizationaware,bourli2020BiasKnowledgeGraph,shomer2023DegreeBiasEmbeddingBased,fu2020FairnessAwareExplainableRecommendation}, and graph transformers~\citep{luo2024FairGTFairnessawareGraph,luo2025FairGPScalableFaira}.
We next discuss obtention of fair graphs in greater detail.

\subsubsection{Estimating Fair Graphs}
\label{sss:fair_gsp}

Compared to machine learning, fairness for graphs in signal processing and statistics is much rarer, with few but interesting new approaches mitigating bias for graph filters~\citep{kose2024FairnessAwareOptimalGraph,kose2024FilteringRewiringBias}, spectral clustering~\citep{moorthy2025EnsuringFairnessSpectral,zhang2023UnifiedFrameworkFair,tarzanagh2023FairCommunityDetection,gupta2022Consistencyconstrainedspectral}, and graph estimation~\citep{navarro2024FairGLASSO,navarro2024MitigatingSubpopulationBias,zhou2024Fairnessawareestimationgraphical}.
Some consider estimating graphs to cluster nodes fairly~\citep{tarzanagh2023FairCommunityDetection,zhang2023UnifiedFrameworkFair,moorthy2025EnsuringFairnessSpectral}, but, as they are primarily interested in balancing groups across clusters, they do not explicitly encourage unbiased pairwise connections.
Others promote fair connectivity by modifying edges in a given, known graph~\citep{kose2024FilteringRewiringBias,spinelli2023Dropedgesadapt,li2021dyadicfairnessExploring,masrour2020BurstingFilterBubble,ma2022LearningFairNode,agarwal2021unifiedframeworkfair}.
Additionally, fair graph generative models share our goal of creating graphs from data with unbiased edges~\citep{wang2025FGSMOTEFairNode,zheng2024FairGenFairGraph,kose2024FairWireFairgraph,wang2023FG2ANFairnessAwareGraph}, although these require observing at least one graph to construct a distribution of similar, unbiased samples, whereas we recover graphs from observed nodal data, more relevant to certain practical applications.

The works closest to our own estimate connections from nodal data to promote fair graphs~\citep{navarro2024FairGLASSO,navarro2024MitigatingSubpopulationBias,zhou2024Fairnessawareestimationgraphical}.
\citet{navarro2024FairGLASSO} and \citet{zhou2024Fairnessawareestimationgraphical} consider statistical approaches to estimate graphical models, although \citet{zhou2024Fairnessawareestimationgraphical} assign observed samples to groups rather than nodes, which is fundamentally different from our setting.
While pioneering the field of fair graph estimation, these papers consider more restrictive signal models, which are special cases of our assumption.
Our work encompasses the same goal and assumptions as~\citet{navarro2024MitigatingSubpopulationBias}, although we differ in our approach, implementation, and analysis.
First, \citet{navarro2024MitigatingSubpopulationBias} present similar definitions for graph fairness, albeit without much connection to past notions of bias nor our perspectives of bias in the graph spectrum.
Second, we provide two new constrained optimization frameworks for estimating fair graphs from stationary graph signals.
Finally, we analyze the performance of our proposed optimization problems, both in terms of accuracy and fairness, with an explicit characterization of when a tradeoff between the two can occur. 

\section{Preliminaries}
\label{s:prelim}

To formalize our problem and proposed solutions, we first introduce the necessary notation and briefly review key concepts from GSP.

\subsection{Notation}
\label{ss:notation}

For any positive integer $N$, we let $[N] := \{1,2,\dots,N\}$.
We represent matrices and vectors by boldfaced upper-case $\bbX$ and lower-case letters $\bbx$, respectively.
Their entries are denoted by $X_{ij}$ and $x_i$ for indices $i$ and $j$, and rows and columns of a matrix $\bbX$ are denoted by $\bbX_{i,\cdot}$ and $\bbX_{\cdot,j}$, respectively.
Calligraphic letters represent index sets, where $\bar{\ccalC}$ denotes the complement of the index set $\ccalC$.
We let $\bbX_{\ccalC,\cdot}$ ($\bbX_{\cdot,\ccalC}$) be the submatrix of $\bbX$ with rows (columns) indexed by $\ccalC$, and $\bbx_{\ccalC}$ is the subvector of $\bbx$ with entries indexed by $\ccalC$.
The boldfaced numbers $\bbzero$ and $\bbone$ represent vectors of all zeros and ones, respectively, and $\bbI$ denotes the identity matrix.
We also let $\bbe_i$ denote the $i$-th standard basis vector, that is, $\bbe_i := \bbI_{\cdot,i}$.
For a matrix $\bbX \in \reals^{M\times N}$, $\vect(\bbX)\in \reals^{MN}$ returns the concatenation of the columns of $\bbX$.
We denote matrix norms by $\normms{\cdot}$ and vector norms by $\normvs{\cdot}$, and for a matrix $\bbX$, $\normv{\bbX}$ evaluates the vector norm of $\vect(\bbX)$.
The operator $\diag(\bbX) \in \reals^N$ applied to a matrix returns a vector containing the diagonal entries of $\bbX \in \reals^{N \times N}$, while $\diag(\bbx) \in \reals^{N\times N}$ returns a diagonal matrix with $\bbx \in \reals^{N}$ populating the diagonal entries.
We introduce two masking operators for square matrices, $\bbX^- = \diag(\diag(\bbX))$ retaining only diagonal entries of $\bbX$ and $\bbX^+ = \bbX-\bbX^-$ only off-diagonal entries.
We use $O(\cdot)$ for big-O notation and $o(\cdot)$ for little-o notation.
The symbols $\otimes$, $\oplus$, $\odot$, and $\circ$ represent the Kronecker product, the Kronecker sum, the Khatri-Rao product, and the Hadamard (element-wise) product, respectively.

\subsection{Graph Signal Processing}
\label{ss:gsp}

Let $\ccalG = (\ccalV,\ccalE)$ denote a weighted, undirected graph without self-loops, consisting of $N$ nodes collected in $\ccalV$ and edges $\ccalE \subseteq \ccalV \times \ccalV$, where $(i,j) \in \ccalE$ if and only if an edge connects nodes $i,j\in \ccalV$ such that $i\neq j$.
Since we are interested in recovering potentially weighted connections, we rely on the graph-shift operator (GSO) $\bbS \in \reals^{N\times N}$ as a convenient way of representing edges in $\ccalG$~\citep{sandryhaila2013discretesignalprocessing,djuric2018cooperativegraphsignal}, where $S_{ij} \neq 0$ if and only if $(i,j)\in\ccalE$.
Arguably, the most common instantiations of $\bbS$ are the adjacency matrix $\bbA$ and the graph Laplacian $\bbL := \diag(\bbd) - \bbA$~\citep{shuman2013emergingfieldsignal,nt2021RevisitingGraphNeural} for the node degree vector $\bbd = \bbA\bbone$, where $d_i$ contains the sum of edge weights connected to node $i\in\ccalV$.
We define the set of valid adjacency matrices as
\alna{
    \SA
    &~:=~&
    \left\{
        \bbS \in \reals^{N\times N}
        \big|~\!
        \bbS = \bbS^\top, ~
        \bbS^+ \geq \bbzero, ~
        \diag(\bbS) = \bbzero, ~
        \bbS\bbone \geq \bbone
    \right\}.
\label{eq:valid_SA}}
Here, $S_{ij}\neq 0$ represents the weight of the edge $(i,j)$, $\bbS$ has zero-valued diagonal entries as we consider no self-loops, and $\bbS\bbone \geq \bbone$ ensures that every node has at least one edge.
We may alternatively consider $\bbS = \bbL$, which can be considered a discretization of the Laplace-Beltrami operator~\citep{ting2010AnalysisConvergenceGraph}.
The set of graph Laplacian matrices is defined as
\alna{
    \SL
    &~:=~&
    \left\{
        \bbS \in \reals^{N\times N}
        \big|~\!
        \bbS = \bbS^\top, ~
        \bbS^+ \leq \bbzero, ~
        \bbS\bbone = \bbzero, ~
        \diag(\bbS) \geq \bbone
    \right\},
\label{eq:valid_SL}}
so if $\bbS \in \SL$, then there exists some $\bbA \in \SA$ such that $\bbS = \diag(\bbd) - \bbA$, where $\bbd = \bbA\bbone = \diag(\bbS)$.
The majority of our analyses will be demonstrated with $\bbS \in \SA$ as the adjacency matrix for simplicity, but for each result we provide alternative discussions for $\bbS \in \SL$ as the graph Laplacian in Appendix~\ref{app:Lapl}.
Because $\ccalG$ is undirected, $\bbS$ can be diagonalized as $\bbS = \bbV \bbLambda \bbV^\top$, where $\bblambda = \diag(\bbLambda)$ is the ordered vector of eigenvalues, or \emph{graph frequencies}, of $\bbS$, while $\bbV = [\bbv_1,\dots,\bbv_N]$ denotes the eigenvectors of $\bbS$.

The field of GSP aims to process and analyze signals on graphs~\citep{sandryhaila2013discretesignalprocessing,shuman2013emergingfieldsignal,djuric2018cooperativegraphsignal}.
We model graph signals $\bbx\in\reals^N$ as real-valued vectors observed on $\ccalG$, with $x_i$ as the signal value at the $i$-th node.
Then, we interpret the matrix-vector product $\bbS^k \bbx$ as the $k$-hop shift of $\bbx$ over the graph $\ccalG$, where the node signal $x_i$ is propagated over the $k$-hop neighborhood of node $i$.
If we collect multiple weighted shifts of graph signals, we naturally arrive at the definition of \emph{linear graph filters} $\bbH(\bbS) = \sum_{k=0}^{\infty} h_k \bbS^k$ as polynomials of $\bbS$, where $\bbH(\bbS)\bbx$ sums weighted diffusions of the input signal $\bbx$ over the graph $\ccalG$ with respect to the GSO $\bbS$~\citep{sandryhaila2013discretesignalprocessing}.
Linear graph filtering intuitively models many real-world processes on graphs, including heat diffusion and opinion spread~\citep{zhu2020NetworkInferenceConsensus,thanou2017LearningHeatDiffusion}, and they form the foundation of graph convolutional networks~\citep{gama2019ConvolutionalNeuralNetwork}.

Defining linear graph filters leads to the notion of stationary graph signals~\citep{marques2017stationarygraphprocesses,perraudin2017stationarysignalprocessing,girault2015translationongraphs}.
In particular, we assume our data are stationary on $\ccalG$ with respect to the GSO $\bbS$, where each observation can be written as an instantiation of the stochastic output of the linear graph filter $\bbH(\bbS)$ excited by white input $\bbw \in \reals^N$ such that $\mbE[\bbw] = \bbzero$ and $\bbE[\bbw\bbw^\top] = \bbI$.
The covariance matrix of the stationary graph signal $\bbx = \bbH(\bbS)\bbw$ is $\bbC = \mbE[\bbx\bbx^\top] = \bbH(\bbS)\bbH(\bbS)^\top = \bbH(\bbS)^2$.
Thus, $\bbC$ and $\bbS$ share the same eigenvectors~\citep{marques2017stationarygraphprocesses,segarra2017networktopologyinference}.
A critical consequence of the shared eigenbasis is that $\bbC\bbS = \bbS\bbC$, a fact commonly exploited for recovering networks from nodal data~\citep{buciulea2025polynomialgraphicallasso,shafipour2020OnlineTopologyInference}.

We are further interested in promoting fairness in graph connections, where nodes are partitioned based on sensitive information, and connections ought not to depend on the resultant subpopulations~\citep{feldman2015certifyingremovingdisparate,hardt2016equalityopportunitysupervised}.
In particular, each node is associated with one of $G$ groups.
For the $g$-th group, membership is denoted by the indicator vector $\bbz^{(g)} \in \{0,1\}^N$, where $z_i^{(g)} = 1$ if and only if node $i$ belongs to group $g$.
We collect all group labels in the indicator matrix $\bbZ = [\bbz^{(1)},\dots,\bbz^{(G)}] \in \{0,1\}^{N \times G}$.
Groups are non-overlapping, that is, $\sum_{g=1}^G Z_{ig} = 1$ for every $i\in[N]$, where each node belongs to exactly one group.
We represent the number of nodes in each group as $N_g = \sum_{i=1}^N Z_{ig}$, and because groups are non-overlapping, $N = \sum_{g=1}^G N_g$.
Observe that we may collect all group sizes in the vector $\bbZ^\top\bbone = [N_1,\dots,N_G]^\top$.
We let $\Nmin = \min_{g} N_g$ and $\Nmax = \max_g N_g$ denote the sizes of the smallest and largest groups, respectively.
Finally, let $\bbd^{(g)} := \bbd \circ \bbz^{(g)}$ denote the masked degree vector containing only the degrees for nodes in group $g \in [G]$.

\section{Measuring Bias for Graph Estimation}
\label{s:metric}

Our goal is to estimate the GSO $\bbS$ of a graph such that its edges are not dependent on nodal groups.
We thus formalize unbiased connectivity by presenting dyadic notions of fairness.
In particular, we characterize topological bias for both group and individual fairness.
Furthermore, we present a novel perspective by measuring bias in the graph frequency domain, a view that has had very limited consideration~\citep{luo2024FUGNNHarmonizingFairness,luo2024FairGTFairnessawareGraph}.

\subsection{Topological Group Fairness}
\label{ss:group_fairness}

We consider a graph with the GSO $\bbS$ to satisfy \emph{group fairness} if the distribution of the edge between any two nodes $i,j\in\ccalV$ is independent of their groups, where we apply the definition from~\citet{navarro2024FairGLASSO,navarro2024MitigatingSubpopulationBias,dong2023FairnessGraphMining},
\alna{
    \mbP\left[
        S_{ij} \big| z_i^{(g)} = z_j^{(g')} = 1 
    \right]
    =
    \mbP\left[
        S_{ij} \big| z_i^{(h)} = z_j^{(h')} = 1 
    \right]
    {\rm~for~all~}
    g,g',h,h'\in[G].
\nonumber}
By the symmetry of $\bbS$, the following condition is equivalent
\alna{
    \mbP\left[
        S_{ij} \big| z_i^{(g)} = z_j^{(g)} = 1
    \right]
    =
    \mbP\left[
        S_{ij} \big| z_i^{(g)} = z_j^{(h)} = 1
    \right]
    {\rm~for~all~}
    g,h\in[G].
\label{eq:dp_group}}
Defining topological fairness via~\eqref{eq:dp_group} extends the notion of \emph{demographic parity (DP)} to the dyadic setting, where DP requires that the outcome for any entity be independent of its group membership~\citep{feldman2015certifyingremovingdisparate}.
In particular, if~\eqref{eq:dp_group} holds, then the distribution of any edge is invariant to nodal groups.
Dyadic DP is the primary goal for fairness on graphs, natural for link prediction~\citep{subramonian2024NetworkedinequalityPreferential,beutel2019FairnessRecommendationRanking} and highly relevant for social network analysis~\citep{saxena2024fairsna,yang2024YourNeighborMatters}.
Many past works define fairness similarly to~\eqref{eq:dp_group}~\citep{yang2022ObtainingDyadicFairness,li2021dyadicfairnessExploring,liu2024Promotingfairnesslink,buyl2021KLDivergenceGraphModel}.
However, most consider a weaker definition, only requiring independence for edges connecting nodes in the same group versus two different groups, which does not ensure that edge distributions are invariant to group labels; for example, under this definition we may still have $\mbP[S_{ij} | z_i^{(g)}=z_j^{(g)}=1] \neq \mbP[S_{ij} | z_i^{(h)} = z_j^{(h)} = 1]$ for some $g\neq h$.
We instead align with those that prohibit preferences toward any group pair~\citep{pal2024FairLinkPrediction,rahman2019FairwalkFairGraph,navarro2024FairGLASSO,navarro2024MitigatingSubpopulationBias}.
While we consider unsigned graphs in~\eqref{eq:valid_SA} and~\eqref{eq:valid_SL}, \eqref{eq:dp_group} does not necessitate unsigned edges~\citep{navarro2024FairGLASSO,saxena2024fairsna}, whereas a large number of previous works consider not only unsigned but unweighted edges.

\subsubsection{Group-level Bias in Graph Spatial Domain}
\label{sss:group_bias_spatial}

To assess violation of~\eqref{eq:dp_group} in practice, we measure differences in the expected edge weights connecting nodes from each pair of groups, defining our group-wise dyadic DP metric as
\alna{
    \Rgr(\bbS)
    &~=~&
    \frac{1}{G^2-G}
    \sum_{g\neq h}
    \left(
        \frac{ \bbz^{(g)\top} \bbS^+ \bbz^{(g)} }{ N_g^2-N_g }
        -
        \frac{ \bbz^{(g)\top} \bbS^+ \bbz^{(h)} }{ N_g N_h }
    \right)^2,
\label{eq:Rgr}}
where we recall that $\bbS^+$ is a version of $\bbS$ with the diagonal entries set to zero. In words, $\Rgr(\bbS)$ measures the squared difference between the average edge connecting node pairs in the same group versus in different groups, ignoring the  diagonal entries.

As with our definition of group fairness in~\eqref{eq:dp_group}, similar metrics to $\Rgr(\bbS)$ in~\eqref{eq:Rgr} have been applied in other graph-based works to measure the violation of dyadic DP.
Note that $\Rgr(\bbS)$ was considered in~\citep{navarro2024FairGLASSO}, which is a modification of the group-wise metric in~\citep{navarro2024MitigatingSubpopulationBias}.
In terms of dyadic fairness, $\Rgr(\bbS)$ aligns with metrics that balance edge probabilities or link prediction scores across group pairs~\citep{li2022FairLPFairLink,yang2022ObtainingDyadicFairness,liu2024Promotingfairnesslink,luo2023Crosslinksmatterlink,khajehnejad2022CrossWalkFairnessEnhancedNode}. 
Moreover, some measures of bias in graph clustering tasks can be interpreted as quantifying topological bias~\citep{liu2023DualNodeEdge,moorthy2025EnsuringFairnessSpectral}.
Finally, while conceptually similar, some dyadic bias metrics are specific to link prediction by promoting equitable prediction accuracy, which requires known ground truth graph connectivity that we do not have for the task of graph estimation~\citep{pal2024FairLinkPrediction,spinelli2023Dropedgesadapt,saxena2022HMEIICTFairnessawarelink}.

\subsubsection{Group-level Bias in Graph Spectral Domain}
\label{sss:group_bias_spectral}

As we consider undirected edges, our GSO is diagonalizable $\bbS=\bbV\bbLambda\bbV^\top$.
Therefore, we also propose to investigate bias in terms of the spectrum of $\bbS$, which captures relevant information about graph structure and the behavior of graph signals~\citep{dabush2024VerifyingSmoothnessGraph,nt2021RevisitingGraphNeural}.
First, let $\tbz^{(g)} := \bbV^\top \bbz^{(g)}$ denote the \emph{frequency response} of the group indicator vector $\bbz^{(g)}$ in terms of $\bbS$ for every $g\in[G]$, which indicates how the $g$-th group is distributed throughout the graph.
For example, if $\bbS \in \SL$ represents the graph Laplacian and group $g$ tends to have more \emph{within-group connections}, then $\tbz^{(g)}$ has a higher concentration at \emph{lower frequencies}.
We can then equivalently express $\Rgr(\bbS)$ from~\eqref{eq:Rgr} as
\alna{
    \Rgr(\bbS)
    &~=~&
    \frac{1}{G^2-G}
    \sum_{g\neq h}
    \left(
        \frac{ \bbz^{(g)\top} (\bbS - \bbS^-) \bbz^{(g)} }{ N_g^2-N_g }
        -
        \frac{ \bbz^{(g)\top} \bbS \bbz^{(h)} }{ N_g N_h }
    \right)^2
&\nonumber\\&
    &~=~&
    \frac{1}{G^2-G}
    \sum_{g\neq h}
    \left(
        \frac{ \tbz^{(g)\top} \bbLambda \tbz^{(g)} }{ N_g^2-N_g }
        -
        \frac{ \tbz^{(g)\top} \bbLambda \tbz^{(h)} }{ N_g N_h }
        -
        \frac{ \bbz^{(g)\top} \diag(\bbS) }{ N_g^2-N_g }
    \right)^2
&\nonumber\\&
    &~=~&
    \frac{1}{G^2-G}
    \sum_{g\neq h}
    \left[
        \bblambda^\top
        \left(
            \frac{ \tbz^{(g)} \circ \tbz^{(g)} }{ N_g^2-N_g }
            -
            \frac{ \tbz^{(g)} \circ \tbz^{(h)} }{ N_g N_h }
            -
            \frac{ (\bbV \circ \bbV)^\top \bbz^{(g)} }{ N_g^2-N_g }
        \right)
    \right]^2.
\label{eq:Rgr_spec_A}}
With some abuse of notation, if $\bbV$ is known, we may write $\Rgr(\bbS) = \Rgr(\bblambda)$.
With~\eqref{eq:Rgr_spec_A}, we can interpret how the graph frequencies $\bblambda$ and the frequency responses $\tbz^{(g)}$ interact to produce bias in $\bbS$.
First, observe that if $\bbS \in \SA$, then $\bblambda^\top (\bbV\circ\bbV)^\top \bbz^{(g)} = \bbz^{(g)\top}\diag(\bbS) = \bbzero$.
In this case, for a given pair of distinct groups $g\neq h$, $\Rgr(\bbS)$ increases with the deviation between $\tbz^{(g)}$ and $\tbz^{(h)}$, with greater effect at higher magnitudes of $\tbz^{(g)}$ and frequencies $\bblambda$.
Thus, it is more critical for groups $g$ and $h$ to show similar behavior at higher frequencies when $\bbS \in \SA$.
Second, if $\bbS \in \SL$, then $\bblambda^\top(\bbV\circ\bbV)^\top \bbz^{(g)} = \bbone^\top \bbd^{(g)}$ sums the degrees of nodes in group $g$, so the difference between frequency responses $\tbz^{(g)}$ and $\tbz^{(h)}$ for $g\neq h$ must account for the differences in how well each group is connected throughout the graph.
In particular, we have that $\Rgr(\bbS) \approx 0$ if, for all $g,h\in[G]$ such that $g\neq h$,
\alna{
    (\bblambda \circ \tbz^{(g)})^\top \left( \frac{\tbz^{(g)}}{ N_g^2-N_g } - \frac{\tbz^{(h)} }{ N_g N_h } \right)
    =
    \sum_{i=1}^N
    \frac{\lambda_i \tilde{z}^{(g)}_i}{N_g}
    \left(
        \frac{ \tilde{z}^{(g)}_i }{ N_g-1 }
        -
        \frac{ \tilde{z}^{(h)}_i }{ N_h }
    \right)
    \approx
    \frac{ \bbone^\top\bbd^{(g)} }{ N_g^2-N_g }.
\label{eq:Rgr_spec_interp}}
Observe that as $\tilde{z}_i^{(g)}$ and $\tilde{z}_i^{(h)}$ increasingly differ, particularly at high frequencies $\lambda_i$, $\bbone^\top\bbd^{(g)}$ must also increase.
More intuitively, even if group $g$ shows greater across-group connectivity than group $h$, we may still have $\Rgr(\bbS)\approx 0$ if group $g$ is densely connected.
Thus, \eqref{eq:Rgr_spec_interp} shows that differences in edge densities across groups can account for discrepancies in within- versus across-group connectivity patterns.
This is expected since minimizing $\Rgr(\bbS)$ balances the \emph{average} edge weights across all pairs of groups.
However, the left-hand side of~\eqref{eq:Rgr_spec_interp} can be negative if any pair of densely connected groups differs too greatly at high frequencies, and then we cannot achieve $\Rgr(\bbS)\approx 0$ since the right-hand side of~\eqref{eq:Rgr_spec_interp} is strictly positive.

\subsection{Topological Individual Fairness}
\label{ss:indiv_fairness}

While group fairness from~\eqref{eq:dp_group} is desirable, we may require a stricter, node-wise condition for fair connections by ensuring that each node has an equal likelihood of connecting to every group.
In particular, this yields a node-wise notion of dyadic DP~\citep{navarro2024FairGLASSO,navarro2024MitigatingSubpopulationBias}
\alna{
    \mbP\left[
        S_{ij} \big| z_j^{(g)} = 1
    \right]
    =
    \mbP\left[
        S_{ij} \big| z_j^{(h)} = 1
    \right]
    {\rm~for~all~}
    g,h\in[G],
    i\in[N].
\label{eq:dp_node}}
Observe that node-wise DP in~\eqref{eq:dp_node} differs from the group-wise definition in~\eqref{eq:dp_group} in that the condition for fair edges in~\eqref{eq:dp_node} for a given node $i \in [N]$ does not depend on its own group.
This concept is more in line with definitions of fairness that require equitable outcomes on a node-by-node basis~\citep{rahman2019FairwalkFairGraph,current2022FairEGMFairLink,chen2022GraphLearningLocalized,liu2023GeneralizedDegreeFairness}, also known as individual fairness~\citep{dwork2012fairnessthroughawareness}.
However, most node-wise, individual fairness notions focus on how an external model treats or represents nodes~\citep{dong2021IndividualFairnessGraph,kang2020InFoRMIndividualFairness,sium2024IndividualFairnessGraphs,song2022GUIDEGroupEquality,tsioutsiouliklis2021FairnessAwarePageRank}, while our definition~\eqref{eq:dp_node} is more related to those that consider unbiased connections~\citep{kose2024FairnessAwareOptimalGraph,kose2024FilteringRewiringBias}.
In fact, the definition of individually fair balanced node clusters by~\citet{gupta2022Consistencyconstrainedspectral} is analogous to~\eqref{eq:dp_node} if node clusters are defined by group membership.
\subsubsection{Node-level Bias in Graph Spatial Domain}
\label{sss:node_bias_spatial}

Similarly to $\Rgr$, we consider the following node-wise DP metric for~\eqref{eq:dp_node}
\alna{
    \Rno(\bbS)
    &~=~&
    \frac{1}{GN}
    \sum_{g=1}^G
    \sum_{i=1}^N
    \left(
        \frac{ [\bbS^+ \bbz^{(g)}]_i }{ N_g }
        -
        \sum_{h\neq g}
        \frac{ [\bbS^+ \bbz^{(h)}]_i }{ N_h (G-1) }
    \right)^2,
\label{eq:Rno}}
which measures any imbalance in how each node connects across groups.
To see this, observe that for terms in~\eqref{eq:Rno} corresponding to the $i$-th node, $\Rno(\bbS)$ sums differences in the average edge weights connecting node $i$ to different groups.
This metric serves as a measure of bias for individual fairness, where $\Rno(\bbS)=0$ only when the edge weights for each node are equally distributed across all groups.
This differs from metrics that promote similar node-level treatment based on centrality or embedding similarity~\citep{lahoti2019Operationalizingindividualfairness,jia2024Aligningrelationallearning,wang2024IndividualFairnessGroup}.
Indeed, $\Rno(\bbS)$ is closely related to measuring the linear correlation between group membership and neighborhoods of nodes~\citep{kose2024FairnessAwareOptimalGraph,kose2024FilteringRewiringBias}, which was adapted by \citet{navarro2024MitigatingSubpopulationBias} to yield~\eqref{eq:Rno} to ensure that nodes do not show preferences for connecting to certain groups.
Other analogous concepts balance connections at a node-by-node level, such as equitable transition probabilities across groups for random walks on graphs~\citep{wang2023FG2ANFairnessAwareGraph,wang2024AdvancingGraphCounterfactual,arnaiz-rodriguez2025StructuralGroupUnfairness}.
For example, the neighborhood fairness metric by~\citet{lee2025DisentanglingAmplifyingDebiasinga} measures group-wise entropy of the neighborhood of each node, indirectly encouraging more diversity in node-wise connections.

\subsubsection{Node-level Bias in Graph Spectral Domain}
\label{sss:node_bias_spectral}

Similar to $\Rgr(\bbS)$ in Section~\ref{sss:group_bias_spectral}, we can express the node-wise bias $\Rno(\bbS)$ with respect to the spectrum of $\bbS$.
As before, we let $\tbz^{(g)} = \bbV^\top \bbz^{(g)}$ be the frequency response of the group indicator vector $\bbz^{(g)}$ for each $g\in[G]$, and we write $\Rno(\bbS)$ from~\eqref{eq:Rno} as
\alna{
    \Rno(\bbS)
    &\,=\,&
    \frac{1}{GN(G-1)^2}
    \sum_{g=1}^G
    \left\| 
        \bbS^+
        \left[
            \sum_{h \neq g}
            \frac{ \bbz^{(g)} }{ N_g }
            -
            \frac{ \bbz^{(h)} }{ N_h }
        \right]
    \right\|_2^2
&\nonumber\\&
    &\,=\,&
    \frac{1}{GN(G-1)^2}
    \sum_{g=1}^G
    \left\| 
    \big( \bbLambda \bbV^\top - \bbV^\top\diag( (\bbV\circ\bbV) \bblambda ) \big)
        \left[
            \sum_{h \neq g}
            \frac{ \bbz^{(g)} }{ N_g }
            -
            \frac{ \bbz^{(h)} }{ N_h }
        \right]
    \right\|_2^2
&\nonumber\\&
    &\,=\,&
    \frac{1}{GN(G-1)^2}
    \sum_{g=1}^G
    \left\| 
        \sum_{h \neq g}
        \bblambda^\top
        \!\!
        \left[
            \diag \! \left(
                \!
                \frac{ \tbz^{(g)} }{ N_g } - \frac{ \tbz^{(h)} }{ N_h }
                \!
            \right)
            \!
            -
            (\bbV\circ\bbV)^\top
            \diag \! \left(
                \!
                \frac{ \bbz^{(g)} }{ N_g } - \frac{ \bbz^{(h)} }{ N_h }
                \!
            \right)
            \!
            \bbV
        \right]
        \!
    \right\|_2^2.
\label{eq:Rno_spec_A}}
We again use the notation $\Rno(\bbS) = \Rno(\bblambda)$ when the eigenvectors $\bbV$ are known.
We first observe that for $\bbS \in \SA$, since $(\bbV\circ\bbV)\bblambda = \diag(\bbS)=\bbzero$, the bias $\Rno(\bbS)$ increases as the frequency response $\tbz^{(g)}$ differs from those of the remaining groups $\tbz^{(h)}$.
This effect increases for higher frequencies in $\bblambda$ as for~\eqref{eq:Rgr_spec_A}, but differently, $\Rno(\bbS)$ requires that entries of $\tbz^{(g)}$ and $\tbz^{(h)}$ exhibit similar behavior rather than only requiring similarity in aggregate.
Moreover, for $\bbS \in \SL$, as $(\bbV\circ\bbV)\bblambda \circ \bbz^{(g)} = \diag(\bbS)\circ\bbz^{(g)} = \bbd^{(g)}$ represents group-wise node degrees, a small $\Rno(\bbS)$ requires that differences in frequency responses be similar to the frequency response of differences in node degrees across groups, that is, $\Rno(\bbS) \approx 0$ if
\alna{
    \bblambda \circ \left(\frac{\tbz^{(g)}}{N_g} - \frac{\tbz^{(h)}}{N_h}\right)
    &~\approx~&
    \bbV^\top \left( \frac{\bbd^{(g)}}{N_g} - \frac{\bbd^{(h)}}{N_h} \right)
\nonumber}
for all $g,h\in[G]$ such that $g\neq h$.
Thus, differences in $\tbz^{(g)}$ and $\tbz^{(h)}$ must account for discrepancies in the distributions of connections $\bbV^\top\bbd^{(g)}$ and $\bbV^\top\bbd^{(h)}$, which aligns with our goal of balancing how edges are distributed across groups for each node.

\begin{myremark}\label{rem:alternative_bias_metrics}
    The proposed metrics $\Rgr$ and $\Rno$ extend classical notions of group fairness, particularly DP, to dyadic settings.
    Other definitions such as equalized odds (EO) or balanced performance can be similarly adapted~\citep{li2021dyadicfairnessExploring,singer2022EqGNNEqualizedNode,buyl2021KLDivergenceGraphModel,dong2023InterpretingUnfairnessGraph,liu2024Promotingfairnesslink,pal2024FairLinkPrediction,saxena2022HMEIICTFairnessawarelink}.
    However, these metrics require prior knowledge of graph structure, which is unlikely for inferring a network.
    While it is often feasible to have a limited subset of known edges, the distribution of edges across group pairs is likely to be poorly represented by this known subset, yielding an inaccurate approximation of bias. 
    Thus, we promote DP for balancing edge connections across groups, which is known to be a prominent source of bias in GNNs and social network analysis~\citep{saxena2024fairsna,dong2022EDITSModelingMitigating}.
    One limitation is that these dyadic extensions of DP, EO, and other definitions do not directly account for bias beyond immediate neighborhoods~\citep{han2024MarginalNodesMatter,zhang2023PrerequisitedrivenFairClustering,wang2023FG2ANFairnessAwareGraph}.
    Since measuring the fairness of connections within a $k$-hop neighborhood requires powers of the GSO $\bbS^k$~\citep{jalali2020informationunfairnesssocial}, we leave this for future work.
\end{myremark}

\section{Fair Network Topology Inference}
\label{s:method}

Next, we consider the task of recovering an unknown graph $\ccalG$ such that our estimate has fair connections in terms of either group fairness in~\eqref{eq:dp_group} or individual fairness in~\eqref{eq:dp_node}.
To this end, we infer the GSO $\bbS$ of $\ccalG$ from a data matrix $\bbX := [\bbx_1,\dots,\bbx_M] \in \reals^{N\times M}$ whose columns comprise $M$ graph signals that are stationary on $\ccalG$ with respect to $\bbS$.
As in Section~\ref{ss:gsp}, stationarity implies that the graph signal covariance matrix $\bbC$ and the GSO $\bbS$ share the same eigenvectors $\bbV$.
The eigenvectors of $\bbS$ following the \emph{spectral template} $\bbV$ given by $\bbC$ implies that $\bbC\bbS = \bbS\bbC$, which is a more tractable condition.
Since infinitely many valid GSOs share the eigenvectors of and thus commute with $\bbC$~\citep{segarra2017networktopologyinference}, we address the inference of $\bbS$ by obtaining a target GSO with desirable structural properties~\citep{rey2022enhancedgraphlearningschemes,zhang2025NetworkGamesInduced}.
In particular, our goal is to obtain the sparsest GSO satisfying the conditions of graph-stationarity~\citep{segarra2017networktopologyinference,marques2017stationarygraphprocesses}, that is, we let our target GSO $\bbS^*$ satisfy
\alna{
    \bbS^* \in~
    &\argmin_{\bbS \in \ccalS}& ~~
    \| \bbS^+ \|_0
    ~~{\rm s.t.}~~
    \bbC\bbS = \bbS\bbC,
\label{eq:opt_C_target}}
where $\ccalS$ enforces valid GSO, such as adjacency matrices with $\ccalS = \SA$ in~\eqref{eq:valid_SA} or graph Laplacians with $\ccalS = \SL$ in~\eqref{eq:valid_SL}.
Because $\bbC$ and any $\bbS \in \ccalS$ are symmetric, they are both diagonalizable, so the equality constraint in~\eqref{eq:opt_C_target} enforces shared eigenvectors between the target GSO $\bbS^*$ and the covariance matrix $\bbC$, while the $\ell_0$ (pseudo)norm encourages minimal nonzero entries in $\bbS^*$.
We thus define our target $\bbS^*$ as perfectly describing stationary graph signals via $\bbC \bbS^* = \bbS^* \bbC$ while having the most parsimonious representation, allowing for interpretable analyses and reduced downstream computational complexity.

\subsection{FairSpecTemp via Convex Optimization}
\label{ss:fst_commut}

While our goal is to obtain $\bbS^*$, we encounter two issues in practice: first, we typically do not have access to the true covariance matrix $\bbC$, and second, we require a GSO that is fair, which does not violate a level of permissible bias in our estimated graph.
To address these concerns, we instead solve an approximation of the problem~\eqref{eq:opt_C_target} given the sample covariance $\hbC := \frac{1}{M}\bbX\bbX^\top$ computed from the data matrix $\bbX\in\reals^{N\times M}$.
We present our approach \emph{Fair Spectral Templates (FairSpecTemp)} as the following optimization problem
\alna{
    \bbSCp \in
    \argmin_{\bbS \in \ccalS} ~~
    \| \bbS^+ \|_0
    ~~{\rm s.t.}~~
    \| \hbC\bbS - \bbS\hbC \|_F \leq \epsilon, ~~
    R(\bbS) \leq \tau^2.
\label{eq:opt_C_l0}}
Observe that $\tau\geq 0$ bounds the permitted structural bias measured by $R$, where we emphasize our proposed bias metrics $R \in \{ \Rgr, \Rno \}$; however, the constraint in~\eqref{eq:opt_C_l0} is adaptable to other metrics such as those in~\citep{wang2023FG2ANFairnessAwareGraph}.
Henceforth, we write $R$ to denote either $\Rgr$ or $\Rno$, and we specify the particular metric when needed.
The first constraint in~\eqref{eq:opt_C_l0} encourages $\bbSCp$ to exhibit graph stationarity given $\hbC$ by requiring that $\bbSCp$ and $\hbC$ approximately commute~\citep{segarra2017networktopologyinference,navarro2022jointinferencemultiple}.

In practice, we instead solve a convex relaxation of~\eqref{eq:opt_C_l0}
\alna{
    \hbSC \in
    \argmin_{\bbS \in \ccalS} ~~
    \| \bbS^+ \|_1
    ~~{\rm s.t.}~~
    \| \hbC\bbS - \bbS\hbC \|_F \leq \epsilon, ~~
    R(\bbS) \leq \tau^2,
\label{eq:opt_C_l1}}
where~\eqref{eq:opt_C_l1} differs from~\eqref{eq:opt_C_l0} by replacing the non-convex $\ell_0$ norm with a convex $\ell_1$ norm.
While~\eqref{eq:opt_C_l1} tends to be amenable to faster, simpler algorithms and yields uniqueness guarantees due to its convexity, it still involves a relaxation of~\eqref{eq:opt_C_l0}.
However, we next provide sufficient conditions under which problems~\eqref{eq:opt_C_l0} and~\eqref{eq:opt_C_l1} yield the same solution~\citep{zhang2016OneConditionSolution,navarro2022jointinferencemultiple}.
To this end, we first consider vectorized versions of the constraints in problems~\eqref{eq:opt_C_l0} and~\eqref{eq:opt_C_l1}.
This allows us to apply guarantees for the equivalence of solutions to the problems in vectorized form~\citep{zhang2016OneConditionSolution}, thus showing that equivalent solutions exist for the original matrix-based problems~\eqref{eq:opt_C_l0} and~\eqref{eq:opt_C_l1}.

First, we let $\ccalD$, $\ccalL$ and $\ccalU$ denote index sets of diagonal, lower triangular, and upper triangular entries of any $N\times N$ matrix, that is,
\alna{
    \ccalD := \{ N(i-1) + i \, \big| ~\! i\in[N] \},
&\nonumber\\&
    \ccalL := \{ N(i-1) + j \, \big| ~\! i,j\in[N], i<j \},
&\nonumber\\&
    \ccalU := \{ N(j-1) + i \, \big| ~\! i,j\in[N], i<j \}.
\nonumber}
Then, our vector-valued optimization variable of interest is $\bbs = \vect(\bbS)_{\ccalL} \in \reals^{N(N-1)/2}$, containing the lower triangular entries of $\bbS$, which completely characterizes the GSO $\bbS$ for $\ccalS \in \{ \SA, \SL \}$.
To see this, observe that since $\bbS$ is symmetric, $\bbs = \vect(\bbS)_{\ccalL} = \vect(\bbS)_{\ccalU}$.
Furthermore, if we let $\bbU := \bbI_{\cdot,\ccalL}+\bbI_{\cdot,\ccalU}$ with $\bbI$ as the $N^2 \times N^2$ identity matrix, then $\bbB \vect(\bbS) = \bbB\bbU\bbs$ for any $\bbB$ with $N^2$ columns.
If $\ccalS=\SA$, $\vect(\bbS)_{\ccalD}=\bbzero$, and if $\ccalS=\SL$, $\vect(\bbS)_{\ccalD} = -\bbE\bbs$ for $\bbE := (\bbone^\top \otimes \bbI)\bbU$.
Thus, we can recover any entry of $\bbS$ from $\bbs$ for $\SA$ or $\SL$.

For the remainder of this section, we show our theoretical result for adjacency matrices, so we set $\ccalS = \SA$.
However, a similar analysis for graph Laplacian GSOs, that is, $\ccalS=\SL$, is shown in Appendix~\ref{app:Lapl_convex}. Reminding that $\oplus$ stands for the Kronecker sum, we define the matrix $\hbSigma := (\hbC \oplus (-\hbC))\bbU$, so that $\| \hbC\bbS-\bbS\hbC \|_F = \| \hbSigma \bbs \|_2$.
We define a matrix $\bbRgr := [ \check{\bbr}^{(1,2)}_{\rm \scriptscriptstyle G}, \dots, \check{\bbr}^{(G-1,G)}_{\rm \scriptscriptstyle G} ]^\top \in \reals^{(G^2-G) \times (N^2-N)/2}$ such that $\Rgr(\bbS) = \| \bbRgr \bbs \|_2^2$, where the rows of $\bbRgr$ comprise the vectors
\alna{
    \check{\bbr}^{(g,h)}_{\rm \scriptscriptstyle G}
    &:=~&
    \frac{1}{\sqrt{ G^2-G }}
    \cdot
    \bbU^\top
    \left(
        \frac{ \bbz^{(g)} \otimes \bbz^{(g)} }{ N_g^2-N_g }
        -
        \frac{ \bbz^{(g)} \otimes \bbz^{(h)} }{ N_g N_h }
    \right)
    \in\reals^{ (N^2-N)/2 }
    \quad \forall ~ g,h\in[G], ~ g\neq h.
\nonumber}
Similarly, we define $\bbRno := [ \check{\bbR}_{\rm \scriptscriptstyle N}^{(1)}, \dots, \check{\bbR}_{\rm \scriptscriptstyle N}^{(G)} ]^\top \in \reals^{GN \times (N^2-N)/2}$ such that $\Rno(\bbS) = \| \bbRno\bbs \|_2^2$ by concatenating the following $G$ matrices
\alna{
    \check{\bbR}^{(g)}_{\rm \scriptscriptstyle N}
    &~:=~&
    \frac{1}{(G-1)\sqrt{GN}}
    \cdot
    \bbU^\top
    \left(
        \left[
            \sum_{h\neq g} 
            \frac{\bbz^{(g)}}{N_g} -  \frac{\bbz^{(h)}}{N_h}
        \right]
        \otimes \bbI
    \right)
    \in \reals^{ (N^2-N)/2 \times N }
    \quad \forall ~ g\in[G].
\nonumber}
For $R \in \{ \Rgr, \Rno \}$, we let $\bbR$ denote the corresponding matrix $\bbRgr$ or $\bbRno$, which we specify as needed.
We next share the following result guaranteeing that a solution to the non-convex problem~\eqref{eq:opt_C_l0} is the unique solution to the convex problem~\eqref{eq:opt_C_l1}.

\begin{mytheorem}\label{thm:C_convex}
    If problems~\eqref{eq:opt_C_l0} and~\eqref{eq:opt_C_l1} are feasible for $\ccalS = \SA$, then for $\bbSCp$ as a solution to~\eqref{eq:opt_C_l0}, we have that $\hbSC = \bbSCp$ is the unique solution to~\eqref{eq:opt_C_l1} if
    \begin{itemize}[left= 15pt .. 24pt, itemsep=1pt]
    \im[(i)] The submatrix $\hbSigma_{\cdot,\ccalI}$ is full column rank, and
    \im[(ii)] There exists a constant $\psi > 0$ such that
    \alna{
        \normms{
            \left(
                \psi^{-2} \bbPhi^\top \bbPhi + \bbI_{\cdot,\bar{\ccalI}} \bbI_{\bar{\ccalI},\cdot}
            \right)^{-1}_{\bar{\ccalI},\ccalI} 
        }_{\infty}
        ~<~
        1,
    \label{eq:thm_C_cvx_cond}}
    \end{itemize}
    where $\ccalI = {\rm supp}(\bbs')$ for $\bbs' = \vect(\bbSCp)_{\ccalL}$, and $\bbPhi := [ \hbSigma^\top, \bbR^\top, \bbE^\top ]^\top$.
\end{mytheorem}
The proof is in Appendix~\ref{app:thm_C_convex}, which is based on that of~\citep[Theorem 1]{segarra2017networktopologyinference} and employs~\citep[Theorem 1]{zhang2016OneConditionSolution}.
Condition \textit{(ii)} guarantees that $\bbSCp$ is a solution to~\eqref{eq:opt_C_l1} by verifying the existence of a solution to the dual problem of~\eqref{eq:opt_C_l1}, ensuring the existence of its associated primal solution $\bbSCp$.
Moreover, condition \textit{(i)} guarantees that it is unique, as a full column rank $\hbSigma_{\cdot,\ccalI}$ requires that there can be no other solution minimizing the $\ell_1$ norm in~\eqref{eq:opt_C_l1}.
Thus, under the conditions of Theorem~\ref{thm:C_convex}, we may solve the convex problem~\eqref{eq:opt_C_l1}, which is also a solution to~\eqref{eq:opt_C_l0}.
As our goal is to obtain the sparsest graph satisfying the commutativity constraint, this result is theoretically satisfactory even if~\eqref{eq:opt_C_l0} has more than one solution.
In practice, there is likely to be a single ground truth graph to be obtained, so while we can guarantee that the unique solution to~\eqref{eq:opt_C_l1} also minimizes the $\ell_0$ norm, we cannot ensure that it approximates the ground truth GSO if~\eqref{eq:opt_C_l0} has more than one solution.
However, we empirically find that solving a convex relaxation of~\eqref{eq:opt_C_l0} performs well on both synthetic and real-data simulations in Section~\ref{s:results}.
An analogous result to Theorem~\ref{thm:C_convex} can be found for $\ccalS = \SL$ in Appendix~\ref{app:Lapl_convex}.

\subsection{FairSpecTemp Performance Guarantees}
\label{ss:theory}

Our goal is to recover the target GSO $\bbS^*$ defined in~\eqref{eq:opt_C_target} by solving~\eqref{eq:opt_C_l0}; however, under the conditions of Theorem~\ref{thm:C_convex}, we can instead solve the convex relaxation~\eqref{eq:opt_C_l1} and still obtain the desired sparse estimate of $\bbS^*$.
Thus, this section presents our main results on the performance of FairSpecTemp in terms of error bounds between the estimate $\hbSC$ from~\eqref{eq:opt_C_l1} and the target GSO $\bbS^*$ from~\eqref{eq:opt_C_target}.
We again restrict to adjacency matrix GSOs $\ccalS = \SA$, but adapting the result for graph Laplacians $\ccalS = \SL$ is straightforward and only changes the error bounds up to a scaling, which we explain in Appendix~\ref{app:Lapl_error}.
We require the following assumption, analogous to those considered in~\citep{navarro2022jointinferencemultiple,segarra2017networktopologyinference}.

\begin{myassumption}\label{assump:error_bound}
    The sample covariance matrix $\hbC = \frac{1}{M} \bbX\bbX^\top$ is obtained from $M$ graph signals $\bbX = [\bbx_1,\dots,\bbx_M] \in\reals^{N\times M}$ as diffusions of zero-mean Gaussian white noise over the graph $\ccalG$ with GSO $\bbS^*$, where the $m$-th graph signal is defined by $\bbx_m = \bbH(\bbS^*)\bbw_m$ for graph filter $\bbH(\bbS^*)$ and $\bbw_m \sim \ccalN(\bbzero,\bbI)$.
    We let
    \begin{itemize}[left= 16pt .. 26pt, itemsep=1pt]
    \im[(A1)]
        $\log N = o( M^{1/3} )$,
    \im[(A2)]
        $\epsilon \geq c_1 N\omega \sqrt{ \frac{\log N}{M} }$ for $\omega:= \max\{ \normm{\bbC^-}_{\infty}, \normm{ (\bbS^* \bbC \bbS^*)^- }_{\infty} \}$ and some $c_1>0$,
    \im[(A3)]
        $\hbSigma=(\hbC\oplus (-\hbC))\bbU$ is full column rank with a smallest singular value $\sigma_{\scriptscriptstyle\min}(\hbSigma)>0$, and
    \im[(A4)] $k \geq \sqrt{ \| \vect(\bbS^*)_{\ccalL} \|_0 }$ for some $k\geq 0$. 
    \end{itemize}
\end{myassumption}
We establish graph-stationarity, our setting of interest, by assuming our data are outputs of a linear graph filter applied to zero-mean white noise, as discussed in Section~\ref{ss:gsp}.
The remaining items let us connect~\eqref{eq:opt_C_l1} to performance by ensuring a sufficiently feasible task.
Assumption \textit{(A1)} imposes a reasonable rate of growth of the number of samples $M$ relative to the number of nodes $N$, while \textit{(A2)} allows us to characterize how estimation error increases with $N$ and decreases with $M$.
We require \textit{(A3)} to relate the minimization of the $\ell_1$ norm in the objective of~\eqref{eq:opt_C_l1} to the commutativity constraint, which yields the bound in estimation error.
Finally, observe that \textit{(A4)} merely introduces a definition for the sparsity of $\bbS^*$, as the choice of $k$ is arbitrary.
Our first main result characterizes the performance of FairSpecTemp when considering group fairness $\Rgr$.

\begin{mytheorem}\label{thm:err_bnd_gr}
    Under Assumption~\ref{assump:error_bound}, consider the GSO estimate $\hbSC \in \reals^{N\times N}$ as the solution to~\eqref{eq:opt_C_l1} with the group-wise bias $R = \Rgr$ bounded above by $\tau \geq 0$, where group membership is encoded in the indicator matrix $\bbZ \in \{0,1\}^{N\times G}$.
    Then, if the smallest group size satisfies $\Nmin \geq 2$, the $\ell_1$ error between $\hbSC$ and the target GSO $\bbS^*$ in~\eqref{eq:opt_C_target} is lower bounded by
    \alna{
        \| \hbSC - \bbS^* \|_1
        &~\geq~&
        \begin{Bmatrix}
            \frac{1}{2} \Nmin \sqrt{G} (\sqrt{ \Rgr(\bbS^*) } - \tau),
            &
            \Rgr(\bbS^*) > \tau^2
            \\
            0,
            &
            \Rgr(\bbS^*) \leq \tau^2
        \end{Bmatrix},
    \label{eq:err_low_gr}}
    and with probability at least $1 - e^{c_2 N}$ for some $c_2 > 0$, the error is upper bounded by
    \alna{
        \| \hbSC - \bbS^* \|_1
        ~\leq~
        \begin{Bmatrix}
            (\phi_1 + \phi_2) \epsilon
            + \phi_3 \sqrt{ \Rgr(\bbS^*) },
            &
            \Rgr(\bbS^*) > \tau^2
            \\
            \phi_1 \epsilon,
            &
            \Rgr(\bbS^*) \leq \tau^2
        \end{Bmatrix},
    \label{eq:err_upp_gr}}
    \vspace{-1em}
    \alna{
        \text{where}~~~~
        \quad
        \phi_1
        =
        \frac{ 4k (2 + k) }{ \sigma_{\scriptscriptstyle\min}(\hbSigma) },
        ~~~~~
        \phi_2
        =
        \frac{ 4 (1 + k) \| \bbZ^\top\bbone \|_2^2 }{ \sigma_{\scriptscriptstyle\min}(\hbSigma) \Nmin },
        ~~~~~
        \phi_3
        =
        2
        (1 + k) G \| \bbZ^\top\bbone \|_2^2,
        &&
        \qquad\quad~~
    &\nonumber\\[-.7cm]&
    \nonumber}
    and $k$, $\sigma_{\scriptscriptstyle\min}(\hbSigma)$, and $\epsilon$ are defined in Assumption~\ref{assump:error_bound}.
\end{mytheorem}
%
The proof of Theorem~\ref{thm:err_bnd_gr} is in Appendix~\ref{app:thm_err_bnd_gr}.
For the setting where $\bbS^*$ is fair enough to be feasible, that is, $\Rgr(\bbS^*)\leq\tau^2$, the result follows directly from classical graph-stationary network inference~\citep{navarro2022jointinferencemultiple}.
Note that $\Nmin\geq 2$ ensures that each group contains at least two nodes, as required in~\eqref{eq:Rgr} so that balancing within- versus across-group edges is feasible.

Critically, Theorem~\ref{thm:err_bnd_gr} demonstrates a \emph{conditional tradeoff between fairness and accuracy}.
First, consider the case when the target GSO $\bbS^*$ is unfair, that is, $\Rgr(\bbS^*) > \tau^2$.
For the upper bound between our estimate $\hbSC$ and $\bbS^*$ in~\eqref{eq:err_upp_gr}, not only will the discrepancy due to the sample covariance increase from $\phi_1\epsilon$ to $(\phi_1+\phi_2)\epsilon$, but the bound also incurs an additional term based on the level of bias in the target $\Rgr(\bbS^*)$.
Additionally, observe that the smallest value of $\| \bbZ^\top\bbone \|_2^2 = \sum_{g} N_g^2$ occurs when all groups have the same number of nodes, so the upper bound in~\eqref{eq:err_upp_gr} grows with the relative imbalance across groups.
For the result in~\eqref{eq:err_low_gr}, as would be expected, an unfair $\bbS^*$ introduces a lower bound on our estimation error, which increases as $\Rgr(\bbS^*)$ becomes further removed from $\tau^2$.
Conversely, if $\Rgr(\bbS^*) \leq \tau^2$, then we require \emph{no sacrifice in performance bounds} in either~\eqref{eq:err_low_gr} or~\eqref{eq:err_upp_gr}, as $\phi_1\epsilon$ corresponds to the classical error bound for estimating $\bbS^*$ from stationary graph signals.
Thus, Theorem~\ref{thm:err_bnd_gr} shows that if the target GSO $\bbS^*$ is unfair, our guarantees on the performance of FairSpecTemp will necessarily suffer, but if $\bbS^*$ is fair, then we may impose fairness for graph estimation with no consequences for our error bounds.
Finally, as expected, since solving~\eqref{eq:opt_C_l1} yields a sparse graph, as the target $\bbS^*$ grows denser, $k$ increases and hence so does the upper bound $\phi_1 \epsilon$.

In addition to the performance of FairSpecTemp with respect to the group-wise bias metric $\Rgr$, we also present an analogous result for individual fairness via $\Rno$.
\begin{mytheorem}\label{thm:err_bnd_no}
    Under Assumption~\ref{assump:error_bound}, consider the GSO estimate $\hbSC \in \reals^{N\times N}$ as the solution to~\eqref{eq:opt_C_l1} with the node-wise bias $R = \Rno$ bounded above by $\tau \geq 0$, where group membership is encoded in the indicator matrix $\bbZ \in \{ 0,1 \}^{N\times G}$.
    Then, the $\ell_1$ error between $\hbSC$ and the target GSO $\bbS^*$ in~\eqref{eq:opt_C_target} is lower bounded by
    \alna{
        \| \hbSC - \bbS^* \|_1
        &~\geq~&
        \begin{Bmatrix}
            \frac{1}{2} \Nmin \sqrt{GN} (\sqrt{ \Rno(\bbS^*) } - \tau),
            &
            \Rno(\bbS^*) > \tau^2
            \\
            0,
            &
            \Rno(\bbS^*) \leq \tau^2
        \end{Bmatrix},
    \label{eq:err_low_no}}
    and with probability at least $1 - e^{c_2 N}$ for some $c_2 > 0$, the error is upper bounded by
    \alna{
        \| \hbSC - \bbS^* \|_1
        ~\leq~
        \begin{Bmatrix}
            (\phi_1 + \phi_4) \epsilon
            + \phi_5 \sqrt{ \Rno(\bbS^*) },
            &
            \Rno(\bbS^*) > \tau^2
            \\
            \phi_1 \epsilon,
            &
            \Rno(\bbS^*) \leq \tau^2
        \end{Bmatrix},
    \label{eq:err_upp_no}}
    \vspace{-1em}
    \alna{
        \text{where}~~~~
        \quad
        \phi_1
        =
        \frac{ 4k (2 + k) }{ \sigma_{\scriptscriptstyle\min}(\hbSigma) },
        ~~~~~
        \phi_4
        =
        \frac{ 2 (1 + k) N_{\scriptscriptstyle\max}^2 \sqrt{G} }{ \sigma_{\scriptscriptstyle\min}(\hbSigma) \Nmin },
        ~~~~~
        \phi_5
        =
        (1 + k) N_{\scriptscriptstyle\max}^2 \sqrt{G^3},
        \qquad\qquad
    &\nonumber\\[-.7cm]&
    \nonumber}
    and $k$, $\sigma_{\scriptscriptstyle\min}(\hbSigma)$, and $\epsilon$ are defined in Assumption~\ref{assump:error_bound}.
\end{mytheorem}
The proof can be found in Appendix~\ref{app:thm_err_bnd_no}, which follows similar steps to the proof of Theorem~\ref{thm:err_bnd_gr}, excepting the steps that obtain the lower bound~\eqref{eq:err_low_no}.
When $\Rno(\bbS^*)\leq\tau^2$ and $\bbS^*$ is therefore feasible, the error bounds for $\hbSC$ are equivalent to those in~\eqref{eq:err_low_gr} and~\eqref{eq:err_upp_gr}.
Theorem~\ref{thm:err_bnd_no} again demonstrates a conditional tradeoff between fairness and accuracy, where an unfair $\bbS^*$ such that $\Rno(\bbS^*) > \tau^2$ increases both the upper and lower bounds on the estimation error of $\hbSC$, while a fair $\bbS^*$ with $\Rno(\bbS^*) \leq \tau^2$ does not negatively affect performance guarantees. 
Observe that the discrepancy due to imposing fairness again depends on the bias in the target GSO $\Rno(\bbS^*)$, but the stricter, node-wise metric $\Rno$ has a larger lower bound in~\eqref{eq:err_low_no} than the group-wise metric $\Rgr$ in~\eqref{eq:err_low_gr}.
Moreover, the additional terms in the upper bound~\eqref{eq:err_upp_no} containing $\phi_4$ and $\phi_5$ depend on the largest group size $N_{\scriptscriptstyle\max}^2$, which achieves its smallest value when all groups have the same size.
Theorem~\ref{thm:err_bnd_no} again demonstrates that FairSpecTemp must sacrifice accuracy for fairness when the target GSO $\bbS^*$ is biased, but a fair $\bbS^*$ requires no tradeoff between fairness and accuracy.

\begin{myremark}\label{rem:sparse_bias_tradeoff_gr}
    Note that the feasibility of $\bbS^*$ for~\eqref{eq:opt_C_l1} determines the upper bound for both~\eqref{eq:err_upp_gr} and~\eqref{eq:err_upp_no} in Theorems~\ref{thm:err_bnd_gr} and~\ref{thm:err_bnd_no}, respectively.
    However, even if $\bbS^*$ is not feasible, that is, $R(\bbS^*)>\tau^2$, the error bound $\phi_1\epsilon$ still holds if $\bbS^*$ is much fairer than it is sparse.
    More specifically, if $R = \Rgr$ and $\| \bbS^* \|_1 \geq 2\| \bbZ^\top \bbone \|_2^2 \left( G\sqrt{\Rgr(\bbS^*)} + \frac{2\epsilon}{ \sigma_{ \scriptscriptstyle\min }(\hbSigma)  N_{ \scriptscriptstyle\min }  } \right)$ or if $R = \Rno$ and $\| \bbS^* \|_1 \geq N_{\scriptscriptstyle\max}^2 \sqrt{G} \left( G\sqrt{ \Rno(\bbS^*) } + \frac{2\epsilon}{ \sigma_{\scriptscriptstyle\min}(\hbSigma) N_{\scriptscriptstyle\min} } \right)$, then the upper bound $\| \hbSC - \bbS^* \|_1 \leq \phi_1 \epsilon$ holds with high probability.
    Thus, even if the target GSO $\bbS^*$ is not fair enough to be feasible, imposing the constraint $R(\bbS^*) \leq \tau^2$ may be preferable if $\bbS^*$ is not sparse.
\end{myremark}

\begin{myremark}\label{rem:upp_bnd_est_bias}
    While we bound the bias in the estimated GSO $\hbSC$ via $\tau^2$, the actual bias $R(\hbSC)$ may be restricted by the observed data encoded in the sample covariance matrix $\hbC$.
    To see this, recall the definition of $\hbSigma$ based on $\hbC$, and observe that under a full-column rank $\hbSigma$ as in Assumption \textit{(A3)},
    $R(\bbS) = \| \bbR \vect(\bbS)_{\ccalL} \|_2^2 \leq \| \bbR\hbSigma^{\dagger} \|_2^2 \epsilon^2$
    for any $\bbS \in \ccalS$ that may be obtained from~\eqref{eq:opt_C_l1}. 
    Hence, we can view $\| \bbR \hbSigma^{\dagger} \|_2$ as measuring the bias in the observed data, where the measurement $R(\hbSC)$ cannot exceed this value scaled by $\epsilon$, which indicates that some introduced error represented by $\epsilon$ may distort the level of possible bias.
\end{myremark}

\subsection{FairSpecTemp via Shared Eigenbasis}
\label{ss:fst_st}

FairSpecTemp proposed in~\eqref{eq:opt_C_l0} encourages similar eigenvectors between $\hbC$ and $\hbSC$ while directly constraining the bias in $\hbSC$.
However, to mitigate the fairness-accuracy tradeoff shown in Theorems~\ref{thm:err_bnd_gr} and~\ref{thm:err_bnd_no}, we propose an alternative to~\eqref{eq:opt_C_l0} that allows us to enforce fairness implicitly. 
Recall that graph stationarity implies that $\bbC \bbS^* = \bbS^* \bbC$ because $\bbS^*$ and $\bbC$ share the same eigenvectors $\bbV = [\bbv_1,\dots,\bbv_N]$ (see Section~\ref{ss:gsp}).
We can therefore equivalently define $\bbS^*$ as
\alna{
    \bbS^*, \bblambda^* \!~\in~
    &\argmin_{\bbS \in \ccalS, ~\! \bblambda} ~~
    \| \bbS^+ \|_0
    ~~{\rm s.t.}~~
    \bbS = \sum_{i=1}^N \lambda_i \bbv_i\bbv_i^\top,
\label{eq:opt_V_target}}
where we not only obtain the GSO $\bbS^*$ but also its eigenvalues $\bblambda^*$ given the eigendecomposition $\bbC = \bbV \bbGamma \bbV^\top$~\citep{segarra2017networktopologyinference}.
Since there are no conditions on $\bblambda$ in~\eqref{eq:opt_V_target}, the constraint $\bbS = \sum_i \lambda_i \bbv_i \bbv_i^\top$ is equivalent to the condition $\bbC \bbS = \bbS \bbC$, and therefore the sets of optimal GSOs for~\eqref{eq:opt_V_target} and~\eqref{eq:opt_C_target} are equivalent.
%
We then propose an analogous approximation of~\eqref{eq:opt_V_target} as in~\eqref{eq:opt_C_l0}, where we estimate the target GSO $\bbS^*$ given the sample covariance matrix $\hbC$ while encouraging fair connections.
With the eigendecomposition $\hbC = \hbV \hbGamma \hat{\bbV}^\top$ and eigenvectors $\hbV = [\hbv_1,\dots,\hbv_N]$, we propose the following adaptation of FairSpecTemp,
\alna{
    \bbSVp,
    \bblambda'_{\rm \scriptscriptstyle V}
    ~\! \in ~
    \argmin_{\bbS \in \ccalS, ~\! \bblambda} ~~
    \| \bbS^+ \|_0
    ~~{\rm s.t.}~~
    \left\| \bbS - \sum_{i=1}^N \lambda_i \hbv_i \hbv_i^\top \right\|_F \leq \epsilon, ~
    R(\bbS) \leq \tau^2
\label{eq:opt_V_l0}}
for $R$ measuring bias, where we again focus on $R \in \{ \Rgr, \Rno \}$, although this choice is flexible to other topological bias penalties~\citep{wang2023FG2ANFairnessAwareGraph}.

Similarly to~\eqref{eq:opt_C_l1}, we relax the $\ell_0$ norm by replacing it with the convex $\ell_1$ norm to get
\alna{
    \hbSV, 
    \hblambda_{\rm \scriptscriptstyle V}
    ~\! \in ~
    \argmin_{\bbS \in \ccalS, ~\! \bblambda } ~~
    \| \bbS^+ \|_1
    ~~{\rm s.t.}~~
    \left\| \bbS - \sum_{i=1}^N \lambda_i \hbv_i\hbv_i^\top \right\|_F \leq \epsilon, ~
    R(\bbS) \leq \tau^2,
\label{eq:opt_V_l1}}
for which we again provide sufficient conditions under which problems~\eqref{eq:opt_V_l0} and~\eqref{eq:opt_V_l1} return the same solution.
Given the definitions $\hbJ := \hbV \odot \hbV$ and $\hbF = (\bbI - \hbJ\hbJ^\top)\bbU$, we present our result guaranteeing that the convex relaxation~\eqref{eq:opt_V_l1} is minimized by a solution to~\eqref{eq:opt_V_l0} as follows.
As before, we show this for adjacency matrices $\ccalS=\SA$, but an analogous result for graph Laplacian GSOs $\ccalS = \SL$ can be found in Appendix~\ref{app:Lapl_convex}.

\begin{mytheorem}\label{thm:V_convex}
    If problems~\eqref{eq:opt_V_l0} and~\eqref{eq:opt_V_l1} are feasible for $\ccalS = \SA$, then for $\bbSVp$ as a solution to~\eqref{eq:opt_V_l0}, we have that $\hbSV = \bbSVp$ is the unique solution to~\eqref{eq:opt_V_l1} if
    \begin{itemize}[left= 15pt .. 24pt, itemsep=1pt]
    \im[(i)] The submatrix $\hbF_{\cdot,\ccalI}$ is full column rank, and
    \im[(ii)] There exists a constant $\psi > 0$ such that
    \alna{
        \normms{
            \left(
                \psi^{-2} \bbPsi^\top \bbPsi + \bbI_{\cdot,\bar{\ccalI}} \bbI_{\bar{\ccalI},\cdot}
            \right)^{-1}_{\bar{\ccalI},\ccalI}
        }_{\infty}
        ~<~
        1,
    \label{eq:thm_V_cvx_cond}}
    \end{itemize}
    where $\ccalI = {\rm supp}(\bbs')$ for $\bbs' = \vect(\bbSVp)_{\ccalL}$, and $\bbPsi := [ \hbF^\top, \bbR^\top, \bbE^\top ]^\top$.
\end{mytheorem}
The proof of Theorem~\ref{thm:V_convex} is also based on~\citep[Theorem 1]{segarra2017networktopologyinference} and~\citep[Theorem 1]{zhang2016OneConditionSolution} and shown in Appendix~\ref{app:thm_V_convex}.
Analogously to Theorem~\ref{thm:C_convex}, condition \textit{(ii)} of Theorem~\ref{thm:V_convex} guarantees that $\bbSVp$ is a solution to~\eqref{eq:opt_V_l1} by verifying the existence of a solution to the dual problem of~\eqref{eq:opt_V_l1}, while condition \textit{(i)} ensures that $\bbSVp$ is the unique solution to~\eqref{eq:opt_V_l1}.
Under the conditions of Theorem~\ref{thm:V_convex}, we may solve the convex problem~\eqref{eq:opt_V_l1} for the sparsest GSO $\bbSVp$ satisfying the commutativity constraint in~\eqref{eq:opt_V_l0} and~\eqref{eq:opt_V_l1}.

For a final relaxation of~\eqref{eq:opt_V_l0}, we loosen the bias constraint with respect to $\bbS$ by solving
\alna{
    \hbS_{\rm \scriptscriptstyle V},
    \hblambda_{\rm \scriptscriptstyle V}
    ~\in~
    \argmin_{\bbS \in \ccalS, ~\! \bblambda} ~~
    \| \bbS^+ \|_1
    ~~{\rm s.t.}~~
    \left\| \bbS - \sum_{i=1}^N \lambda_i \hbv_i\hbv_i^\top \right\|_F \leq \epsilon, ~
    R(\bblambda) \leq \tau^2,
\label{eq:opt_V_P_l1}}
where we encourage fairness via the eigenvalues $\bblambda$ instead of the GSO $\bbS$.
Observe that~\eqref{eq:opt_V_P_l1} estimates graphs while relaxing both the graph-stationarity and fairness constraints in comparison with~\eqref{eq:opt_C_l1}.
More specifically, jointly optimizing $\bbS$ and $\bblambda$ allows more degrees of freedom, although this increased flexibility can cause greater error when $\hbV$ differs further from the true eigenvectors $\bbV$.
Moreover, by using our spectral bias metrics in~\eqref{eq:Rgr_spec_A} and~\eqref{eq:Rno_spec_A} for~\eqref{eq:opt_V_P_l1}, we promote fairness via $R(\bblambda) \approx R(\bbS)$ since $\hbV$ is known.
Note that $R(\hblambda_{\rm \scriptscriptstyle V})\leq \tau^2$ does not guarantee that $R(\hbS_{\rm \scriptscriptstyle V}) \leq \tau^2$.
However, since~\eqref{eq:Rgr_spec_A} and~\eqref{eq:Rno_spec_A} consist of terms multiplying $\bblambda$ by matrices or vectors that can be precomputed, algorithms based on~\eqref{eq:opt_V_P_l1} can enjoy far simpler computations than~\eqref{eq:opt_V_l1}. 
Additionally, we show in Section~\ref{s:results} that, given sufficiently many samples, solving the more flexible~\eqref{eq:opt_V_P_l1} can accomplish both fair and accurate graph estimation, even when the target graph is biased.

\section{Numerical Evaluation}
\label{s:results}

In this section, we apply FairSpecTemp using both the commutativity condition in~\eqref{eq:opt_C_l1} and the shared eigenbases formulation in~\eqref{eq:opt_V_P_l1} to estimate graphs with unbiased connections, considering both group $\Rgr$ and individual fairness $\Rno$. 
Problems of the form in~\eqref{eq:opt_C_l1} or~\eqref{eq:opt_V_P_l1} are widely studied and can be solved by many efficient algorithms.
We adopt the Fast Iterative Shrinkage-Thresholding Algorithm (FISTA) to handle the associated constraints and the non-smooth $\ell_1$ norm~\citep{beck2009fastiterativeshrinkagethresholding,navarro2024FairGLASSO}.
To evaluate the performance of FairSpecTemp, we first conduct experiments on synthetic graphs under a variety of controlled settings.
This enables us to observe situations where a trade-off between fair and accurate estimation arises, as well as others where both can be achieved simultaneously.
We then move to real-world data by applying our graph estimation method to a financial investment task using historical stock market information.
In this setting, the structure of the estimated graphs informs dynamic investment decisions over time.

We aim to recover a target GSO $\bbS^*$, which represents a ground truth graph of interest in the following simulations.
To assess an estimate GSO $\hbS$, we measure error, group-wise bias, and node-wise bias respectively as
\alna{
    d(\hbS,\bbS^*)
    &~=~&
    \frac{
        \| \hbS - \bbS^* \|_F^2
    }{
        \| \bbS^* \|_F^2
    },
    ~~
    b_{\scriptscriptstyle\rm G}(\hbS)
    \,=\,
    \frac{
        (N^2-N) \sqrt{ \Rgr(\hbS) }
    }{
        2\| \hbS \|_1
    },
    ~~
    b_{\scriptscriptstyle\rm N}(\hbS)
    \,=\,
    \frac{
        (N^2-N) \sqrt{ \Rno(\hbS) }
    }{
        2\| \hbS \|_1
    },
\label{eq:perf_funcs}}
where the bias is normalized by average edge weight.
Throughout the ensuing simulations, we consider the following methods: 
(i)~$\st$: SpecTemp without fairness, that is,~\eqref{eq:opt_C_l1} with $\tau=\infty$~\citep{segarra2017networktopologyinference}; 
(ii)~$\fstcg$: FairSpecTemp via~\eqref{eq:opt_C_l1} with $R=\Rgr$;
(iii)~$\fstcn$: FairSpecTemp via~\eqref{eq:opt_C_l1} with $R=\Rno$;
(iv)~$\fstvg$: FairSpecTemp via~\eqref{eq:opt_V_P_l1} with $R=\Rgr$;
(v)~$\fstvn$: FairSpecTemp via~\eqref{eq:opt_V_P_l1} with $R=\Rno$;
(vi)~$\strw$: SpecTemp with randomly rewired edges;
(vii)~$\stba$: SpecTemp with edges reweighted to be balanced across group pairs;
and 
(viii)~$\fgl$: Fair GLASSO~\citep{navarro2024FairGLASSO}, estimating fair Gaussian graphical models with $R = \Rgr$.
As discussed in Section~\ref{ss:related}, fair graph estimation from nodal data is highly limited, with $\fgl$ as one existing method developed for the same task, albeit assuming a stricter graph signal model.
Hence, we consider additional baselines to evaluate how well FairSpecTemp can recover fair, accurate graphs.
In particular, for $\strw$, we randomly rewire edges of the $\st$ estimate, which, with sufficiently many rewirings, effectively decouples dependencies between edges and group labels.
However, rewiring edges may not reduce bias for highly imbalanced groups since edge densities across group pairs may be relatively unchanged.
We therefore also consider $\stba$, where we maintain the support of the initial $\st$ estimate and instead rescale GSO entries for more balanced edge weights across pairs of groups.
The code to reproduce all results, along with implementation details, is available on GitHub\footnote{\url{https://github.com/mn51/fair_nti}}.

\begin{figure*}[t]
    \centering
    \begin{minipage}{0.15\textwidth}
        \begin{center}
            \includegraphics[width=\textwidth]{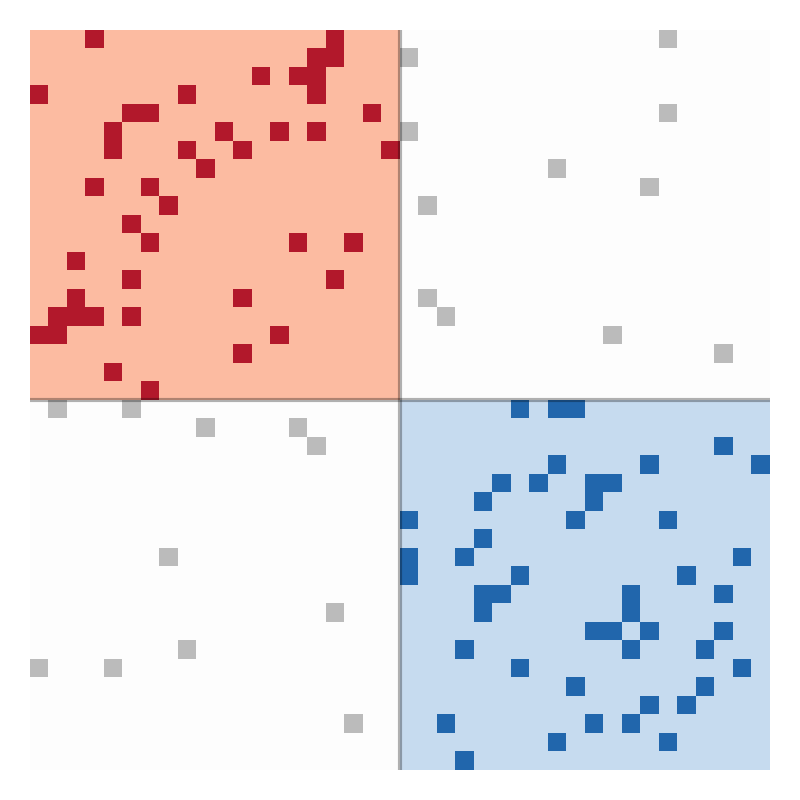}
            \\
            {\scriptsize (a) Within-group preference}
        \end{center}
    \end{minipage}
    \begin{minipage}{0.15\textwidth}
        \begin{center}
            \includegraphics[width=\textwidth]{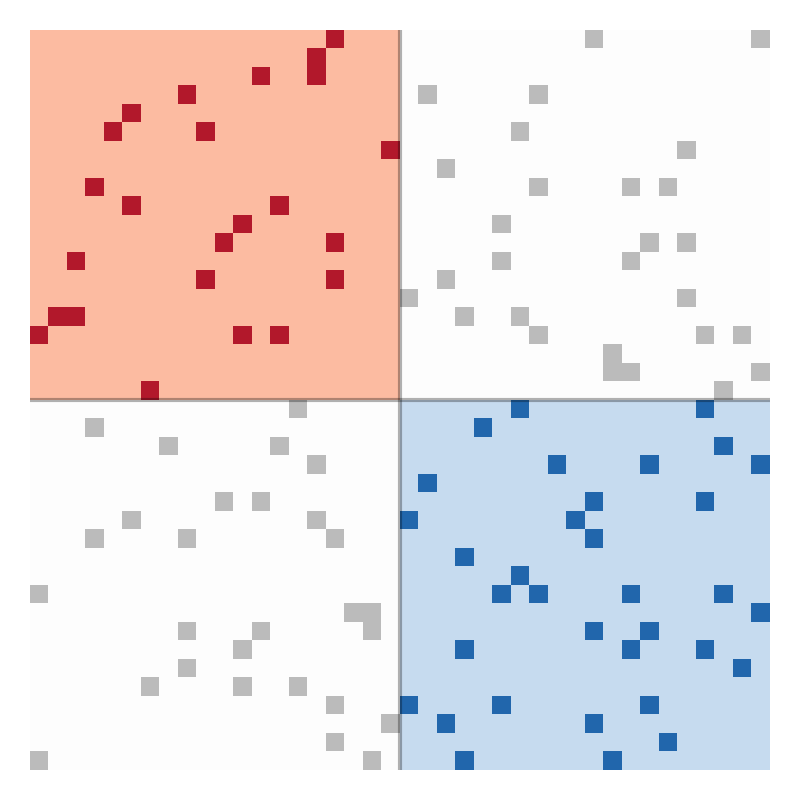}
            \\
            {\scriptsize (b) Group-wise balance}
        \end{center}
    \end{minipage}
    \begin{minipage}{0.15\textwidth}
        \begin{center}
            \includegraphics[width=\textwidth]{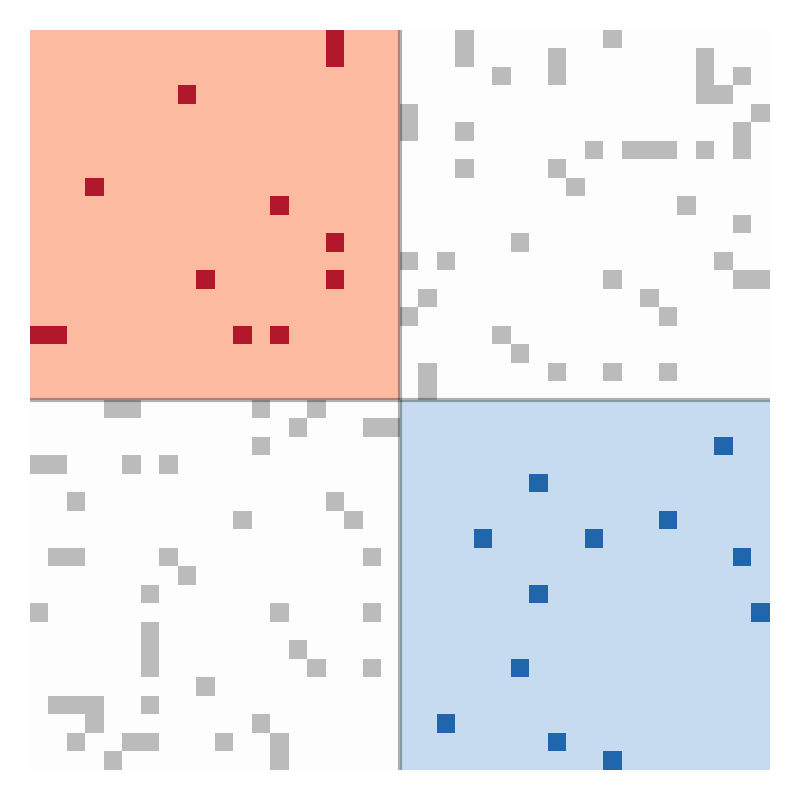}
            \\
            {\scriptsize (c) Across-group preference}
        \end{center}
    \end{minipage}
    \hspace{.25cm}
    \begin{minipage}{0.15\textwidth}
        \begin{center}
            \includegraphics[width=\textwidth]{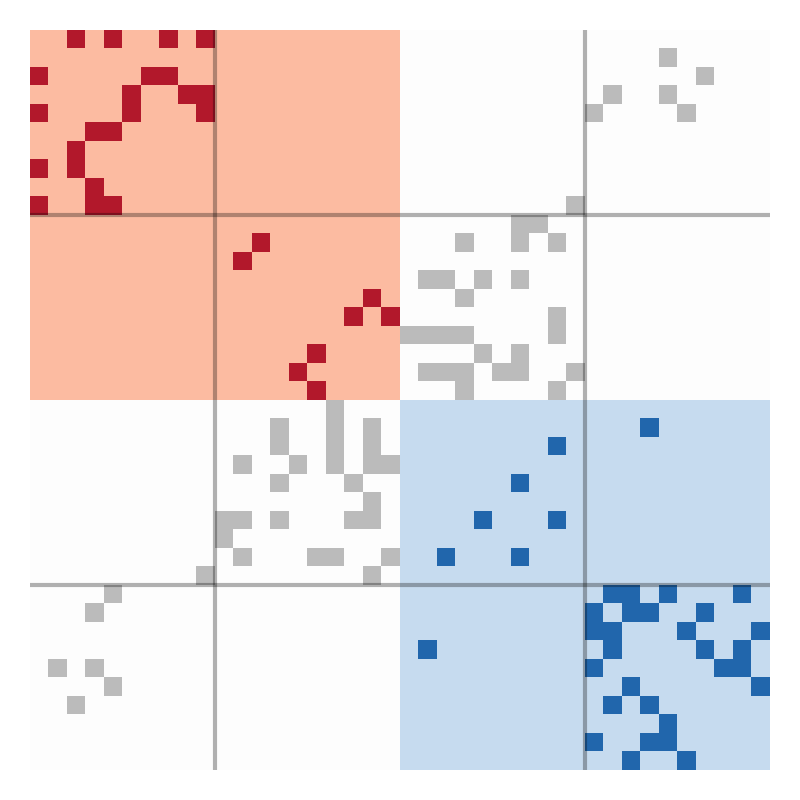}
            \\
            {\scriptsize (d) Sub-group preference}
        \end{center}
    \end{minipage}
    \begin{minipage}{0.15\textwidth}
        \begin{center}
            \includegraphics[width=\textwidth]{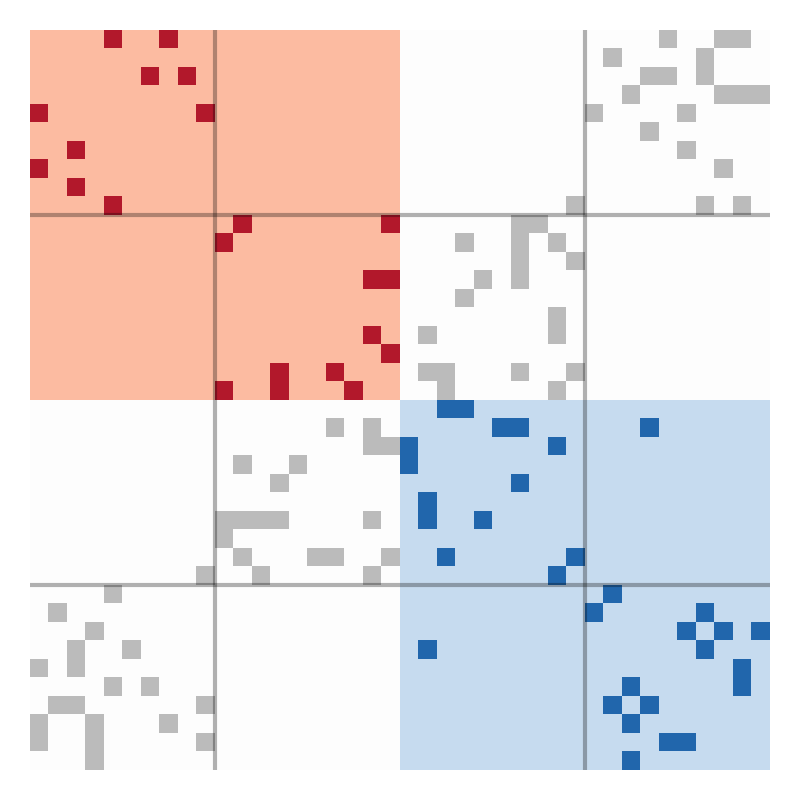}
            \\
            {\scriptsize (e) Sub-group balance}
        \end{center}
    \end{minipage}
    \begin{minipage}{0.15\textwidth}
        \begin{center}
            \includegraphics[width=\textwidth]{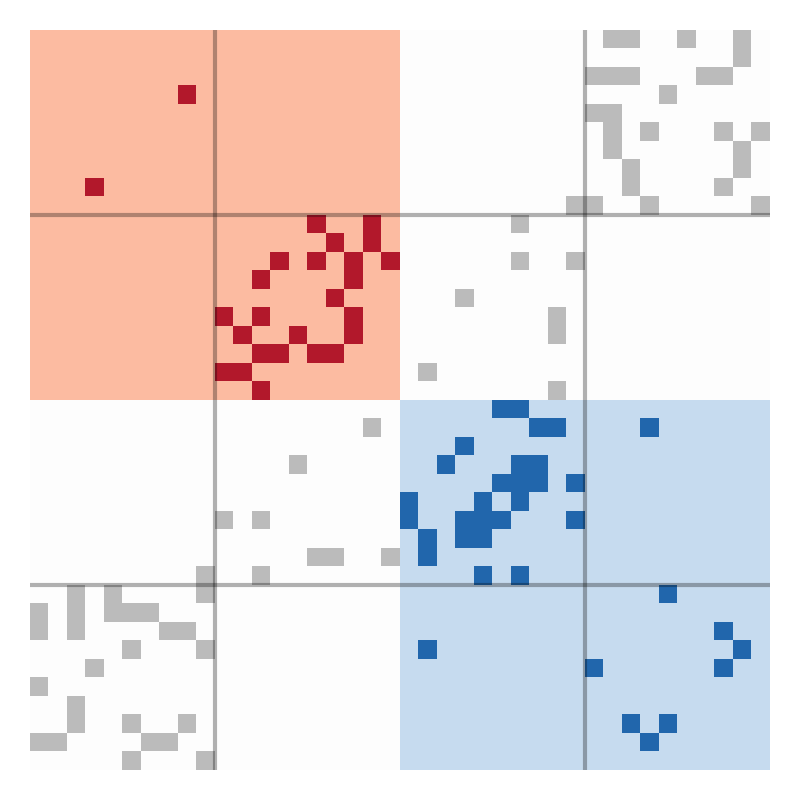}
            \\
            {\scriptsize (f) Sub-group preference}
        \end{center}
    \end{minipage}
    
    \caption{
    Target $\bbS^*$ as $\Rgr(\bbS^*)$ and $\Rno(\bbS^*)$ vary.
    Red denotes node pairs in group 1, blue node pairs in group 2, and gray node pairs in different groups.
    Both $\Rgr(\bbS^*)$ and $\Rno(\bbS^*)$ vary in (a)-(c), while $\Rgr(\bbS^*)$ is low but $\Rno(\bbS^*)$ varies in (d)-(f).
    (a) Majority of edges connect the \textit{same} group ($\Rgr(\bbS^*)$ and $\Rno(\bbS^*)$ high).
    (b) Balanced number of edges within and across groups ($\Rgr(\bbS^*)$ and $\Rno(\bbS^*)$ low).
    (c) Majority of edges connect \textit{different} groups ($\Rgr(\bbS^*)$ and $\Rno(\bbS^*)$ high).
    (d) Subgroups of nodes show strong within- or across-group preferences ($\Rgr(\bbS^*)$ low, $\Rno(\bbS^*)$ high).
    (e) Every node has relatively balanced edges across groups ($\Rgr(\bbS^*)$ and $\Rno(\bbS^*)$ low).
    (f) Alternate subgroups of nodes show strong within- or across-group preferences ($\Rgr(\bbS^*)$ low, $\Rno(\bbS^*)$ high).
    }
\label{fig:synth_topbias_vis}
\end{figure*}

\begin{figure*}[t]
    \centering
    \begin{minipage}[b][][b]{0.3\textwidth}
        \centering
        \scalebox{.85}{\begin{tikzpicture}[baseline,scale=.9,trim axis left, trim axis right]

\pgfmathsetmacro{\initmarksize}{3.5}

\pgfplotstableread{data/exp4_groupbias.csv}\errbiastable

\begin{axis}[
    xlabel={(a) Across-group ratio},
    xmin=0.2,
    xmax=0.8,
    ylabel={Error $d(\hbS,\bbS^*)$},
    yticklabel style={
        /pgf/number format/fixed,
        /pgf/number format/precision=2,
        font=\footnotesize
    },
    ymin=-0.05,
    ymax=0.55,
    grid style=densely dashed,
    grid=both,
    legend style={
        at={(.5,1.02)},
        anchor=south,
        font=\footnotesize},
    legend columns=3,
    width=200,
    height=175,
    label style={font=\small},
    tick label style={font=\small}
    ]

    \addplot[black!80!white, mark=triangle*, solid]  table 
        [x=ratio_range, y=err_med_st] {\errbiastable};

    \addplot[Blue2, mark=o, solid] table 
        [x=ratio_range, y=err_med_fstcg] {\errbiastable};

    \addplot[Red2, mark=square, solid, mark size=1.5pt] table 
        [x=ratio_range, y=err_med_fstcn] {\errbiastable};

    \addplot[white!70!Blue2, mark=*, solid] table 
        [x=ratio_range, y=err_med_fstvg] {\errbiastable};

    \addplot[white!70!Red2, mark=square*, solid, mark size=1.5pt] table 
        [x=ratio_range, y=err_med_fstvn] {\errbiastable};

\end{axis}

\end{tikzpicture}}
    \end{minipage}
    \hspace{.45cm}
    \begin{minipage}[b][][b]{0.3\textwidth}
        \centering
        \scalebox{.85}{\begin{tikzpicture}[baseline,scale=.9,trim axis left, trim axis right]

\pgfmathsetmacro{\initmarksize}{3.5}

\pgfplotstableread{data/exp4_groupbias.csv}\errbiastable

\begin{axis}[
    xlabel={(b) Across-group ratio},
    xmin=0.2,
    xmax=0.8,
    ylabel={Group-wise bias $b_{\scriptscriptstyle \rm G}(\hbS)$},
    yticklabel style={
        /pgf/number format/fixed,
        /pgf/number format/precision=2,
        font=\footnotesize
    },
    ymin=-0.05,
    ymax=0.65,
    grid style=densely dashed,
    grid=both,
    legend style={
        at={(.5,1.02)},
        anchor=south,
        font=\footnotesize},
    legend columns=5,
    width=200,
    height=175,
    label style={font=\small},
    tick label style={font=\small}
    ]

    \addplot[black!80!white, mark=triangle*, solid]  table 
        [x=ratio_range, y=biasg_med_st] {\errbiastable};

    \addplot[Blue2, mark=o, solid] table 
        [x=ratio_range, y=biasg_med_fstcg] {\errbiastable};

    \addplot[Red2, mark=square, solid, mark size=1.5pt] table 
        [x=ratio_range, y=biasg_med_fstcn] {\errbiastable};

    \addplot[white!70!Blue2, mark=*, solid] table 
        [x=ratio_range, y=biasg_med_fstvg] {\errbiastable};

    \addplot[white!70!Red2, mark=square*, solid, mark size=1.5pt] table 
        [x=ratio_range, y=biasg_med_fstvn] {\errbiastable};
        
    \addlegendentry{$\st$~~~~~~}
    \addlegendentry{$\fstcg$~~~~~~}
    \addlegendentry{$\fstcn$~~~~~~}
    \addlegendentry{$\fstvg$~~~~~~}
    \addlegendentry{$\fstvn$}
\end{axis}

\end{tikzpicture}}
    \end{minipage}
    \hspace{.55cm}
    \begin{minipage}[b][][b]{0.3\textwidth}
        \centering
        \scalebox{.85}{\begin{tikzpicture}[baseline,scale=.9,trim axis left, trim axis right]

\pgfmathsetmacro{\initmarksize}{3.5}

\pgfplotstableread{data/exp4_groupbias.csv}\errbiastable

\begin{axis}[
    xlabel={(c) Across-group ratio},
    xmin=0.2,
    xmax=0.8,
    ylabel={Node-wise bias $b_{\scriptscriptstyle \rm N}(\hbS)$},
    yticklabel style={
        /pgf/number format/fixed,
        /pgf/number format/precision=2,
        font=\footnotesize
    },
    ymin=0.15,
    ymax=0.85,
    grid style=densely dashed,
    grid=both,
    legend style={
        at={(.5,1.02)},
        anchor=south,
        font=\footnotesize},
    legend columns=3,
    width=200,
    height=175,
    label style={font=\small},
    tick label style={font=\small}
    ]

    \addplot[black!80!white, mark=triangle*, solid]  table 
        [x=ratio_range, y=biasn_med_st] {\errbiastable};

    \addplot[Blue2, mark=o, solid] table 
        [x=ratio_range, y=biasn_med_fstcg] {\errbiastable};

    \addplot[Red2, mark=square, solid, mark size=1.5pt] table 
        [x=ratio_range, y=biasn_med_fstcn] {\errbiastable};

    \addplot[white!70!Blue2, mark=*, solid] table 
        [x=ratio_range, y=biasn_med_fstvg] {\errbiastable};

    \addplot[white!70!Red2, mark=square*, solid, mark size=1.5pt] table 
        [x=ratio_range, y=biasn_med_fstvn] {\errbiastable};
\end{axis}

\end{tikzpicture}}
    \end{minipage}
    \\[.3cm]
    \begin{minipage}[b][][b]{0.3\textwidth}
        \centering
        \scalebox{.85}{\begin{tikzpicture}[baseline,scale=.9,trim axis left, trim axis right]

\pgfmathsetmacro{\initmarksize}{3.5}

\pgfplotstableread{data/exp4_indivbias.csv}\errbiastable

\begin{axis}[
    xlabel={(d) Subgroup edges rewired},
    xmin=0.2,
    xmax=0.8,
    ylabel={Error $d(\hbS,\bbS^*)$},
    yticklabel style={
        /pgf/number format/fixed,
        /pgf/number format/precision=2,
        font=\footnotesize
    },
    ymin=-0.05,
    ymax=0.55,
    grid style=densely dashed,
    grid=both,
    legend style={
        at={(.5,1.02)},
        anchor=south,
        font=\footnotesize},
    legend columns=3,
    width=200,
    height=175,
    label style={font=\small},
    tick label style={font=\small}
    ]

    \addplot[Blue2, mark=o, solid] table 
        [x=ratio_range, y=err_med_fstcg] {\errbiastable};

    \addplot[Red2, mark=square, solid, mark size=1.5pt] table 
        [x=ratio_range, y=err_med_fstcn] {\errbiastable};

    \addplot[black!80!white, mark=triangle*, solid]  table 
        [x=ratio_range, y=err_med_st] {\errbiastable};

    \addplot[white!70!Blue2, mark=*, solid] table 
        [x=ratio_range, y=err_med_fstvg] {\errbiastable};

    \addplot[white!70!Red2, mark=square*, solid, mark size=1.5pt] table 
        [x=ratio_range, y=err_med_fstvn] {\errbiastable};
    
\end{axis}

\end{tikzpicture}}
    \end{minipage}
    \hspace{.45cm}
    \begin{minipage}[b][][b]{0.3\textwidth}
        \centering
        \scalebox{.85}{\begin{tikzpicture}[baseline,scale=.9,trim axis left, trim axis right]

\pgfmathsetmacro{\initmarksize}{3.5}

\pgfplotstableread{data/exp4_indivbias.csv}\errbiastable

\begin{axis}[
    xlabel={(e) Subgroup edges rewired},
    xmin=0.2,
    xmax=0.8,
    ylabel={Group-wise bias $b_{\scriptscriptstyle \rm G}(\hbS)$},
    yticklabel style={
        /pgf/number format/fixed,
        /pgf/number format/precision=2,
        font=\footnotesize
    },
    ymin=-0.05,
    ymax=0.65,
    grid style=densely dashed,
    grid=both,
    legend style={
        at={(.5,1.02)},
        anchor=south,
        font=\footnotesize},
    legend columns=5,
    width=200,
    height=175,
    label style={font=\small},
    tick label style={font=\small}
    ]

    \addplot[black!80!white, mark=triangle*, solid]  table 
        [x=ratio_range, y=biasg_med_st] {\errbiastable};

    \addplot[Blue2, mark=o, solid] table 
        [x=ratio_range, y=biasg_med_fstcg] {\errbiastable};

    \addplot[Red2, mark=square, solid, mark size=1.5pt] table 
        [x=ratio_range, y=biasg_med_fstcn] {\errbiastable};

    \addplot[white!70!Blue2, mark=*, solid] table 
        [x=ratio_range, y=biasg_med_fstvg] {\errbiastable};

    \addplot[white!70!Red2, mark=square*, solid, mark size=1.5pt] table 
        [x=ratio_range, y=biasg_med_fstvn] {\errbiastable};
        
\end{axis}

\end{tikzpicture}}
    \end{minipage}
    \hspace{.55cm}
    \begin{minipage}[b][][b]{0.3\textwidth}
        \centering
        \scalebox{.85}{\begin{tikzpicture}[baseline,scale=.9,trim axis left, trim axis right]

\pgfmathsetmacro{\initmarksize}{3.5}

\pgfplotstableread{data/exp4_indivbias.csv}\errbiastable

\begin{axis}[
    xlabel={(f) Subgroup edges rewired},
    xmin=0.2,
    xmax=0.8,
    ylabel={Node-wise bias $b_{\scriptscriptstyle \rm N}(\hbS)$},
    yticklabel style={
        /pgf/number format/fixed,
        /pgf/number format/precision=2,
        font=\footnotesize
    },
    ymin=0.15,
    ymax=0.85,
    grid style=densely dashed,
    grid=both,
    legend style={
        at={(.5,1.02)},
        anchor=south,
        font=\footnotesize},
    legend columns=3,
    width=200,
    height=175,
    label style={font=\small},
    tick label style={font=\small}
    ]

    \addplot[Blue2, mark=o, solid] table 
        [x=ratio_range, y=biasn_med_fstcg] {\errbiastable};

    \addplot[Red2, mark=square, solid, mark size=1.5pt] table 
        [x=ratio_range, y=biasn_med_fstcn] {\errbiastable};

    \addplot[black!80!white, mark=triangle*, solid]  table 
        [x=ratio_range, y=biasn_med_st] {\errbiastable};

    \addplot[white!70!Blue2, mark=*, solid] table 
        [x=ratio_range, y=biasn_med_fstvg] {\errbiastable};

    \addplot[white!70!Red2, mark=square*, solid, mark size=1.5pt] table 
        [x=ratio_range, y=biasn_med_fstvn] {\errbiastable};
    



\end{axis}

\end{tikzpicture}}
    \end{minipage}
    %
    
    \caption{
    \textbf{Top row:} Performance of estimates $\hbS$ as $\Rgr(\bbS^*)$ and $\Rno(\bbS^*)$ vary corresponding to Figure~\ref{fig:synth_topbias_vis}(a)-(c).
    \textbf{Bottom row:} Performance of estimates $\hbS$ for low $\Rgr(\bbS^*)$ as $\Rno(\bbS^*)$ varies corresponding to Figure~\ref{fig:synth_topbias_vis}(d)-(f).
    \textbf{Left column:} (a) and (d) depict estimation error $d(\hbS,\bbS^*)$.
    \textbf{Middle column:} (b) and (e) show \emph{group-wise} bias $b_{\scriptscriptstyle\rm G}(\hbS)$.
    \textbf{Right column:} (c) and (f) show \emph{node-wise} bias $b_{\scriptscriptstyle\rm N}(\hbS)$.
    }
\label{fig:synth_topbias}
\end{figure*}
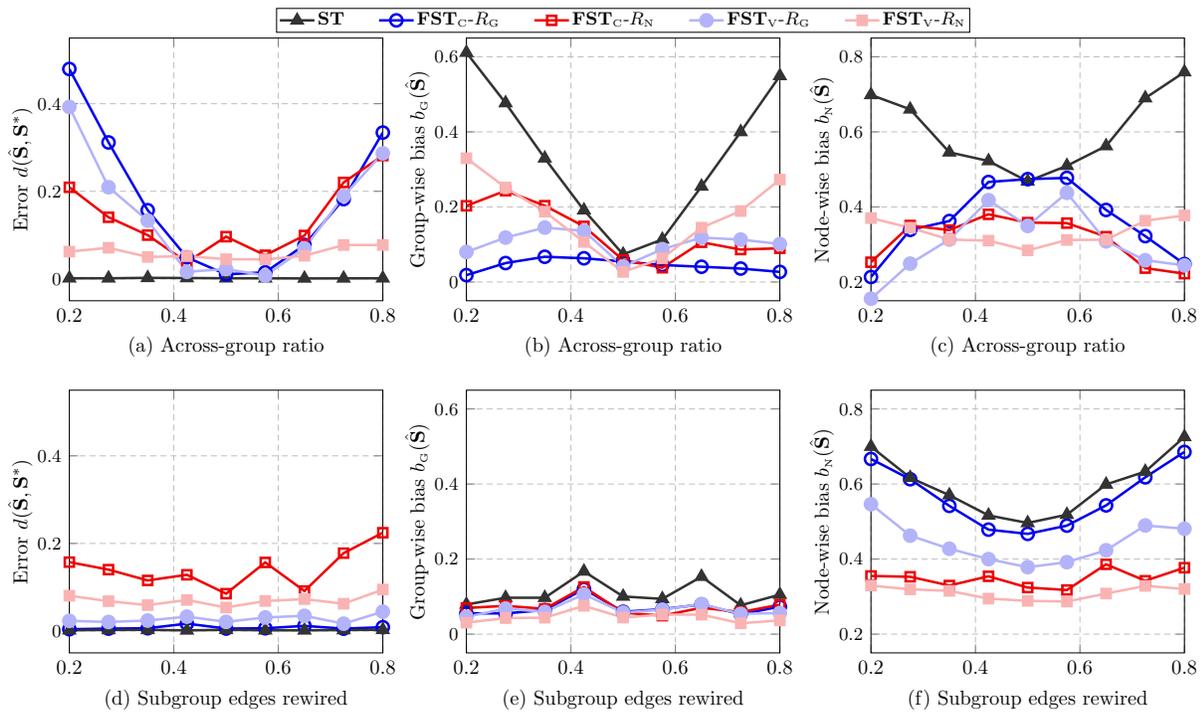

\subsection{Group Versus Individual Fairness}
\label{ss:group-v-indiv}

We first estimate unbiased graphs by imposing either group fairness via $\Rgr$ in~\eqref{eq:Rgr} or individual fairness via $\Rno$ in~\eqref{eq:Rno} as the connectivity of the target graph $\bbS^*$ varies.
We compare $\st$ to all four variants of FairSpecTemp,~\eqref{eq:opt_C_l1} and~\eqref{eq:opt_V_P_l1} with $R \in \{ \Rgr,\Rno \}$.
%
We consider $N=40$ nodes, $G=2$ groups, $M=10^5$ graph signals, and, for a fixed number of edges, we modify the topology in $\bbS^*$ to vary the bias $\Rgr(\bbS^*)$ or $\Rno(\bbS^*)$.
In particular, we examine the following two scenarios.

First, we vary the ratio of across-group edges, altering both $\Rgr(\bbS^*)$ and $\Rno(\bbS^*)$.
We start with a graph that is highly modular with respect to groups, where the fraction of across-group edges is 0.2, visualized in Figure~\ref{fig:synth_topbias_vis}(a).
Then, we randomly rewire an increasing number of within-group edges to connect nodes across groups until the across-group fraction is 0.8, that is, the graph is highly bipartite with respect to groups, shown in Figure~\ref{fig:synth_topbias_vis}(c).
The graph is fair with respect to both $\Rgr(\bbS^*)$ and $\Rno(\bbS^*)$ when the across-group edge ratio is 0.5 in Figure~\ref{fig:synth_topbias_vis}(b), as each node has an equal likelihood of connecting to either group.
We plot the error, group-wise bias, and node-wise bias for the estimated GSOs in Figure~\ref{fig:synth_topbias}(a)-(c) corresponding to the scenario in Figure~\ref{fig:synth_topbias_vis}(a)-(c).
As expected, all FairSpecTemp approaches achieve fairer GSOs than $\st$.
However, when $\bbS^*$ is unfair, that is, the across-group edge ratio is closer to 0 or 1, the improved bias comes at a sacrifice in accuracy, particularly for $\fstcg$ and $\fstvg$, which yield the lowest group-wise bias $b_{\scriptscriptstyle\rm G}$.
Conversely, when $\bbS^*$ is fair with an across-group edge ratio of 0.5, only $\fstcg$ and $\fstvg$ rival $\st$ in both error and group-wise bias $b_{\scriptscriptstyle\rm G}$.
Interestingly, $\fstvg$ also improves the node-wise bias $b_{\scriptscriptstyle\rm N}$, aligning with our intuition that the increased degrees of freedom for~\eqref{eq:opt_V_P_l1} allows greater flexibility in estimation, potentially mitigating the tradeoff between fairness and accuracy.
%

Second, we consider a low $\Rgr(\bbS^*)$ but a varying $\Rno(\bbS^*)$.
To this end, we partition both groups into two subgroups.
We first generate a graph such that one subgroup per group exhibits high within-group connectivity, and we randomly rewire an increasing subset of these edges to connect across groups.
Similarly, the remaining two subgroups exhibit high across-group connectivity, which are increasingly rewired to become within-group edges.
Since the number of within- versus across-group edges remain balanced, the group-wise bias $\Rgr(\bbS^*)$ remains low, but if the fraction of rewired edges is close to 0 or 1 as depicted in Figure~\ref{fig:synth_topbias_vis}(d) and (f), then $\Rno(\bbS^*)$ will be high, as nodes show strong preferences for either within- or across-group connections.
Figure~\ref{fig:synth_topbias_vis}(e) shows when $\Rno(\bbS^*)$ is low, that is, when the fraction of rewired edges is 0.5, as all nodes demonstrate balanced connectivity between both groups.
Graph estimation performance for this setting is shown in Figure~\ref{fig:synth_topbias}(d)-(f).
Since $\Rgr(\bbS^*)$ is always low, the group-wise fair methods $\fstcg$ and $\fstvg$ obtain estimates similarly as accurate as $\st$, and all methods obtain a low group-wise bias $b_{\rm\scriptscriptstyle G}$.
However, only $\fstcn$ and $\fstvn$ also reduce the node-wise bias $b_{\rm\scriptscriptstyle N}$, with the greatest sacrifice in accuracy for $\fstcn$.
Indeed, Figure~\ref{fig:synth_topbias} demonstrates scenarios where~\eqref{eq:opt_V_P_l1} exhibits superior performance, as $\fstvn$ shows the lowest bias in terms of both $b_{\rm \scriptscriptstyle G}$ and $b_{\rm \scriptscriptstyle N}$ with less error than $\fstcn$, while $\fstvg$ may incur greater $b_{\rm \scriptscriptstyle N}$, yet it remains lower than that of $\fstcg$.
Figure~\ref{fig:synth_topbias} illustrates scenarios in which we may prefer encouraging group versus individual fairness by choice of $R$; in the remaining experiments, we focus on the more common goal of group fairness, that is, we measure estimation bias with $b_{\rm \scriptscriptstyle G}$.

\subsection{Performance as Graph Size Varies}
\label{ss:vary_size}

To further investigate the differences between FairSpecTemp in~\eqref{eq:opt_C_l1} and~\eqref{eq:opt_V_P_l1}, we next consider estimating unbiased graphs as the number of nodes increases.
We generate Erdos-Renyi (ER) graphs with $G=2$ groups and $M=10^6$ stationary graph signals while increasing the number of nodes $N$.
To maintain comparable sparsity as $N$ varies, we let the ER edge probabilities be $p = \frac{5}{N-1}$ for a constant expected degree of 5.
We compare graph estimation with and without the fairness constraints: $\st$, $\fstcg$, and $\fstvg$.
Figure~\ref{fig:synth_biaserror}(a) reports both the estimation error $d(\hbS,\bbS^*)$ (left $y$-axis) and the group-wise bias $b_{\rm \scriptscriptstyle G}(\hbS)$ (right $y$-axis).
As expected, both FairSpecTemp approaches achieve lower bias than $\st$, but $\fstcg$ with an explicit bias constraint $\Rgr(\hbS) \leq \tau^2$ exhibits more stability in both bias and error across all graph sizes $N$.
In contrast, $\fstvg$ maintains a higher error than both $\fstcg$ and $\st$, yet only $\fstvg$ decreases both error and bias as $N$ increases, with a more competitive error between $\fstvg$ and $\fstcg$ at $N = 100$.
Thus, while $\fstcg$ produces the more expected result of balancing bias and error, the more flexible $\fstvg$ can obtain fairer graphs for larger graph sizes.
%

\subsection{Estimating Fair Graphs from Biased Data}
\label{ss:biased_data}

We next consider a setting likely to be encountered in real-world data: The graph structure to be estimated is fair, but the observed data is biased.
To simulate an increasing level of within-group preference in the graph data, we first generate an ER graph with GSO $\bbS^*$ of $N=100$ nodes and edge probability $p=0.05$ so the connections are fair with respect to the $G=2$ nodal groups.
Then, we progressively bias the connections in $\bbS^*$ by increasing within-group edge weights while simultaneously decreasing across-group edge weights, and we generate increasingly biased sets of $M=10^6$ samples from the perturbed graphs.
We plot the error and group-wise bias for the estimated graphs in Figure~\ref{fig:synth_biaserror} as the level of bias in the data increases.
Each curve corresponds to a method $\st$, $\stba$, $\fstcg$, $\fstvg$, and $\fstcn$, and filled, darker markers represent the least biased data.

We observe that $\st$ sees the greatest degradation in terms of both accuracy and fairness, whereas all FairSpecTemp methods become less accurate as the data grows more biased but maintain lower bias in estimation.
Since we inject bias in data via imbalanced edge weights, we examine how effectively $\stba$ can mitigate this effect by rebalancing edge weights.
As expected, reweighting edges in $\stba$ may reduce bias, but altering the estimate from $\st$ can require a larger sacrifice of accuracy for a sufficiently fair outcome.
As for the node-wise $\fstcn$, we consistently see slightly higher bias and error relative to both $\fstcg$ and $\fstvg$.
This is particularly noticeable when the bias in the data is small (filled markers in Figure~\ref{fig:synth_biaserror}(b)), as the stricter node-wise constraint for $\fstcn$ may struggle to achieve a similarly low bias without a greater sacrifice in accuracy, as shown in Figure~\ref{fig:synth_topbias}.
We also note a tradeoff between $\fstcg$ and $\fstvg$ as the observed data differs further from the true, fair distribution.
When the data is fairer, both $\fstcg$ and $\fstvg$ achieve Pareto optimal solutions with comparable fairness and accuracy.
However, as the bias in the data grows, $\fstvg$ maintains the fairest estimates but incurs increasingly greater error, whereas graphs estimated via $\fstcg$ may not be as fair, but in comparison with $\st$, $\fstcg$ maintains a competitive accuracy with far lower bias.
\draft{ 
}

\begin{figure*}[t]
    \centering
    \begin{minipage}[b][][b]{0.3\textwidth}
        \scalebox{.85}{\begin{tikzpicture}[baseline,scale=0.9,trim axis left, trim axis right]

\pgfplotstableread{data/exp3_results_median.csv}\errtable

\definecolor{btr}{RGB}{255,255,255}

\definecolor{bnp}{RGB}{8, 48, 107}     
\definecolor{bdp}{RGB}{8, 81, 156}     
\definecolor{bnw}{RGB}{33, 102, 172}    

\definecolor{fnp}{RGB}{103, 0, 13}      
\definecolor{fdp}{RGB}{178, 24, 43}     
\definecolor{fnw}{RGB}{214, 96, 77}     

\definecolor{fax}{RGB}{200,0,3}
\definecolor{bax}{RGB}{4,50,255}

\begin{axis}[
    xlabel={},
    xtick={},
    xmin=30,
    xmax=100,
    axis y line*=right,
    ylabel={\textcolor{Red2}{Group-wise bias $b_{\scriptscriptstyle \rm G}(\hbS)$}},
    ylabel shift=-10pt,
    ylabel style={ rotate=180 },
    xticklabel style={font=\footnotesize},
    yticklabel style={
        Red2,
        /pgf/number format/fixed,
        /pgf/number format/precision=2,
        font=\footnotesize
    },
    ymin=2e-2,
    ymax=0.18,
    grid style=densely dashed,
    grid=major,
    legend style={
        at={(.5,1.02)},
        anchor=south,
        font=\footnotesize},
    legend columns=6,
    width=180,
    height=175,
    xtick pos=left,
    ytick pos=right,
    label style={font=\small},
    tick label style={font=\small}
    ]
    
    \addlegendimage{empty legend}
        \label{bias_empty}
    \addplot[black!40!white, mark=triangle*, densely dashed, mark size=3pt] 
        table [x=NN, y=ST-group-bias] {\errtable};
        \label{bias_st}
    \addplot[Red2, mark=*, densely dashed] 
        table [x=NN, y=FST-group-bias] {\errtable};
        \label{bias_fstcg}
    \addplot[white!60!Red2, mark=square*, densely dashed] 
        table [x=NN, y=FST-Eig-group-bias] {\errtable};
        \label{bias_fstvg}

\end{axis}

\begin{axis}[
    xlabel={(a) Number of nodes $N$},
    xmin=30,
    xmax=100,
    axis y line*=left,
    ylabel={\textcolor{Blue2}{Error $d(\hbS,\bbS^*)$}},
    xticklabel style={font=\footnotesize},
    yticklabel style={
        Blue2,
        /pgf/number format/fixed,
        /pgf/number format/precision=2,
        font=\footnotesize
    },
    ymin=1.1e-4,
    ymax=2,
    ymode=log,
    legend style={
        at={(.5,1.02)},
        anchor=south,
        font=\footnotesize},
    legend columns=4,
    width=180,
    height=175,
    xtick pos=left,
    ytick pos=left,
    label style={font=\small},
    tick label style={font=\small}
    ]
    
    \addlegendimage{empty legend}
        \label{err_empty}
     \addplot[black!85!white, mark=triangle, solid, mark size=3pt] 
        table [x=NN, y=ST-group-error] {\errtable};
        \label{err_st}
    \addplot[Blue2, mark=o, solid] 
        table [x=NN, y=FST-group-error] {\errtable};
        \label{err_fstcg}
    \addplot[white!60!Blue2, mark=square, solid] 
        table [x=NN, y=FST-Eig-group-error] {\errtable};
        \label{err_fstvg}

    \addlegendimage{empty legend}
    \addlegendimage{/pgfplots/refstyle=bias_st}
    \addlegendimage{/pgfplots/refstyle=bias_fstcg}
    \addlegendimage{/pgfplots/refstyle=bias_fstvg}
    
    \addlegendentry{\textbf{Left axis:}~~}
    \addlegendentry{$\st$}
    \addlegendentry{$\fstcg$}
    \addlegendentry{$\fstvg$}
    \addlegendentry{\textbf{Right axis:}~~}
    \addlegendentry{$\st$}
    \addlegendentry{$\fstcg$}
    \addlegendentry{$\fstvg$}
\end{axis}

\end{tikzpicture}}
    \end{minipage}
    \hspace{0.9cm}
    \begin{minipage}[b][][b]{0.3\textwidth}
        \scalebox{.85}{\begin{tikzpicture}[baseline,scale=.9,trim axis left, trim axis right]

\pgfmathsetmacro{\initmarksize}{3.5}

\pgfplotstableread{data/exp2_bias_error_table_rw2.csv}\errbiastable

\begin{axis}[
    xlabel={(b) Error $d(\hbS,\bbS^*)$},
    ylabel={Group-wise bias $b_{\scriptscriptstyle \rm G}(\hbS)$},
    xticklabel style={font=\footnotesize},
    yticklabel style={
        /pgf/number format/fixed,
        /pgf/number format/precision=2,
        font=\footnotesize
    },
    ymax=1.55,
    grid style=densely dashed,
    grid=both,
    legend style={
        at={(.5,1.02)},
        anchor=south,
        font=\footnotesize},
    legend columns=3,
    width=180,
    height=175,
    label style={font=\small},
    tick label style={font=\small}
    ]

    \addplot[black!80!white, mark=triangle, solid, mark size=2pt] table [x=ST-group-error, y=ST-group-bias] {\errbiastable};
    \addplot[mark options={black!40!white}, mark repeat=20, mark size=\initmarksize, only marks, mark=triangle*, forget plot] table [x=ST-group-error, y=ST-group-bias] {\errbiastable};

    \addplot[Blue2, mark=o, solid] table [x=FST-group-error, y=FST-group-bias] {\errbiastable};
    \addplot[mark options={black!30!Blue2}, mark repeat=20, mark size=\initmarksize, only marks, mark=*, forget plot] table [x=FST-group-error, y=FST-group-bias] {\errbiastable};

    \addplot[Red2, mark=square, solid, mark size=1.5pt]  table [ x=FST-node-error, y=FST-node-bias] {\errbiastable};
    \addplot[mark options={black!30!Red2}, mark repeat=20, mark size=\initmarksize, only marks, mark=square*, forget plot] table [x=FST-node-error, y=FST-node-bias] {\errbiastable};

    \addplot[white!20!Purple2, mark=diamond, solid, mark size=2pt] table [x=Rw-group-error, y=Rw-group-bias] {\errbiastable};
    \addplot[mark options={black!20!Purple2}, mark repeat=20, mark size=\initmarksize, only marks, mark=diamond*, forget plot] table [x=Rw-group-error, y=Rw-group-bias] {\errbiastable};
    

    \addplot[white!70!Blue2, mark=o, solid] table [x=ST-Eig-group-error, y=ST-Eig-group-bias] {\errbiastable};
    \addplot[mark options={white!40!Blue2}, mark repeat=20, mark size=\initmarksize, only marks, mark=*, forget plot] table [x=ST-Eig-group-error, y=ST-Eig-group-bias] {\errbiastable};


    \addlegendentry{$\st$~}
    \addlegendentry{$\fstcg$~}
    \addlegendentry{$\fstcn$}
    \addlegendentry{$\stba$~}
    \addlegendentry{$\fstvg$}
\end{axis}

\end{tikzpicture}}
    \end{minipage}
    \hspace{0.1cm}
    \begin{minipage}[b][][b]{0.3\textwidth}
        \scalebox{.85}{\begin{tikzpicture}[baseline,scale=.9,trim axis left, trim axis right]

\pgfmathsetmacro{\initmarksize}{3.5}

\pgfplotstableread{data/exp1_errors_and_bias.csv}\errbiastable

\begin{axis}[
    xlabel={(c) Error $d(\hbS,\bbS^*)$},
    ylabel={Group-wise bias $b_{\scriptscriptstyle \rm G}(\hbS)$},
    xticklabel style={font=\footnotesize},
    yticklabel style={
        /pgf/number format/fixed,
        /pgf/number format/precision=2,
        font=\footnotesize
    },
    grid style=densely dashed,
    grid=both,
    legend style={
        at={(.5,1.02)},
        anchor=south,
        font=\footnotesize},
    legend columns=3,
    width=180,
    height=175,
    label style={font=\small},
    tick label style={font=\small}
    ]
    
    \addplot[black!80!white, mark=triangle, solid, mark size=2pt] table [x=mean-errors, y=mean-est-bias] {\errbiastable};
    \addplot[mark options={black!40!white}, mark repeat=20, mark size=\initmarksize, only marks, mark=triangle*, forget plot] table [x=mean-errors, y=mean-est-bias] {\errbiastable};

    \addplot[Blue2, mark=o, solid] table [x=mean-ferrors, y=mean-fest-bias] {\errbiastable};
    \addplot[mark options={black!30!Blue2}, mark repeat=20, mark size=\initmarksize, only marks, mark=*, forget plot] table [x=mean-ferrors, y=mean-fest-bias] {\errbiastable};

    \addplot[Red2, mark=square, solid, mark size=1.5pt]  table [ x=mean-ferrorsn, y=mean-fest-biasn] {\errbiastable};
    \addplot[mark options={black!30!Red2}, mark repeat=20, mark size=\initmarksize, only marks, mark=square*, forget plot] table [x=mean-ferrorsn, y=mean-fest-biasn] {\errbiastable};
    
    \addplot[black!5!Gold2, mark=diamond, solid, mark size=2pt] table [x=mean-rw-errors, y=mean-rw-est-bias] {\errbiastable};
    \addplot[mark options={black!20!Gold2}, mark repeat=20, mark size=\initmarksize, only marks, mark=diamond*, forget plot] table [x=mean-rw-errors, y=mean-rw-est-bias] {\errbiastable};

    \addplot[white!70!Blue2, mark=o, solid] table [x=mean-errors-eig, y=mean-est-bias-eig] {\errbiastable};
    \addplot[mark options={white!40!Blue2}, mark repeat=20, mark size=\initmarksize, only marks, mark=*, forget plot] table [x=mean-errors-eig, y=mean-est-bias-eig] {\errbiastable};

    \addplot[black!10!Green2, mark=triangle, solid, mark options={rotate=180}, mark size=3pt] table [x=mean-errors-FGL, y=mean-est-bias-FGL] {\errbiastable};
    \addplot[black!50!Green2, mark repeat=20, mark size=\initmarksize, only marks, mark=triangle*, forget plot, mark options={rotate=180}] table [x=mean-errors-FGL, y=mean-est-bias-FGL] {\errbiastable};

    \addlegendentry{$\st$}
    \addlegendentry{$\fstcg$}
    \addlegendentry{$\fstcn$}
    \addlegendentry{$\strw$}
    \addlegendentry{$\fstvg$}
    \addlegendentry{$\fgl$}
\end{axis}

\end{tikzpicture}}
    \end{minipage}
    \caption{Performance of estimates $\hbS$ for different graph estimation methods under varying conditions. Bias $b_{\rm\scriptscriptstyle G}(\hbS)$ and error $d(\hbS,\bbS^*)$ (a) as the number of nodes $N$ increases, (b) as the data becomes more biased toward within-group connections,  
    (c) as the number of samples $M$ increases.}
    \label{fig:synth_biaserror}
\end{figure*}
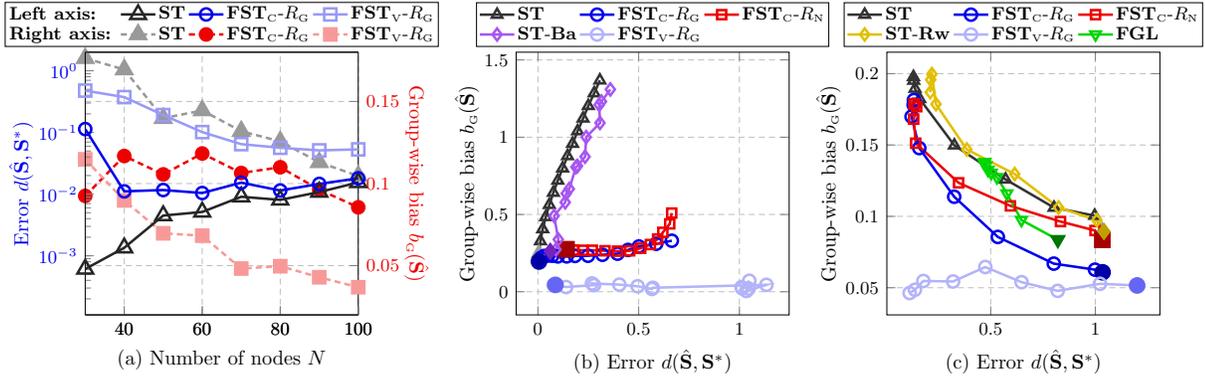

\subsection{Comparing Fair Graph Estimation Approaches}
\label{ss:baseline}

We compare FairSpecTemp in~\eqref{eq:opt_C_l1} and~\eqref{eq:opt_V_P_l1} with different fair graph estimation baselines, where we estimate ER graphs of $N=100$ nodes and an edge probability $p=0.05$ for $G=2$ imbalanced groups.
We show bias and error as the number of stationary graph signals increases from $M=10^2$ (filled markers) to $M=10^6$ in Figure~\ref{fig:synth_biaserror}(c).
Since the true graph $\bbS^*$ is biased due to imbalanced groups, a consistent trend across methods is an increase in both accuracy and group-wise bias as $M$ increases.
In general, all fair graph methods achieve lower bias than $\st$ for all values of $M$.
The random edge rewirings in $\strw$ can achieve slight improvements in bias relative to $\st$, but the additional stochasticity also introduces a slight increase in error.
While $\fgl$ obtains consistently fairer graphs than $\st$, it assumes a stricter signal model than our graph-stationarity assumption, yielding worse error at higher $M$.
As for FairSpecTemp methods, we first note the value of direct bias constraints via~\eqref{eq:opt_C_l1}: $\fstcg$ and $\fstcn$ both achieve lower bias than $\st$ while maintaining accuracy.
However, since $\Rno$ is stricter than $\Rgr$, $\fstcn$ requires a looser constraint to maintain competitive accuracy, hence the fairer estimates from $\fstcg$.
We also observe the different advantages between FairSpecTemp in~\eqref{eq:opt_C_l1} and~\eqref{eq:opt_V_P_l1}.
When $M$ is low, our estimates of the covariance matrix eigenvectors $\hbV$ will be poor and therefore detrimental for the flexible $\fstvg$, therefore applying $\fstcg$ is preferable.
However, in a high-sample regime, we observe the power of increased degrees of freedom in $\fstvg$, as we achieve competitive accuracy in estimation, but we see a dramatic improvement in bias for $\fstvg$ relative to all other methods.
This reveals the value of both variants of FairSpecTemp.
We may require explicit bias constraints via $\fstcg$, but certain settings allow us to mitigate the tradeoff between fairness and accuracy using $\fstvg$.

\subsection{Estimating Fair Graphs for Investing}
\label{ss:finance}

Finally, we apply graph estimation and our group-wise bias metric $\Rgr$ for a real-world financial investment task, where we estimate graphs from real-world market data.
Specifically, we consider graphs connecting companies from the S\&P 500 index with groups corresponding to different sectors, and, as our nodal observations, we observe company log-returns indicating changes in market value over time.
A given set of $N$ companies and $M$ daily log-return values comprises a data matrix $\bbX \in \reals^{N \times M}$ from which we estimate graphs indicating which companies exhibit similarities in their return profiles.
Thus, a common approach is to use graph structure to design investment strategies, such as investing when the graph exhibits low connectivity since minimizing correlations among investments is key to reducing risk~\citep{cardoso2020learning}.
We instead consider employing $b_{\rm\scriptscriptstyle G}(\hbS)$ by investing when the \emph{group-wise bias} is sufficiently high.
Indeed, since companies typically exhibit more similar behavior within the same sector, we may expect a higher $b_{\rm\scriptscriptstyle G}(\hbS)$ in general due to within-group edge preferences.
Thus, a low $b_{\rm\scriptscriptstyle G}(\hbS)$ implies a notable increase in similarities across sectors, indicating a riskier period such as the COVID-19 pandemic.
We proceed with financial investment in two scenarios, one short-term with frequent investments opportunities and another, longer-term setting with less frequent graph evaluations.

The first setup involves a three-month window in which we estimate a graph $\hbS_t$ every other day via sliding windows of log-return observations from February 2020 to May 2020, where the $N=71$ nodes represent companies from $G = 3$ sectors: Communication Services, Materials, and Energy, chosen for their relative balance in group size.
We obtain a sequence of graphs $\{\hbS_t\}_{t=1}^{38}$, from which we estimate the time-varying group-wise bias $\{ b_{\rm\scriptscriptstyle G}(\hbS_t) \}_{t=1}^{38}$.
If the bias exceeds a threshold, then this implies greater diversity across sectors and we invest; otherwise, we do not.
Figure~\ref{fig:invest}(a) shows the value of investment for different graph estimation approaches and the group-wise bias of the estimated graphs.
We compare $\st$, $\strw$, $\fstcg$, and $\fstcn$, along with two more traditional approaches, one estimating a correlation matrix across companies, referred to as \textbf{Corr.}, and the baseline \textbf{Strategy I}, which refers to investing and selling every other day, regardless of bias $b_{\rm\scriptscriptstyle G}(\hbS_t)$~\citep{cardoso2020learning}.
For a fair comparison, each method is assigned the threshold that yields the largest possible final investment value.
Figure~\ref{fig:invest}(a) illustrates that $\fstcg$ clearly outperforms its competitors with the highest value at the end of the investment period, with $\fstcn$ and $\st$ as the next best.
A closer look at the bias in Figure~\ref{fig:invest}(a) shows that $\fstcg$, $\fstcn$, and $\st$ show similar trends of lower values of $b_{\rm\scriptscriptstyle G}(\hbS_t)$ around March 2020, the period in which investment was unwise based on the steep decrease in value of the naive \textbf{Strategy I}.
However, since $\st$ and $\fstcn$ do not explicitly encourage low group-wise bias, they experience greater noise in $b_{\rm\scriptscriptstyle G}(\hbS_t)$ over time, which can yield false positive investment flags.
In contrast, $\fstcg$ returns a consistently more stable sequence $b_{\rm\scriptscriptstyle G}(\hbS_t)$, so higher values of bias are more indicative of appropriate investment opportunities.

For the second setup, we consider a more challenging scenario with a longer period, from January 2019 to December 2022; less frequent investment opportunities, estimating graphs weekly for graph sequences of length $170$; and a larger number of companies $N = 112$ spanning a different set of sectors: Real Estate, Energy, and Healthcare.
We again present the investment value and estimated bias $b_{\rm\scriptscriptstyle G}(\hbS)$ in Figure~\ref{fig:invest}(b), which support our observations from the short-term setup in Figure~\ref{fig:invest}(a).
Prior to the COVID-19 pandemic, all methods behave similarly; however, starting early 2020, only $\st$, $\fstcg$, and $\fstcn$ both maintain and increase their value in comparison with $\strw$, \textbf{Corr.}, and \textbf{Strategy~I}, which experience a large decrease in value at the beginning of the pandemic.
Here, $\fstcn$ obtains the highest return value, closely followed by $\fstcg$ and $\st$.
In such a setting with more companies and larger gaps in between investment opportunities, we find a potential advantage to imposing node-wise fairness by specifically encouraging each node show diversity in similarity across sectors.
This more explicit encouragement of diversity in connectivity may be necessary for larger-scale, volatile situations.
Thus, we not only demonstrate the value of our bias metrics for real-world applications, where we impose more informative structural rules for investment that lead to reliable strategies, but we also show that FairSpecTemp performs well with our proposed investment strategy.

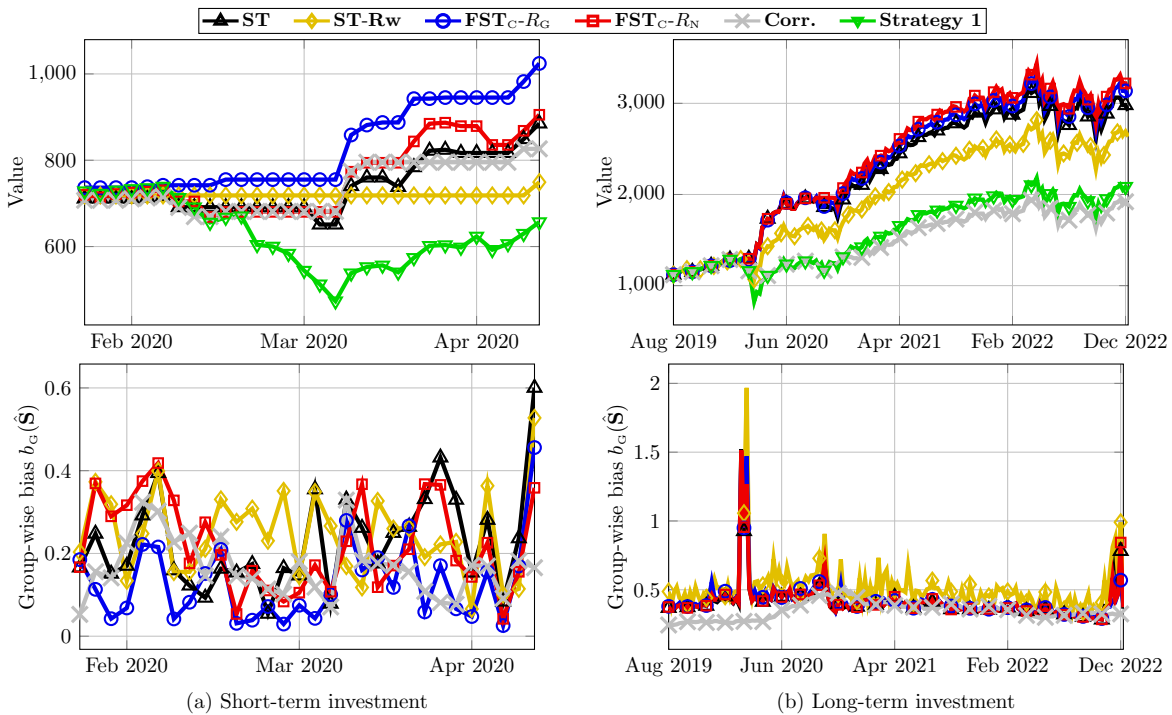
\begin{figure*}[t]
    \centering
    \begin{minipage}[b]{.3\textwidth}
        \begin{tikzpicture}[baseline,scale=.8,trim axis left, trim axis right]

\pgfmathsetmacro{\initmarksize}{3.5}

\pgfplotstableread{data/finance_results.csv}\errbiastable

\begin{axis}[
    xmin=18,
    xmax=76,
    ylabel={Value},
    xtick={2,24,46,68},
    xticklabels={Jan 2020, Feb 2020, Mar 2020, Apr 2020},
    grid=both,
    legend style={
        at={(.25, 1.02)},
        anchor=south west,
        font=\footnotesize},
    legend columns=6,
    width=260,
    height=180,
    label style={font=\small},
    tick label style={font=\small}
    ]

    \addplot[Black, line width=2, solid, forget plot]  table [ x=time, y=st ] {\errbiastable};
    \addplot[Black, mark=triangle, line width=1, only marks, mark size=2.5pt, forget plot]  table [ x=time, y=st ] {\errbiastable};
    \addplot[Black, mark=triangle, line width=1.5, solid, mark size=2pt] coordinates {(-10,0)};
    
    \addplot[black!5!Gold2, line width=2, solid, forget plot] table [x=time, y=strw] {\errbiastable};
    \addplot[black!5!Gold2, mark=diamond, line width=1, only marks, mark size=2.5pt, forget plot] table [x=time, y=strw] {\errbiastable};
    \addplot[black!5!Gold2, mark=diamond, line width=1.5, solid, mark size=2pt] coordinates {(-10,0)};

    \addplot[Blue2, line width=2, solid, forget plot] table [x=time, y=fstg] {\errbiastable};
    \addplot[Blue2, mark=o, line width=1, only marks, solid, forget plot] table [x=time, y=fstg] {\errbiastable};
    \addplot[Blue2, mark=o, line width=1.5, solid] coordinates {(-10,0)};

    \addplot[Red2, line width=2, solid, forget plot] table [x=time, y=fstn] {\errbiastable};
    \addplot[Red2, mark=square, line width=1, only marks, mark size=1.5, forget plot] table [x=time, y=fstn] {\errbiastable};
    \addplot[Red2, mark=square, line width=1.5, solid, mark size=1.5] coordinates {(-10,0)};

    \addplot[Gray0, line width=2, solid, forget plot]  table [ x=time, y=corr ] {\errbiastable};
    \addplot[Gray0, mark=x, line width=1.5, only marks, mark size=3.5pt, forget plot]  table [ x=time, y=corr ] {\errbiastable};
    \addplot[Gray0, mark=x, line width=1.5, solid, mark size=3pt] coordinates {(-10,0)};

    \addplot[black!10!Green2, line width=2, solid, forget plot]  table [ x=time, y=strat1 ] {\errbiastable};
    \addplot[black!10!Green2, mark=triangle, line width=1, only marks, mark options={rotate=180}, mark size=3.5pt, forget plot] table [ x=time, y=strat1 ] {\errbiastable};
    \addplot[black!10!Green2, mark=triangle, line width=1.5, solid, mark options={rotate=180}, mark size=2.5pt] coordinates {(-10,0)};

    \addlegendentry{$\st$~~~~}
    \addlegendentry{$\strw$~~~~}
    \addlegendentry{$\fstcg$~~~~}
    \addlegendentry{$\fstcn$~~~~}
    \addlegendentry{\textbf{Corr.}~~~~}
    \addlegendentry{\textbf{Strategy 1}}
\end{axis}

\end{tikzpicture}
    \end{minipage}
    \hspace{3cm}
    \begin{minipage}[b]{.3\textwidth}
        \begin{tikzpicture}[baseline,scale=.8,trim axis left, trim axis right]

\pgfmathsetmacro{\initmarksize}{3.5}

\pgfplotstableread{data/finance_results_5d.csv}\errbiastable

\begin{axis}[
    xmin=1,
    xmax=170,
    ylabel={Value},
    xtick={1,43,85,127,169},
    xticklabels={Aug 2019, Jun 2020, Apr 2021, Feb 2022, Dec 2022},
    grid=both,
    legend style={
        at={(.25, 1.02)},
        anchor=south west,
        font=\footnotesize},
    legend columns=6,
    width=260,
    height=180,
    label style={font=\small},
    tick label style={font=\small}
    ]

    \addplot[Black, line width=2, solid, forget plot]  table [ x=time, y=st ] {\errbiastable};
    \addplot[Black, mark=triangle, line width=1, only marks, mark repeat=7, mark size=2.5pt, forget plot]  table [ x=time, y=st ] {\errbiastable};
    \addplot[Black, mark=triangle, line width=1.5, solid, mark size=2pt] coordinates {(-10,0)};

    \addplot[black!5!Gold2, line width=2, solid, forget plot] table [x=time, y=strw] {\errbiastable};
    \addplot[black!5!Gold2, mark=diamond, line width=1, mark repeat=5, only marks, mark size=2.5pt, forget plot] table [x=time, y=strw] {\errbiastable};
    \addplot[black!5!Gold2, mark=diamond, line width=1.5, solid, mark size=2pt] coordinates {(-10,0)};

    \addplot[Blue2, line width=2, solid, forget plot] table [x=time, y=fstg] {\errbiastable};
    \addplot[Blue2, mark=o, line width=1, mark repeat=7, only marks, solid, forget plot] table [x=time, y=fstg] {\errbiastable};
    \addplot[Blue2, mark=o, line width=1.5, solid] coordinates {(-10,0)};

    \addplot[Red2, line width=2, solid, forget plot] table [x=time, y=fstn] {\errbiastable};
    \addplot[Red2, mark=square, line width=1, mark repeat=7, only marks, mark size=1.5, forget plot] table [x=time, y=fstn] {\errbiastable};
    \addplot[Red2, mark=square, line width=1.5, solid, mark size=1.5] coordinates {(-10,0)};

    \addplot[Gray0, line width=2, solid, forget plot]  table [ x=time, y=corr ] {\errbiastable};
    \addplot[Gray0, mark=x, line width=1.5, mark repeat=7, only marks, mark size=3.5pt, forget plot]  table [ x=time, y=corr ] {\errbiastable};
    \addplot[Gray0, mark=x, line width=1.5, solid, mark size=3pt] coordinates {(-10,0)};

    \addplot[black!10!Green2, line width=2, solid, forget plot]  table [ x=time, y=str1 ] {\errbiastable};
    \addplot[black!10!Green2, mark=triangle, line width=1, mark repeat=7, only marks, mark options={rotate=180}, mark size=3.5pt, forget plot] table [ x=time, y=str1 ] {\errbiastable};
    \addplot[black!10!Green2, mark=triangle, line width=1.5, solid, mark options={rotate=180}, mark size=2.5pt] coordinates {(-10,0)};

\end{axis}

\end{tikzpicture}
    \end{minipage}
    \\
    \begin{minipage}[b]{.3\textwidth}
        \begin{tikzpicture}[baseline,scale=.8,trim axis left, trim axis right]

\pgfmathsetmacro{\initmarksize}{3.5}

\pgfplotstableread{data/finance_algconn.csv}\errbiastable

\begin{axis}[
    xlabel={(a) Short-term investment},
    xmin=18,
    xmax=76,
    ylabel={Group-wise bias $b_{\scriptscriptstyle \rm G}(\hbS)$},
    xtick={2,24,46,68},
    xticklabels={Jan 2020, Feb 2020, Mar 2020, Apr 2020},
    grid=both,
    legend style={
        at={(.5, 1.02)},
        anchor=south west,
        font=\footnotesize},
    legend columns=5,
    width=260,
    height=180,
    label style={font=\small},
    tick label style={font=\small}
    ]

    \addplot[Black, line width=2, solid, forget plot]  table [ x=time, y=st ] {\errbiastable};
    \addplot[Black, mark=triangle, line width=1, only marks, mark size=2.5pt, forget plot]  table [ x=time, y=st ] {\errbiastable};
    \addplot[Black, mark=triangle, line width=1.5, solid, mark size=2pt] coordinates {(-10,0)};

    \addplot[black!5!Gold2, line width=2, solid, forget plot] table [x=time, y=strw] {\errbiastable};
    \addplot[black!5!Gold2, mark=diamond, line width=1, only marks, mark size=2.5pt, forget plot] table [x=time, y=strw] {\errbiastable};
    \addplot[black!5!Gold2, mark=diamond, line width=1.5, solid, mark size=2pt] coordinates {(-10,0)};

    \addplot[Blue2, line width=2, solid, forget plot] table [x=time, y=fstg] {\errbiastable};
    \addplot[Blue2, mark=o, line width=1, only marks, solid, forget plot] table [x=time, y=fstg] {\errbiastable};
    \addplot[Blue2, mark=o, line width=1.5, solid] coordinates {(-10,0)};

    \addplot[Red2, line width=2, solid, forget plot] table [x=time, y=fstn] {\errbiastable};
    \addplot[Red2, mark=square, line width=1, only marks, mark size=1.5, forget plot] table [x=time, y=fstn] {\errbiastable};
    \addplot[Red2, mark=square, line width=1.5, solid, mark size=1.5] coordinates {(-10,0)};

    \addplot[Gray0, line width=2, solid, forget plot]  table [ x=time, y=corr ] {\errbiastable};
    \addplot[Gray0, mark=x, line width=1.5, only marks, mark size=3.5pt, forget plot]  table [ x=time, y=corr ] {\errbiastable};
    \addplot[Gray0, mark=x, line width=1.5, solid, mark size=3pt] coordinates {(-10,0)};

\end{axis}

\end{tikzpicture}
    \end{minipage}
    \hspace{3cm}
    \begin{minipage}[b]{.3\textwidth}
        \begin{tikzpicture}[baseline,scale=.8,trim axis left, trim axis right]

\pgfmathsetmacro{\initmarksize}{3.5}

\pgfplotstableread{data/alg_conn_exp4.csv}\errbiastable

\begin{axis}[
    xlabel={(b) Long-term investment},
    xmin=1,
    xmax=170,
    ylabel={Group-wise bias $b_{\scriptscriptstyle \rm G}(\hbS)$},
    xtick={1,43,85,127,169},
    xticklabels={Aug 2019, Jun 2020, Apr 2021, Feb 2022, Dec 2022},
    grid=both,
    legend style={
        at={(.5,1.02)},
        anchor=south,
        font=\footnotesize},
    legend columns=5,
    width=260,
    height=180,
    label style={font=\small},
    tick label style={font=\small}
    ]

    \addplot[Black, line width=2, solid, forget plot]  table [ x=time, y=st ] {\errbiastable};
    \addplot[Black, mark=triangle, line width=1, only marks, mark size=2.5pt, mark repeat=7, forget plot]  table [ x=time, y=st ] {\errbiastable};
    \addplot[Black, mark=triangle, line width=1.5, solid, mark size=2pt] coordinates {(-10,0)};

    \addplot[black!5!Gold2, line width=2, solid, forget plot] table [x=time, y=strw] {\errbiastable};
    \addplot[black!5!Gold2, mark=diamond, line width=1, only marks, mark size=2.5pt, mark repeat=7, forget plot] table [x=time, y=strw] {\errbiastable};
    \addplot[black!5!Gold2, mark=diamond, line width=1.5, solid, mark size=2pt] coordinates {(-10,0)};

    \addplot[Blue2, line width=2, solid, forget plot] table [x=time, y=fstg] {\errbiastable};
    \addplot[Blue2, mark=o, line width=1, only marks, solid, mark repeat=7, forget plot] table [x=time, y=fstg] {\errbiastable};
    \addplot[Blue2, mark=o, line width=1.5, solid] coordinates {(-10,0)};

    \addplot[Red2, line width=2, solid, forget plot] table [x=time, y=fstn] {\errbiastable};
    \addplot[Red2, mark=square, line width=1, only marks, mark size=1.5, mark repeat=7, forget plot] table [x=time, y=fstn] {\errbiastable};
    \addplot[Red2, mark=square, line width=1.5, solid, mark size=1.5] coordinates {(-10,0)};

    \addplot[Gray0, line width=2, solid, forget plot]  table [ x=time, y=corr ] {\errbiastable};
    \addplot[Gray0, mark=x, line width=1.5, only marks, mark size=3.5pt, mark repeat=7, forget plot]  table [ x=time, y=corr ] {\errbiastable};
    \addplot[Gray0, mark=x, line width=1.5, solid, mark size=3pt] coordinates {(-10,0)};

\end{axis}

\end{tikzpicture}
    \end{minipage}
    \caption{Investment value and group-wise bias $b_{\rm\scriptscriptstyle G}(\hbS)$ of estimated graphs $\hbS$ over time: (a) three months estimating graphs every other day and (b) two years and five months with weekly graph estimation.}
\label{fig:invest}
\end{figure*}

\section{Conclusion}
\label{s:conclusion}

In this work, we defined group and individual fairness for graphs, along with metrics to quantify topological bias.
These metrics allowed us to propose two constrained optimization problems for estimating fair graphs from stationary graph signals, with performance guarantees demonstrating when we do or do not experience a tradeoff between fairness and accuracy. 
Future work will see expansion towards other definitions of fairness for graph connectivity, such as extending the notions of multi-precision to graphs. 
As mentioned previously, most existing topological bias metrics consider fairness in terms of single-hop connections, but we aim to develop similar, effective approaches with multi-hop fairness considerations.
Moreover, we may estimate fair graphs while determining the optimal balance between group or individual fairness, which we can impose through adaptive or bilevel optimization.
We also plan to extend this task to more difficult graph settings such as estimating unbiased directed graphs for fair influence maximization.
As another example, nodes in social networks often do not have known group memberships.
Finally, our optimization-based, GSP-driven work aims to participate in the push towards the intersection between GSP and graph-based machine learning.
Incorporating well-founded tools into highly effective models leads to trustworthy yet powerful methods, a necessity for tools increasingly being employed in large-scale yet sensitive applications.

\acks{Research was supported by NSF under award CCF-2340481, the Army Research Office under Grant Number W911NF-17-S-0002, the Spanish AEI (10.13039/501100011033) grants PID2022-136887NB-I00 and PID2023-149457OB-I00, and the Community of Madrid via the Ellis Madrid Unit and grants URJC/CAM F1180 and TEC-2024/COM-89. The views and conclusions contained in this document are those of the authors and should not be interpreted as representing the official policies, either expressed or implied, of the Army Research Office or the U.S. Army or the U.S. Government. The U.S. Government is authorized to reproduce and distribute reprints for Government purposes notwithstanding any copyright notation herein.}

\appendix

\section{Auxiliary Results}
\label{app:supp}

\begin{mylemma}[Claim 2, \citealt{navarro2022jointinferencemultiple}]\label{lem:target_feas}
    Under conditions \textit{(A1)} and \textit{(A2)} of Assumption~\ref{assump:error_bound}, 
    with probability at least $1 - e^{-c_2 \log N}$ for some constant $c_2 > 0$, we have that
    \alna{
        \| \hbC \bbS^* - \bbS^* \hbC \|_F \leq \epsilon.
    \label{eq:lem1}}
\end{mylemma}

\begin{mylemma}\label{lem:grp_tradeoff}
    For any $\bbs\in\reals^{ \frac{1}{2}(N^2-N) }$ such that $\bbs \geq \bbzero$, $\| \hbSigma\bbs \|_2 \leq \epsilon$, and $\| \bbRgr\bbs \|_2 \leq \tau < \| \bbRgr\bbs^* \|_2$, 
    \alna{
        \| \bbs \|_1
        &~\leq~&
        \| \bbZ^\top\bbone \|_2^2
        \left(
            G \| \bbRgr\bbs^* \|_2
            +
            \frac{ 2 \epsilon }{ \sigma_{\scriptscriptstyle\min}(\hbSigma) N_{\scriptscriptstyle\min}}
        \right).
    \label{eq:grp_tradeoff}}
\end{mylemma}

\begin{mylemma}\label{lem:nod_tradeoff}
    For any $\bbs\in\reals^{ \frac{1}{2}(N^2-N) }$ such that $\bbs \geq \bbzero$, $\| \hbSigma\bbs \|_2 \leq \epsilon$, and $\| \bbRno\bbs \|_2 \leq \tau < \| \bbRno\bbs^* \|_2$, 
    \alna{
        \| \bbs \|_1
        &~\leq~&
        \frac{N_{\scriptscriptstyle\max}^2 \sqrt{G}}{2}
        \left(
            G
            \| \bbRno \bbs^* \|_2
            +
            \frac{ 2 \epsilon }{ \sigma_{\scriptscriptstyle\min}(\hbSigma) N_{\scriptscriptstyle\min} }
        \right).
    \label{eq:nod_tradeoff}}
\end{mylemma}

\section{Proof of Theorem~\ref{thm:C_convex}}
\label{app:thm_C_convex}

Let $\bbR$ denote the matrix $\bbRgr$ or $\bbRno$ corresponding to the choice of $R$.
We introduce the following vectorization of the optimization problem~\eqref{eq:opt_C_l0} as 
\alna{
    \bbs'
    ~\in~
    &\argmin_{\bbs}& ~~
    \| \bbs \|_0
    \quad
    ~{\rm s.t.~} 
    \quad
    \| \hbSigma \bbs \|_2 \leq \epsilon,~
    \| \bbR\bbs \|_2 \leq \tau,~
    \bbE \bbs \geq \bbone, ~
    \bbs \geq \bbzero,
\label{eq:vec_C_l0}}
where $\bbs' = \vect(\bbSCp)_{\ccalL}$.
We also vectorize the problem~\eqref{eq:opt_C_l1} as
\alna{
    \hat{\Omega}
    =~
    &\argmin_{\bbs}& ~~
    \| \bbs \|_1
    \quad
    ~{\rm s.t.~} 
    \quad
    \| \hbSigma \bbs \|_2 \leq \epsilon,~
    \| \bbR\bbs \|_2 \leq \tau,~
    \bbE \bbs \geq \bbone, ~
    \bbs \geq \bbzero.
\label{eq:vec_C_l1}}
We proceed with showing that under the conditions of Theorem~\ref{thm:C_convex}, the solution set $\hat{\Omega}$ is a singleton containing only $\bbs'$.
Recall that $\bbPhi = [\hbSigma^\top, \bbR^\top, \bbE^\top]^\top$, and consider the auxiliary problem for a given solution $\hbs \in \hat{\Omega}$
\alna{
    \tilde{\Omega}(\hbs)
    =~
    &\argmin_{\bbs}& ~~
    \| \bbs \|_1
    ~~{\rm s.t.~~} ~
    \bbPhi \bbs = \bbPhi \hbs.
\label{eq:vec_C_l1_aux}}
Observe that any solution in $\tilde{\Omega}(\hbs)$ is non-negative, thus we need only consider the first three constraints of~\eqref{eq:vec_C_l1}.
To see this, let $\tbs \in \tilde{\Omega}(\hbs)$.
\alna{
    \| \tbs \|_1
    \leq
    \| \hbs \|_1
    =
    \bbone^\top \bbE \hbs
    =
    \bbone^\top \bbE \tbs
    \leq
    \| \tbs \|_1,
\nonumber}
thus $\|\tbs\|_1 = \bbone^\top \tbs$ and so $\tbs \geq \bbzero$.
Moreover, we have that $\hbs \in \tilde{\Omega}(\hbs)$.
We proceed with showing that $\bbs'$ is the unique solution to~\eqref{eq:vec_C_l1_aux}, that is, $\tilde{\Omega}(\hbs) = \{ \bbs' \}$.

If there exists a vector $\bby$ such that $\bby \in \Im{\bbPhi^\top}$, $\bby_{\ccalI} = {\rm sign}(\bbs'_{\ccalI})$, and $\| \bby_{\bar{\ccalI}} \|_{\infty} < 1$, then $\bbs' \in \tilde{\Omega}(\hbs)$~\citep{zhang2016OneConditionSolution}.
Thus, consider the following optimization problem
\alna{
    \min_{\bby_{\bar{\ccalI}},\bbr} &&~
    \| \bby_{\bar{\ccalI}} \|_2^2
    + \psi^2 \| \bbr \|_2^2
    \quad ~~ 
    {\rm s.t.} 
    ~~ \quad
    \bby_{\bar{\ccalI}} = \bbPhi_{\cdot,\bar{\ccalI}}^\top \bbr, ~
    \bbPhi_{\cdot,\ccalI}^\top \bbr = {\rm sign}(\bbs'_{\ccalI}).
\nonumber
}
Then, with $\bbt = [\psi\bbr^\top, -\bby_{\bar{\ccalI}}^\top]^\top$ and $\bbQ = [ \psi^{-1} \bbPhi^\top, \bbI_{\cdot,\bar{\ccalI}} ]$, we can rewrite the above problem as
\alna{
    \min_{\bbt} &&~
    \| \bbt \|_2^2
    \quad ~~ 
    {\rm s.t.} 
    ~~ \quad
    \bbQ\bbt = \bbI_{\cdot, \ccalI} {\rm sign}(\bbs'_{\ccalI}),
\nonumber
}
which has solution $\bbt^* = -\bbQ^\top (\bbQ\bbQ^\top)^{-1} \bbI_{\cdot,\ccalI} {\rm sign}(\bbs'_{\ccalI})$.
Then, we obtain 
\alna{
    \bby_{\bar{\ccalI}}^* = - \bbI_{\bar{\ccalI}, \cdot} \left(
        \psi^{-2} \bbPhi^\top \bbPhi + \bbI_{\cdot,\bar{\ccalI}} \bbI_{\bar{\ccalI},\cdot}
        \right)^{-1} \bbI_{\cdot, \ccalI} {\rm sign}(\bbs'_{\ccalI}),
        \nonumber
}
which we can then bound as
\alna{
    \| \bby_{\bar{\ccalI}}^* \|_{\infty}
    &~\leq~&
    \normms{
        \left(
            \psi^{-2} \bbPhi^\top \bbPhi + \bbI_{\cdot,\bar{\ccalI}} \bbI_{\bar{\ccalI},\cdot}
        \right)^{-1}_{\bar{\ccalI},\ccalI} 
    }_{\infty}
    ~<~
    1,
\nonumber
}
where the final inequality is due to assumption \textit{(A1)}.
Thus, there exists a vector $\bby \in {\rm Im}(\bbPhi^\top)$ such that $\bby_{\ccalI} = {\rm sign}(\bbs'_{\ccalI})$ and $\normv{\bby_{\bar{\ccalI}}}_{\infty}<1$, which satisfies the conditions for $\bbs' \in \tilde{\Omega}(\hbs)$.

By assumption \textit{(A2)} of Theorem~\ref{thm:C_convex}, $\hbSigma_{\cdot,\ccalI}$ is full column rank, and thus so is $\bbPhi_{\cdot,\ccalI}$.
This guarantees that $\tilde{\Omega}(\hbs) = \{ \bbs' \}$, that is, $\bbs'$ is the unique solution of~\eqref{eq:vec_C_l1_aux}.
Since $\hbs \in \tilde{\Omega}(\hbs)$, we have that $\bbs' = \hbs$.
Importantly, assumptions \textit{(A1)} and \textit{(A2)} do not rely on the choice of $\hbs \in \hat{\Omega}$.
Thus, $\bbs'$ is the sole element of $\tilde{\Omega}(\hbs)$ for any $\hbs \in \hat{\Omega}$, and therefore $\hat{\Omega} = \{\bbs'\}$.

Thus, we have shown that under \textit{(A1)} and \textit{(A2)} of Theorem~\ref{thm:C_convex}, $\bbs'$ is the unique solution of~\eqref{eq:vec_C_l1}, and we can relax problem~\eqref{eq:vec_C_l0} for a convex problem~\eqref{eq:vec_C_l1} that has a unique solution equivalent to the desired solution $\bbs'$.
Finally, recall that~\eqref{eq:vec_C_l0} and~\eqref{eq:vec_C_l1} are vectorized versions of~\eqref{eq:opt_C_l0} and~\eqref{eq:opt_C_l1}, respectively, so this implies that $\hbSC = \bbSCp$ for $\hbSC$ minimizing~\eqref{eq:opt_C_l1} and $\bbSCp$ minimizing~\eqref{eq:opt_C_l0}, as desired.
$\hfill\blacksquare$

\section{Proof of Lemma~\ref{lem:target_feas}}
\label{app:lem_target_feas}

The following proof exploits the result in Claim 2 of~\citep{navarro2022jointinferencemultiple}, which bounds the Frobenius norm of the commutator $\hbC\bbS^* - \bbS^* \hbC$ for a set of $K$ graphs.
In particular, we apply Claim 2 in~\citep{navarro2022jointinferencemultiple} for $K = 1$, which requires satisfying four conditions.
However, as three are trivial for $K = 1$, we proceed with showing that the remaining condition is equivalent to \textit{1)} and \textit{2)} in the statement of Lemma~\ref{lem:target_feas}.

For condition \textit{4)} of Theorem 2 of~\citep{navarro2022jointinferencemultiple}, given that $K=1$ we have that 
\alna{
    \log N
    &~=~&
    o\left(
        \min\left\{ \frac{M}{(\log M)^2}, M^{1/3} \right\}
    \right)
    =
    o(M^{1/3}),
\nonumber}
since $M^{1/3} \leq M/(\log M)^2$ for every $M > 1$.
$\hfill\blacksquare$

\section{Proof of Lemma~\ref{lem:grp_tradeoff}}
\label{app:lem_grp_tradeoff}

For any $\bbS \in \ccalS$, we let the matrix $\bbS^{(g)}$ denote the submatrix of $\bbS^+$ containing edges connecting nodes in group $g$ for each $g \in [G]$ and $\bbS^{(g,h)}$ the submatrix containing edges connecting nodes between groups $g$ and $h$ for every $g,h\in[G]$ such that $g\neq h$.
For any $\bbS \in \ccalS$, 
\alna{
    \| \bbS^+ \|_1^2
    &~=~&
    \left(
        \sum_{g=1}^G 
        \| \bbS^{(g)} \|_1
        +
        \sum_{g\neq h}
        \| \bbS^{(g,h)} \|_1
    \right)^2
    ~\leq~
    \left(
        \sum_{g=1}^G
        \| \bbS^{(g)} \|_1
    \right)^2
        +
    \left(
        \sum_{g\neq h}
        \| \bbS^{(g,h)} \|_1
    \right)^2
&\nonumber\\&
    &~\leq~&
    2
    \left( \sum_{g=1}^G (N_g^2-N_g)^2 \right) \!\!
    \left( \sum_{g=1}^G \frac{ \| \bbS^{(g)} \|_1^2 }{ (N_g^2-N_g)^2 } \right)
    +
    2
    \left( \sum_{g\neq h} N_g^2 N_h^2 \right) \!\!
    \left( \sum_{g\neq h} \frac{ \| \bbS^{(g,h)} \|_1^2 }{ N_g^2 N_h^2 } \right)
&\nonumber\\&
    &~\leq~&
    2
    \left( \sum_{g\neq h} \frac{ N_g^4 }{ G-1 } \right)
    \left( \sum_{g\neq h} \frac{ \| \bbS^{(g)} \|_1^2 }{ (N_g^2-N_g)^2 } \right)
    +
    2
    \left( \| \bbZ^\top\bbone \|_2^4 - \sum_{g=1}^G N_g^4 \right)
    \left( \sum_{g\neq h} \frac{ \| \bbS^{(g,h)} \|_1^2 }{ N_g^2 N_h^2 } \right)
&\nonumber\\&
    &~\leq~&
    2
    \| \bbZ^\top\bbone \|_2^4
    \sum_{g\neq h}
    \left(\frac{\| \bbS^{(g)} \|_1}{N_g^2-N_g}\right)^2
    + 
    \left(\frac{\| \bbS^{(g,h)} \|_1}{N_g N_h}\right)^2.
\nonumber}
Then, by the definition of $\|\bbRgr\bbs\|_2$ and the fact that $\| \bbs \|_1 = \frac{1}{2} \| \bbS^+ \|_1$ for $\bbs = \vect(\bbS)_{\ccalL}$,
\alna{
    \|\bbs\|_1^2
    &~\leq~&
    \frac{G^2 \| \bbZ^\top\bbone \|_2^4}{2}
    \| \bbRgr\bbs \|_2^2
    +
    \| \bbZ^\top\bbone \|_2^4
    \sum_{g\neq h}
    \frac{ \| \bbS^{(g)} \|_1 \| \bbS^{(g,h)} \|_1 }{(N_g^3-N_g^2)N_h}
&\nonumber\\&
    &~\leq~&
    \frac{G^2 \| \bbZ^\top\bbone \|_2^4}{2}
    \| \bbRgr\bbs \|_2^2
    +
    \frac{\| \bbZ^\top\bbone \|_2^4}{ N_{\scriptscriptstyle\min}^2-\Nmin }
    \sum_{g=1}^G
    \| \bbS^{(g)} \|_F
    \sum_{h\neq g}
    \| \bbS^{(g,h)} \|_F
&\nonumber\\&
    &~\leq~&
    \frac{G^2 \| \bbZ^\top\bbone \|_2^4}{2}
    \| \bbRgr\bbs \|_2^2
    +
    \frac{\| \bbZ^\top\bbone \|_2^4}{ 2 (N_{\scriptscriptstyle\min}^2-\Nmin) }
    \| \bbS^+ \|_F^2
&\nonumber\\&
    &~\leq~&
    \frac{G^2 \| \bbZ^\top\bbone \|_2^4}{2}
    \| \bbRgr\bbs \|_2^2
    +
    \frac{ \epsilon^2 \| \bbZ^\top\bbone \|_2^4}{ \sigma_{\scriptscriptstyle\min}^2(\hbSigma) (N_{\scriptscriptstyle\min}^2-\Nmin) }
    ~\leq~
    G^2 \| \bbZ^\top\bbone \|_2^4
    \| \bbRgr\bbs \|_2^2
    +
    \frac{ 2 \epsilon^2 \| \bbZ^\top\bbone \|_2^4}{ \sigma_{\scriptscriptstyle\min}^2(\hbSigma) N_{\scriptscriptstyle\min}^2 },
\nonumber}
allowing us to write~\eqref{eq:grp_tradeoff} via the triangle inequality as desired.
$\hfill\blacksquare$

\section{Proof of Lemma~\ref{lem:nod_tradeoff}}
\label{app:lem_nod_tradeoff}

First, for any $\bbS \in \ccalS$ with $\bbs = \vect(\bbS)_{\ccalL}$ we have that
\alna{
    \| \bbS^+\bbone \|_2^2
    &~=~&
    \sum_{i=1}^N 
    \left( \sum_{j=1}^N S^+_{ij} \right)^2
    \leq~
    N \| \bbS^+ \|_F^2
    ~=~
    2N \| \bbs \|_2^2
    ~\leq~
    \frac{ 2 N \epsilon^2 }{ \sigma^2_{\scriptscriptstyle\min}(\hbSigma) }
\label{eq:nod_tradeoff_eq1}}
since $\hbSigma$ is full column rank.
Second, we have that
\alna{
    \| \bbS^+ \|_1^2
    &~=~&
    \left(
        \sum_{g=1}^G 
        \| \bbS^+\bbz^{(g)} \|_1
    \right)^2
    \leq~
    \left(
        \sum_{g=1}^G
        \sqrt{N_g}
        \| \bbS^+\bbz^{(g)} \|_2
    \right)^2
    ~\leq~
    N_{\scriptscriptstyle \max}
    \| \bbS^+\bbZ \|_F^2.
\label{eq:nod_tradeoff_eq2}}
Then, by the definition of $\bbRno$, we can lower bound $\| \bbRno\bbs \|_2^2$ by terms containing $\|\bbS^+\bbone\|_2^2$ and $\|\bbS^+\bbZ\|_F^2$, that is,
\alna{
    \| \bbRno \bbs \|_2^2
    &~=~&
    \frac{1}{GN}
    \sum_{g=1}^G
    \sum_{i=1}^N
    \left(
        \frac{ [ \bbS^+ \bbz^{(g)} ]_i }{ N_g }
        -
        \frac{1}{G-1}
        \sum_{h\neq g}
        \frac{ [ \bbS^+ \bbz^{(h)} ]_i }{ N_h }
    \right)^2
&\nonumber\\&
    &~\geq~&
    \frac{1}{ G N N_{\scriptscriptstyle\max}^2 }
    \|\bbS^+\bbZ\|_F^2
    +
    \frac{1}{ G (G-1)^2 N N_{\scriptscriptstyle\max}^2 }
    \sum_{g=1}^G
    \| \bbS^+ ( \bbone - \bbz^{(g)} ) \|_2^2
&\nonumber\\&
    &&
    \qquad
    -
    \frac{2}{ G (G-1) N N_{\scriptscriptstyle\min}^2 }
    \sum_{g=1}^G
    (\bbS^+ \bbz^{(g)})^\top
    \bbS^+ ( \bbone - \bbz^{(g)} )
&\nonumber\\&
    &~=~&
    \frac{1}{ G N N_{\scriptscriptstyle\max}^2 }
    \|\bbS^+\bbZ\|_F^2
    +
    \frac{(G-2)\| \bbS^+\bbone \|_2^2}{ G (G-1)^2 N N_{\scriptscriptstyle\max}^2 }
    +
    \frac{1}{ G (G-1)^2 N N_{\scriptscriptstyle\max}^2 }
    \|\bbS^+\bbZ\|_F^2
&\nonumber\\&
    &&
    \qquad
    -
    \frac{2}{ G (G-1) N N_{\scriptscriptstyle\min}^2 }
    \left(
        \| \bbS^+\bbone \|_2^2
        -
        \|\bbS^+\bbZ\|_F^2
    \right)
&\nonumber\\&
    &~\geq~&
    \frac{G}{(G-1)^2 N N_{\scriptscriptstyle\max}^2}
    \|\bbS^+\bbZ\|_F^2
    \frac{2}{G (G-1) N N_{\scriptscriptstyle\min}^2}
    \normv{ \bbS^+\bbone }_2^2.
\label{eq:bias_nod_lower}}
With~\eqref{eq:nod_tradeoff_eq1},~\eqref{eq:nod_tradeoff_eq2}, and~\eqref{eq:bias_nod_lower}, we bound $\| \bbRno \bbs \|_2^2$ below by $\|\bbS^+\|_1^2$ with
\alna{
    \| \bbRno\bbs \|_2^2
    &~\geq~&
    \frac{ 1 }{ (G-1)^2 N N_{\scriptscriptstyle\max}^3 } 
    \| \bbS^+ \|_1^2
    -
    \frac{ 4 \epsilon^2 }{ G(G-1) N_{\scriptscriptstyle\min}^2 \sigma_{\scriptscriptstyle\min}^2(\hbSigma) },
\nonumber}
which implies that
\alna{
    \| \bbS^+ \|_1^2
    &~\leq~&
    (G-1)^2 N N_{\scriptscriptstyle\max}^3 \| \bbRno\bbs \|_2^2
    +
    \frac{4 (G-1) N N_{\scriptscriptstyle\max}^3 \epsilon^2 }{ G N_{\scriptscriptstyle\min}^2 \sigma_{\scriptscriptstyle\min}^2(\hbSigma) }
&\nonumber\\&
    &~\leq~&
    G^2 N N_{\scriptscriptstyle\max}^3 \| \bbRno\bbs \|_2^2
    +
    \frac{4 N N_{\scriptscriptstyle\max}^3 \epsilon^2 }{ N_{\scriptscriptstyle\min}^2 \sigma_{\scriptscriptstyle\min}^2(\hbSigma) }
    ~\leq~
    G^3 N_{\scriptscriptstyle\max}^4 \| \bbRno\bbs \|_2^2
    +
    \frac{4 G N_{\scriptscriptstyle\max}^4 \epsilon^2 }{ N_{\scriptscriptstyle\min}^2 \sigma_{\scriptscriptstyle\min}^2(\hbSigma) }
\nonumber}
yielding the inequality in~\eqref{eq:nod_tradeoff} by the triangle inequality and since $\| \bbS^+ \|_1^2 = 4\| \bbs \|_1^2$.
$\hfill\blacksquare$

\section{Proof of Theorem~\ref{thm:err_bnd_gr}}
\label{app:thm_err_bnd_gr}

We bound the error between the estimate $\hbSC$ and $\bbS^*$ via the vectorizations $\hbs$ as the solution to~\eqref{eq:vec_C_l1} for $\bbR = \bbRgr$ and $\bbs^* = \vect(\bbS^*)_{\ccalL}$.
Note that since $\hbSigma$ is full column rank, then the problem~\eqref{eq:vec_C_l1} has a unique solution, that is, its solution set is a singleton $\hat{\Omega} = \{\hbs\}$~\citep[Proposition 2]{shafipour2020OnlineTopologyInference}.

To relate the error $\| \hbs - \bbs^* \|_1$ to the sparsity in $\bbs^*$, we define $\ccalK := {\rm supp}(\bbs^*)$ along with the vector $\hbu \in \reals^{ N(N-1)/2 }$ to minimize the distance
\alna{
    \hbu
    :=
    \argmin_{\bbu}
    \| \hbs - \bbu  \|_1
    ~~{\rm s.t.}~~
    {\rm supp}(\bbu) \subseteq \ccalK,
\nonumber}
and we proceed with bounding the distance $\xi := \| \hbs - \hbu \|_1$.
We have that
\alna{
    \xi
    &~=~&
    \min_{\bbu}~
    \| \hbs - \bbu \|_1
    ~~{\rm s.t.}~~
    {\rm supp}(\bbu) \subseteq \ccalK
&\nonumber\\&
    &~=~&
    \max_{\bbv}~
    \min_{\bbu}~
    \| \hbs - \bbu \|_1
    + \bbv^\top \bbu_{\bar{\ccalK}}
&\nonumber\\&
    &~=~&
    \max_{\bbw}~
    \min_{\bbu}~
    \| \hbs - \bbu \|_1
    + \bbw^\top \bbu
    ~~~{\rm s.t.}~~~
    {\rm supp}(\bbw) \subseteq \bar{\ccalK},~
    \normv{\bbw}_{\infty}\leq 1,
\nonumber}
where we require that $\normv{\bbw}_{\infty} \leq 1$, otherwise an entry of $\hbu$ may be $-\infty$, yielding unbounded minimization of $\xi$.
Then, we have that
\alna{
    \xi
    &~\leq~&
    \max_{\bbw}~
    \bbw^\top \hbs
    ~~~{\rm s.t.}~~~
    {\rm supp}(\bbw) \subseteq \bar{\ccalK},~
    \normv{\bbw}_{\infty}\leq 1
&\nonumber\\&
    &~=~&
    \max_{\bbw}~
    (\sign(\bbs^*) + \bbw)^\top \hbs
    - \sign(\bbs^*)^\top \hbs
    ~~~{\rm s.t.}~~~
    {\rm supp}(\bbw) \subseteq \bar{\ccalK},~
    \normv{\bbw}_{\infty}\leq 1
&\nonumber\\&
    &~\leq~&
    \| \hbs \|_1 - \| \bbs^* \|_1
    + \sign(\bbs^*)^\top (\bbs^* - \hbs),
\nonumber}
where the last inequality arises because $\sign(\bbs^*)$ and $\bbw$ have non-overlapping supports and $\normv{\bbw}_{\infty} \leq 1$, so $({\rm sign}(\bbs^*) + \bbw)^\top \hbs \leq \| \hbs \|_1$, and ${\rm sign}(\bbs^*)^\top \bbs^* = \| \bbs^* \|_1$.
Moreover, as $\hbSigma$ is full column rank, 
\alna{
    \xi
    &~\leq~&
    \| \hbs \|_1 - \| \bbs^* \|_1
    + \| (\bbs^* - \hbs)_{\ccalK} \|_1
    ~\leq~
    \| \hbs \|_1 - \| \bbs^* \|_1
    + \sqrt{|\ccalK|} \| \bbs^* - \hbs \|_2
&\nonumber\\&
    &~\leq~&
    \| \hbs \|_1 - \| \bbs^* \|_1
    + \frac{\sqrt{|\ccalK|}}{\sigma_{\scriptscriptstyle\min}(\hbSigma)} \| \hbSigma(\bbs^* - \hbs) \|_2
    ~\leq~
    \| \hbs \|_1 - \| \bbs^* \|_1
    + \frac{2\epsilon\sqrt{|\ccalK|}}{\sigma_{\scriptscriptstyle\min}(\hbSigma)},
\label{eq:xi_l1diff_bound}}
where the final equality is by Lemma~\ref{lem:target_feas}.
With our bound for $\xi = \| \hbs - \hbu \|_1$, we can also bound the distance $\| \bbs^* - \hbu \|_1$.
Since ${\rm supp}(\hbu) \subseteq \ccalK$ and $k \geq \sqrt{|\ccalK|}$ by Assumption \textit{(A4)}, 
\alna{
    \normv{\bbs^* - \hbu}_1
    &~\leq~&
    k \normv{\bbs^* - \hbu}_2
    ~\leq~
    k \normv{\hbs - \hbu}_1
    +
    k \normv{\bbs^* - \hbs}_2
&\nonumber\\&
    &~\leq~&
    k \normv{\hbs - \hbu}_1
    +
    \frac{k}{\sigma_{\scriptscriptstyle\min}(\hbSigma)}
    \normv{\hbSigma(\bbs^* - \hbs)}_2
    ~\leq~
    k \normv{\hbs - \hbu}_1
    +
    \frac{2k\epsilon}{\sigma_{\scriptscriptstyle\min}(\hbSigma)}.
\nonumber}
Then, by the definition of $\xi$ and~\eqref{eq:xi_l1diff_bound} we have that
\alna{
    \| \hbs - \bbs^* \|_1
    &~\leq~&
    \xi + \| \bbs^* - \hbu \|_1
&\nonumber\\&
    &~\leq~&
    (1 + k)
    \xi
    +
    \frac{2k\epsilon}{\sigma_{\scriptscriptstyle\min}(\hbSigma)}
    ~\leq~
    (1 + k)
    \left( \normv{ \hbs }_1 - \normv{ \bbs^* }_1 \right)
    +
    (2 + k)
    \frac{2k\epsilon}{\sigma_{\scriptscriptstyle\min}(\hbSigma)}.
\label{eq:err_tri_equ}}
Note that if $\bbs^*$ is a feasible solution of~\eqref{eq:vec_C_l1}, then $\| \hbs \|_1 \leq \|\bbs^*\|_1$ by the optimality of $\hbs$, and the error bound consists only of the second term in the right-hand side of~\eqref{eq:err_tri_equ}.
This proves the lower and upper bounds in~\eqref{eq:err_low_gr} and~\eqref{eq:err_upp_gr}, respectively, when $\Rgr(\bbS^*) \leq \tau^2$.
Thus, we proceed with bounding the error when $\bbs^*$ is not feasible, that is, $\Rgr(\bbS^*) > \tau^2$.

By Lemma~\ref{lem:grp_tradeoff}, if $\|\bbRgr\bbs^*\|_2 > \tau$ then
\alna{
    \| \hbS - \bbS^* \|_1
    &~=~&
    2 \| \hbs - \bbs^* \|_1
    ~\leq~
    \frac{ 4k \epsilon (2 + k) }{ \sigma_{\scriptscriptstyle\min}(\hbSigma) }
    +
    2G \| \bbZ^\top\bbone \|_2^2
    (1 + k)
    \normv{ \bbR \bbs^* }_2
    +
    \frac{ 4 \epsilon \|\bbZ^\top\bbone\|_2^2 (1 + k) }{ \sigma_{\scriptscriptstyle\min}(\hbSigma) N_{\scriptscriptstyle\min} },
\nonumber}
equivalent to the upper bound in~\eqref{eq:err_upp_gr} when $\Rgr(\bbS^*) = \| \bbRgr\bbs^* \|_2^2 > \tau^2$.

We then move on to demonstrating the lower bound in~\eqref{eq:err_low_gr}.
For any $\bbS \in \ccalS$, $\bbS^{(g)} \in \reals^{N_g\times N_g}$ denotes the submatrix of $\bbS^+$ containing edges connecting nodes in group $g$ for every $g\in[G]$.
Similarly, for every $g,h\in[G]$ such that $g\neq h$, $\bbS^{(g,h)} \in \reals^{N_g \times N_h}$ is the submatrix of $\bbS$ containing edges connecting nodes between groups $g$ and $h$.
By the definition of $\bbRgr$ and the fact that $\bbS\geq \bbzero$,
\alna{
    \| \bbRgr\bbs \|_2^2
    &~=~&
    \frac{1}{G^2-G}
    \sum_{g\neq h}
    \left(
        \frac{ \bbone^\top \bbS^{(g)} \bbone }{ N_g^2 - N_g }
        -
        \frac{ \bbone^\top \bbS^{(g,h)} \bbone }{ N_g N_h }
    \right)^2
&\nonumber\\&
    &~\leq~&
    \frac{1}{G^2-G}
    \sum_{g\neq h}
    \left( \frac{ \| \bbS^{(g)} \|_1 }{ N_g^2 - N_g } \right)^2
    +
    \left( \frac{ \| \bbS^{(g,h)} \|_1 }{ N_g N_h } \right)^2
&\nonumber\\&
    &~\leq~&
    \frac{1}{G^2-G}
    \sum_{g\neq h}
    \left(
        \frac{ \|\bbS^{(g)}\|_F^2 }{ N_g^2 - N_g }
        +
        \frac{ \| \bbS^{(g,h)} \|_F^2 }{ N_g N_h }
    \right)
&\nonumber\\&
    &~\leq~&
    \frac{1}{G( N_{\scriptscriptstyle\min}-1 )^2}
    \left(
        \sum_{g=1}^G
        \|\bbS^{(g)}\|_F^2
        +
        \sum_{g\neq h}
        \| \bbS^{(g,h)} \|_F^2
    \right)
&\nonumber\\&
    &~=~&
    \frac{1}{G( N_{\scriptscriptstyle\min}-1 )^2}
    \|\bbS^+\|_F^2
    ~\leq~
    \frac{4}{G N_{\scriptscriptstyle\min}^2}
    \|\bbS\|_F^2.
\nonumber}
Then, we have that
\alna{
    \| \hbS - \bbS^* \|_1
    &~\geq~&
    \| \hbS - \bbS^* \|_F
    ~\geq~
    \frac{ N_{\scriptscriptstyle\min} \sqrt{G} }{2}
    \| \bbRgr (\hbs - \bbs^*) \|_2
    ~\geq~
    \frac{ N_{\scriptscriptstyle\min} \sqrt{G} }{2}
    ( \| \bbRgr\bbs^* \|_2 - \tau ),
\nonumber}
yielding the result in~\eqref{eq:err_low_gr} as desired since $\sqrt{\Rgr(\bbS^*)} = \| \bbRgr\bbs^* \|_2$ and $\| \hbS - \bbS^* \|_1 \geq 0$.
$\hfill\blacksquare$

\section{Proof of Theorem~\ref{thm:err_bnd_no}}
\label{app:thm_err_bnd_no}

As in Appendix~\ref{app:thm_err_bnd_gr}, we bound the error between $\hbSC$ and $\bbS^*$ via the vectorizations $\hbs$ and $\bbs^*$ for $\bbR = \bbRno$.
We may repeat the steps in the proof of Theorem~\ref{thm:err_bnd_gr} until equation~\eqref{eq:err_tri_equ}.
As before, when $\Rno(\bbS^*) \leq \tau^2$, the target $\bbs^*$ is a feasible solution of~\eqref{eq:vec_C_l1}, and we obtain the upper and lower error bounds in~\eqref{eq:err_upp_no} and~\eqref{eq:err_low_no}, respectively, when $\Rno(\bbS^*) \leq \tau^2$.
We then continue with the case when $\Rno(\bbS^*) > \tau^2$.
By Lemma~\ref{lem:nod_tradeoff}, if $\| \bbRno\bbs^* \|_2 > \tau$ then
\alna{
    \| \hbS - \bbS^* \|_1
    &~=~&
    2 \| \hbs - \bbs^* \|_1
&\nonumber\\&
    &~\leq~&
    \frac{ 4 k \epsilon (2 + k)  }{ \sigma_{\scriptscriptstyle\min}(\hbSigma) }
    +
    \frac{ 2 \epsilon N_{\scriptscriptstyle\max}^2 \sqrt{G} (1 + k) }{ \sigma_{\scriptscriptstyle\min}(\hbSigma) N_{\scriptscriptstyle\min} }
    +
    G N_{\scriptscriptstyle\max}^2 \sqrt{G}
    (1 + k)
    \normv{ \bbRno \bbs^* }_2,
\nonumber}
yielding the upper bound in~\eqref{eq:err_upp_no} when $\Rno(\bbS^*) = \| \bbRno\bbs^* \|_2^2 > \tau^2$.

For the lower bound in~\eqref{eq:err_low_no}, we use the definition of $\bbRno$ and $\bbS\geq\bbzero$ to get
\alna{
    \| \bbRno\bbs \|_2^2
    &~=~&
    \frac{1}{GN}
    \sum_{g=1}^G
    \sum_{i=1}^N
    \left(
        \frac{ [\bbS^+ \bbz^{(g)}]_i }{ N_g }
        -
        \frac{1}{G-1}
        \sum_{h\neq g}
        \frac{ [\bbS^+ \bbz^{(h)}]_i }{ N_h }
    \right)^2
&\nonumber\\&
    &~\leq~&
    \frac{1}{GN}
    \sum_{g=1}^G
    \sum_{i=1}^N
    \frac{ [ \bbS^+\bbz^{(g)} ]_i^2 }{ N_g^2 }
    +
    \frac{1}{G (G-1)^2 N}
    \sum_{g=1}^G
    \sum_{i=1}^N
    \left(
        \sum_{h\neq g}
        \frac{ [ \bbS^+\bbz^{(h)} ]_i }{ N_h }
    \right)^2
&\nonumber\\&
    &~\leq~&
    \frac{1}{GN N_{\scriptscriptstyle\min}^2}
    \| \bbS^+\bbZ \|_F^2
    +
    \frac{1}{G (G-1)^2 N N_{\scriptscriptstyle\min}^2 }
    \sum_{g=1}^G
    \| \bbS^+ (\bbone - \bbz^{(g)}) \|_2^2
&\nonumber\\&
    &~\leq~&
    \frac{1}{GN N_{\scriptscriptstyle\min}^2}
    \| \bbS^+\bbZ \|_F^2
    +
    \frac{G-2}{G (G-1)^2 N N_{\scriptscriptstyle\min}^2 }
    \normv{\bbS^+\bbone}_2^2
    +
    \frac{1}{G (G-1)^2 N N_{\scriptscriptstyle\min}^2 }
    \| \bbS^+\bbZ \|_F^2
&\nonumber\\&
    &~\leq~&
    \left(
        \frac{1}{GN N_{\scriptscriptstyle\min}^2}
        +
        \frac{1}{ G(G-1) N N_{\scriptscriptstyle\min}^2 }
    \right) \normv{\bbS^+}_1^2
&\nonumber\\&
    &~\leq~&
    \frac{2}{ (G-1) N N_{\scriptscriptstyle\min}^2 }
    \normv{\bbS^+}_1^2
    ~\leq~
    \frac{4}{ G N N_{\scriptscriptstyle\min}^2 }
    \normv{\bbS}_1^2,
\nonumber}
which gives
\alna{
    \| \hbS - \bbS^* \|_1
    &~\geq~&
    \frac{ N_{\scriptscriptstyle\min} \sqrt{GN} }{ 2 }
    \| \bbRno (\hbs - \bbs^*) \|_2
    ~\geq~
    \frac{ N_{\scriptscriptstyle\min} \sqrt{GN} }{ 2 }
    \left(
        \| \bbRno \bbs^* \|_2 - \tau
    \right),
\nonumber}
as in~\eqref{eq:err_low_no} as desired for $\Rno(\bbS^*) > \tau^2$.
$\hfill\blacksquare$

\section{Proof of Theorem~\ref{thm:V_convex}}
\label{app:thm_V_convex}

The proof follows analogous steps to the proof of Theorem~\ref{thm:C_convex} in Appendix~\ref{app:thm_C_convex}.
We again proceed only for $\ccalS = \SA$, but the result also holds for $\ccalS = \SL$, as discussed in Appendix~\ref{app:Lapl}.

First, consider the following vectorization of~\eqref{eq:opt_V_l0}
\alna{
    \bbs', \bblambda'_{\rm\scriptscriptstyle V}
    ~\in~
    &\argmin_{\bbs, \bblambda}& ~~
    \| \bbs \|_0
    \quad
    ~{\rm s.t.~} 
    \quad
    \| \bbU\bbs - \hbJ \bblambda \|_2 \leq \epsilon,~
    \| \bbR\bbs \|_2 \leq \tau,~
    \bbE \bbs \geq \bbone, ~
    \bbs \geq \bbzero,
\label{eq:vec_V_l0_aux}}
whose optimal set of GSOs can also be obtained via
\alna{
    \bbs'
    ~\in~
    &\argmin_{\bbs}& ~~
    \| \bbs \|_0
    \quad
    ~{\rm s.t.~} 
    \quad
    \| \bbU\bbs - \hbJ \hbJ^\top \bbU \bbs \|_2 \leq \epsilon,~
    \| \bbR\bbs \|_2 \leq \tau,~
    \bbE \bbs \geq \bbone, ~
    \bbs \geq \bbzero
\label{eq:vec_V_l0}}
since the pseudo-inverse of $\hbJ$ is $\hbJ^{\dagger} = \hbJ^\top$~\citep{segarra2017networktopologyinference}.
Indeed, for any solution $(\bbs'_1,\bblambda'_1)$ of~\eqref{eq:vec_V_l0_aux}, $\bbs'_1$ is feasible for~\eqref{eq:vec_V_l0}, and therefore $\| \bbs'_2 \|_0 \leq \| \bbs'_1 \|_0$ for any solution $\bbs'_2$ of~\eqref{eq:vec_V_l0}.
Conversely, $(\bbs'_2,\hbJ^\top \bbU \bbs'_2)$ is a feasible solution of~\eqref{eq:vec_V_l0_aux}, therefore $\| \bbs'_1 \|_0 \leq \| \bbs'_2 \|_0$, so~\eqref{eq:vec_V_l0_aux} and~\eqref{eq:vec_V_l0} share the same optimal vectorized GSOs.
We also vectorize~\eqref{eq:opt_V_l1} as
\alna{
    \hat{\Omega}
    =~
    &\argmin_{\bbs}& ~~
    \| \bbs \|_1
    \quad
    ~{\rm s.t.~} 
    \quad
    \| \hbF \bbs \|_2 \leq \epsilon,~
    \| \bbR\bbs \|_2 \leq \tau,~
    \bbE \bbs \geq \bbone, ~
    \bbs \geq \bbzero.
\label{eq:vec_V_l1}}
The remainder of the proof follows the same steps as in Appendix~\ref{app:thm_C_convex}.
$\hfill\blacksquare$

\section{Fair Spectral Templates with Graph Laplacian}
\label{app:Lapl}

Here, we introduce analogous theoretical results for graph Laplacian GSOs $\ccalS = \SL$.

\subsection{FairSpecTemp Convex Relaxation}
\label{app:Lapl_convex}

We introduce $\hbSigma_{\rm\scriptscriptstyle L} := [\hbC \oplus (-\hbC)]_{\cdot,\ccalD}\bbE + \hbSigma$ such that $\|\hbC\bbS-\bbS\hbC\|_F = \|\hbSigma_{\rm\scriptscriptstyle L}\bbs\|_2$ for $\bbS\in\SL$ and $\bbs = \vect(\bbS)_{\ccalL}$. 
Then, we adapt the result on the convex relaxation of~\eqref{eq:opt_C_l0} in Theorem~\ref{thm:C_convex} for $\ccalS = \SL$ as follows.

\begin{mycorollary}\label{cor:C_convex}
    If problems~\eqref{eq:opt_C_l0} and~\eqref{eq:opt_C_l1} are feasible for $\ccalS = \SL$, then for $\bbSCp$ as a solution to~\eqref{eq:opt_C_l0}, we have that $\hbSC = \bbSCp$ is the unique solution to~\eqref{eq:opt_C_l1} if
    \begin{itemize}[left= 15pt .. 24pt, itemsep=1pt]
    \im[(A1)] The submatrix $[\hbSigma_{\rm \scriptscriptstyle L}]_{\cdot,\ccalI}$ is full column rank, and
    \im[(A2)] There exists a constant $\psi > 0$ such that
    \alna{
        \normms{
            \left(
                \psi^{-2} \bbPhi^\top \bbPhi + \bbI_{\cdot,\bar{\ccalI}} \bbI_{\bar{\ccalI},\cdot}
            \right)^{-1}_{\bar{\ccalI},\ccalI} 
        }_{\infty}
        ~<~
        1,
    \label{eq:cor_C_cvx_cond}}
    \end{itemize}
    where $\ccalI = {\rm supp}(\bbs')$ for $\bbs' = \vect(\bbSCp)_{\ccalL}$, and $\bbPhi := [ \hbSigma_{\rm \scriptscriptstyle L}^\top, \bbR^\top, \bbE^\top ]^\top$.
\end{mycorollary}

\begin{proof}
    Observe that the equivalent vectorization of~\eqref{eq:opt_C_l0} for $\ccalS = \SL$ is
    \alna{
        \bbs'
        ~\in~
        &\argmin_{\bbs}& ~~
        \| \bbs \|_0
        ~~~~{\rm s.t.}~~~~
        \| \hbSigma_{\rm \scriptscriptstyle L} \bbs \|_2 \leq \epsilon,~
        \| \bbR\bbs \|_2 \leq \tau,~
        -\bbE \bbs \geq \bbone, ~
        \bbs \leq \bbzero,
    \nonumber}
    so an analogous procedure with the same auxiliary problem~\eqref{eq:vec_C_l1_aux} yields the same guarantee.
\end{proof}
Similarly, we adapt Theorem~\ref{thm:V_convex} on the convex relaxation of~\eqref{eq:opt_V_l0} for $\ccalS=\SL$.
To this end, we define $\hbF_{\rm\scriptscriptstyle L} := [\bbI-\hbJ\hbJ^\top]_{\cdot,\ccalD}\bbE + \hbF$.

\begin{mycorollary}\label{cor:V_convex}
    If problems~\eqref{eq:opt_V_l0} and~\eqref{eq:opt_V_l1} are feasible for $\ccalS = \SL$, then for $\bbSVp$ as a solution to~\eqref{eq:opt_V_l0}, we have that $\hbSV = \bbSVp$ is the unique solution to~\eqref{eq:opt_V_l1} if
    \begin{itemize}[left= 15pt .. 24pt, itemsep=1pt]
    \im[(A1)] The submatrix $[\hbF_{\rm\scriptscriptstyle L}]_{\cdot,\ccalI}$ is full column rank, and
    \im[(A2)] There exists a constant $\psi > 0$ such that
    \alna{
        \normms{
            \left(
                \psi^{-2} \bbPsi^\top \bbPsi + \bbI_{\cdot,\bar{\ccalI}} \bbI_{\bar{\ccalI},\cdot}
            \right)^{-1}_{\bar{\ccalI},\ccalI}
        }_{\infty}
        ~<~
        1,
    \label{eq:cor_V_cvx_cond}}
    \end{itemize}
    where $\ccalI = {\rm supp}(\bbs')$ for $\bbs' = \vect(\bbSVp)_{\ccalL}$, and $\bbPsi := [ \hbF_{\rm\scriptscriptstyle L}^\top, \bbR, \bbE^\top ]^\top$.
\end{mycorollary}
The proof of Corollary~\ref{cor:V_convex} is analogous to that of Corollary~\ref{cor:C_convex} and is thus omitted.

\subsection{FairSpecTemp Error Bound}
\label{app:Lapl_error}

Let conditions \textit{(A1)}, \textit{(A2)}, and \textit{(A4)} hold from Assumption~\ref{assump:error_bound}, and let $\hbSigma_{\rm\scriptscriptstyle L}$ be full column rank, analogous to condition \textit{(A3)}.
To see why the error bounds in Theorems~\ref{thm:err_bnd_gr} and~\ref{thm:err_bnd_no} are equivalent for $\ccalS = \SA$ and $\ccalS = \SL$ up to a scaling, first observe that by the proofs of the lower bounds in~\eqref{eq:err_low_gr} and~\eqref{eq:err_low_no}, we obtain the same results if $\diag(\bbS) \neq \bbzero$.
Then, for the upper bounds in~\eqref{eq:err_upp_gr} and~\eqref{eq:err_upp_no} and $\hbS,\bbS^*\in\SL$, observe that
\alna{
    \| \hbS - \bbS^* \|_1
    &~\leq~&
    \| (\hbS - \bbS^*)^+ \|_1
    +
    \| (\hbS - \bbS^*)^- \|_1
&\nonumber\\&
    &~=~&
    2\| \hbs - \bbs^* \|_1
    +
    \| \bbE(\hbs - \bbs^*) \|_1
&\nonumber\\&
    &~\leq~&
    (2 + \sigma_{\max}(\bbE))
    \| \hbs - \bbs^* \|_1,
\nonumber}
thus the upper bounds from Theorems~\ref{thm:err_bnd_gr} and~\ref{thm:err_bnd_no} are scaled by $\frac{2 + \sigma_{\max}(\bbE)}{2}$ for $\ccalS = \SL$ with the assumption that $\hbSigma_{\rm\scriptscriptstyle L}$ is full column rank, which allows us to repeat the steps in Appendices~\ref{app:thm_err_bnd_gr} and~\ref{app:thm_err_bnd_no} for the proofs of Theorems~\ref{thm:err_bnd_gr} and~\ref{thm:err_bnd_no}, respectively.


\vskip 0.2in
\bibliography{citations}

\end{document}